\PassOptionsToPackage{table}{xcolor}
\documentclass{article}

\usepackage{iclr2027_conference,times}

\usepackage{microtype}
\usepackage{hyperref}
\usepackage{url}
\definecolor{darkblue}{rgb}{0.0,0.0,0.55}
\hypersetup{colorlinks=true, citecolor=darkblue, linkcolor=darkblue, urlcolor=darkblue}
\usepackage{booktabs}
\usepackage{amsmath,amssymb,amsthm}
\usepackage[table]{xcolor}
\usepackage{graphicx}
\usepackage{xspace}
\usepackage{algorithm}
\usepackage{algorithmic}
\usepackage{subcaption}
\usepackage{adjustbox}
\usepackage{multirow}
\usepackage{array}
\usepackage{tabularx}
\usepackage{flafter}
\usepackage{placeins}


\definecolor{tblHeader}{HTML}{E2EAF2}
\definecolor{tblSubheader}{HTML}{F1F4F7}
\definecolor{tblRazor}{HTML}{F9E8EB}
\definecolor{tblPrecursor}{HTML}{ECD9E6}
\definecolor{tblPositive}{HTML}{E6F2EA}
\definecolor{tblNegative}{HTML}{F4D8CF}
\definecolor{tblWarning}{HTML}{FAF0DD}
\definecolor{tblAccent}{HTML}{A61B29}
\newcommand{\thead}[1]{\textbf{\strut #1}}

\usepackage{amsmath,amsfonts,bm}

\def\eqref#1{equation~\ref{#1}}
\def\plaineqref#1{\ref{#1}}

\def\1{\bm{1}}

\DeclareMathAlphabet{\mathsfit}{\encodingdefault}{\sfdefault}{m}{sl}
\SetMathAlphabet{\mathsfit}{bold}{\encodingdefault}{\sfdefault}{bx}{n}

\newcommand{\method}{\textsc{Razor}\xspace}
\newcommand{\rcsloo}{RCS\discretionary{-}{}{-}LOO\xspace}
\newcommand{\thinkclose}{\texttt{\textless/think\textgreater}\xspace}
\newcommand{\rcs}{RCS\xspace}
\newcommand{\rcsrefill}{RCS\discretionary{-}{}{-}Refill\xspace}
\newcommand{\reaprms}{REAP\discretionary{-}{}{-}RMS\xspace}
\newcommand{\calib}{\textsc{RazorCal}\xspace}
\newcommand{\con}{c}
\newcommand{\resid}{r}
\newcommand{\routeset}{S}
\newcommand{\remset}{\mathcal{A}}

\theoremstyle{plain}
\newtheorem{proposition}{Proposition}
\theoremstyle{definition}
\newtheorem{assumption}{Assumption}
\newtheorem{remark}{Remark}

\definecolor{prune25}{rgb}{0.90,0.98,0.90}
\definecolor{prune50}{rgb}{0.78,0.94,0.78}
\definecolor{modelA}{rgb}{0.88,0.92,0.98}
\definecolor{modelB}{rgb}{0.95,0.90,0.97}

\title{RAZOR: Pruning Replaceable Experts in LLMs}

\newcommand{\linkicon}[1]{\hbox to 0.95em{%
  \includegraphics[height=1.55ex]{figures/icons/#1.pdf}\hfil}}
\author{Mingyang Song, Mao Zheng \\
Foundation Model Department, Tencent, China \\
\texttt{nickmysong@tencent.com} \\[5pt]
\normalfont\small\linkicon{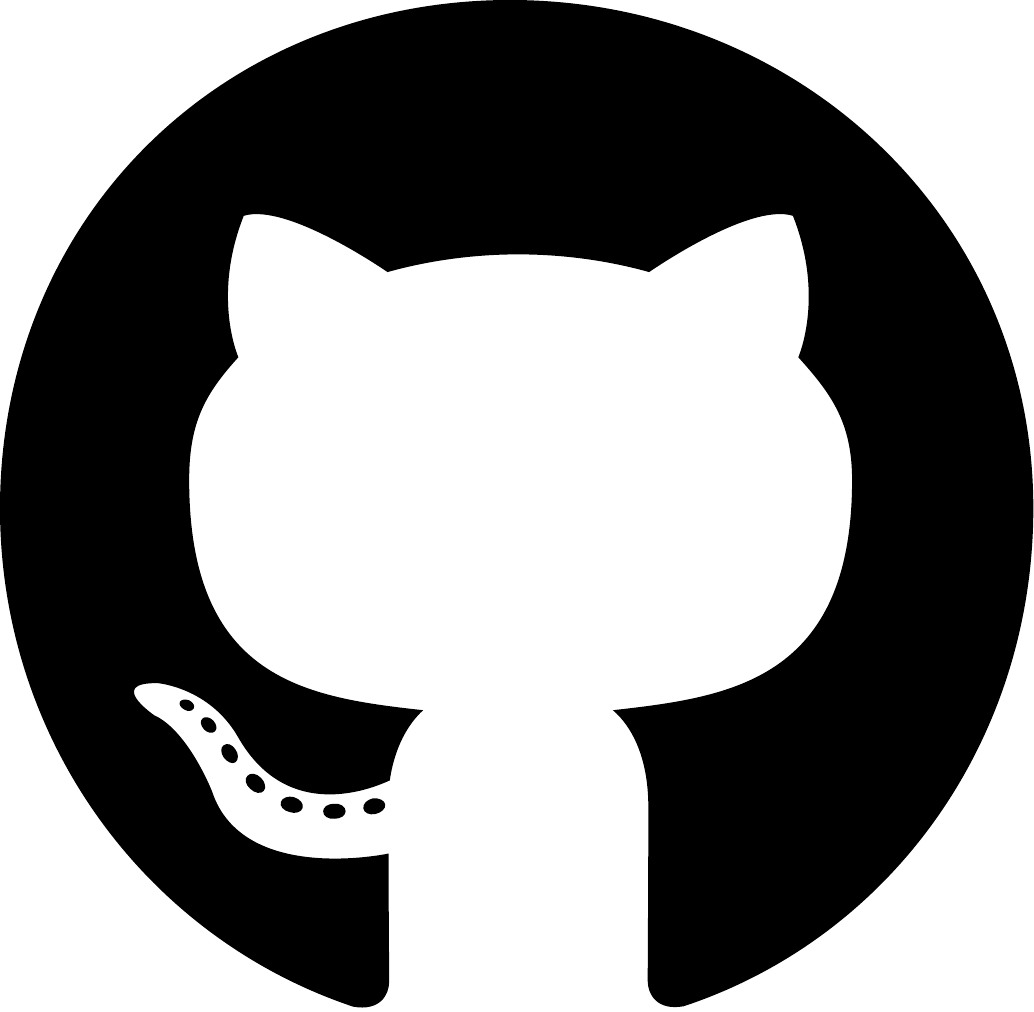}\href{https://github.com/nick7nlp/Razor}{Code}
\quad
\normalfont\small\linkicon{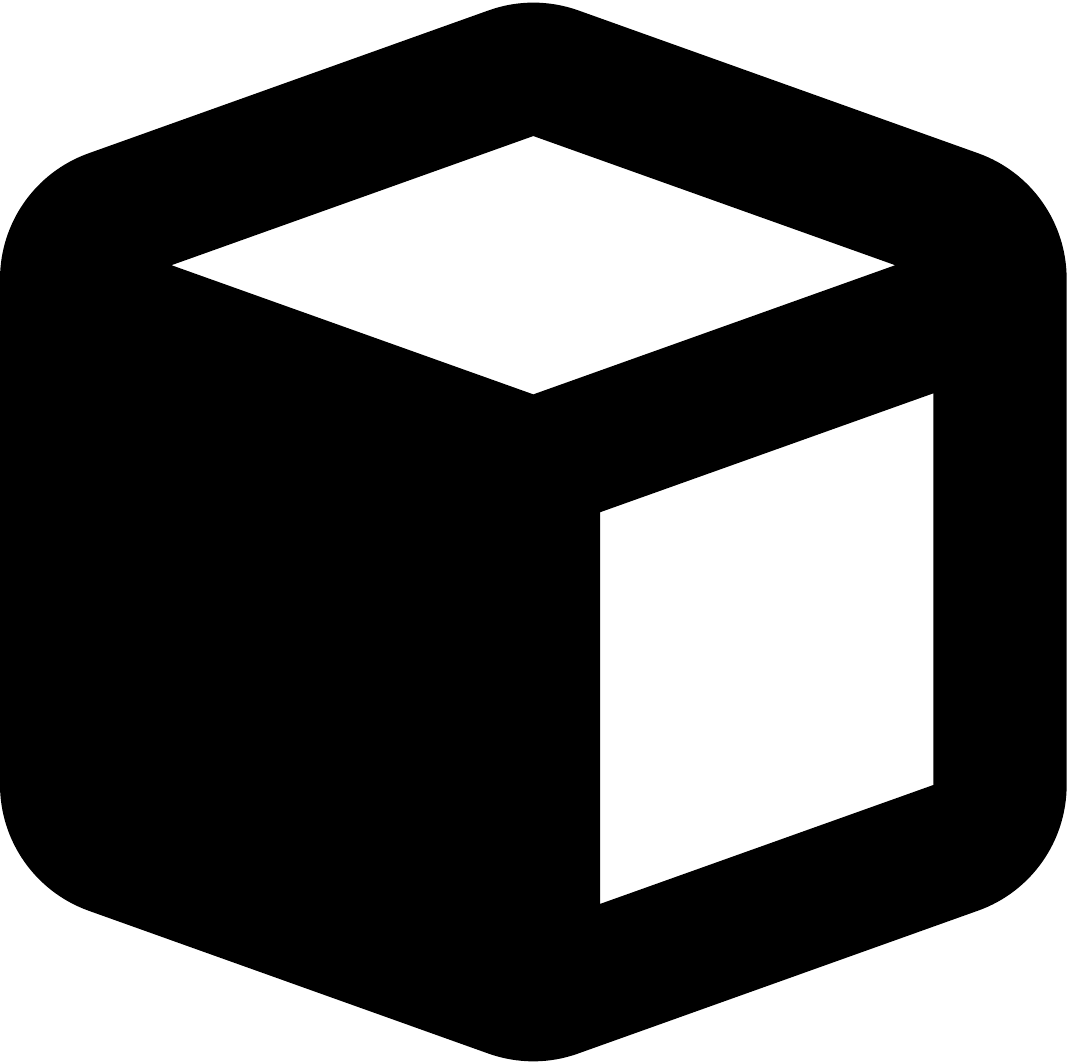}\href{https://huggingface.co/collections/Nickyang/razor}{Models}}

\newif\ifpreprint
\preprinttrue          %

\ifpreprint\iclrfinalcopy\fi

\begin{document}
\maketitle
\ifpreprint\lhead{Work in Progress}\fi

\begin{abstract}
Mixture-of-experts (MoE) models activate only a few experts per token but store
the entire expert pool. Pruning this pool requires identifying experts whose
removal preserves model behavior. Routing frequency and output magnitude do
not fully describe deletion damage, which also depends on how the surviving
and replacement experts compensate for the removed output.
We introduce \method, a training-free pruning method based on consensus
residuals, the deviations of expert outputs from their original weighted
mixture. At a fixed layer input, these residuals give the exact output change
for a single deletion under survivor renormalization and router refill.
\method aggregates this damage by conditional root mean square and selects
experts under a layerwise budget using forward computation alone, without
gradients, subset search, or recovery training. Against frequency,
activation-norm, and REAP baselines on
GLM-4.7-Flash and Qwen3.6-35B-A3B at 25\% and 50\% expert removal, it attains
the highest macro average over nine reasoning-intensive tasks in all four
model--budget settings, gaining 2.12--5.59 points over REAP and lowering
reverse KL in all four. On DeepSeek-V4-Flash-0731 and Hy3, it also achieves
the highest macro average among the three residual criteria.
Local exactness does not guarantee better joint pruning. Generation analyses
show changes in diversity, formatting, and termination despite higher task scores.
\end{abstract}

\section{Introduction}

Mixture-of-experts (MoE) models increase capacity while activating only a small
subset of experts per token, yet sparse computation does not reduce the cost of
storing the full expert pool
\citep{shazeer2017,lepikhin2021gshard,fedus2022switch}. Whole-expert pruning
addresses this burden by reducing the pool itself
\citep{lu2024notallexperts,muzio2024seer,reap2026,eep2024}. Our focus is pruning
reasoning models while retaining performance on tasks that require multi-step
problem solving. Changes to intermediate predictions can affect later steps,
motivating an evaluation of both predictive fidelity and complete task outputs.
At a fixed budget, we assess \emph{functional replaceability}, the extent to
which the surviving computation can compensate for a removed expert.

Common pruning scores characterize experts by selection frequency or output
magnitude \citep{muzio2024seer,fantasticexperts2025,reap2026}. These statistics
do not capture how expert outputs combine, even when the output norm is
weighted by routing mass. A large contribution may be
replaceable by the surviving mixture, whereas a small one may supply a component
that the survivors cannot recover. Isolated importance therefore differs from
deletion damage, which depends on the computation available after removal.

\begin{figure}[!t]
  \centering
  \includegraphics[width=\textwidth]{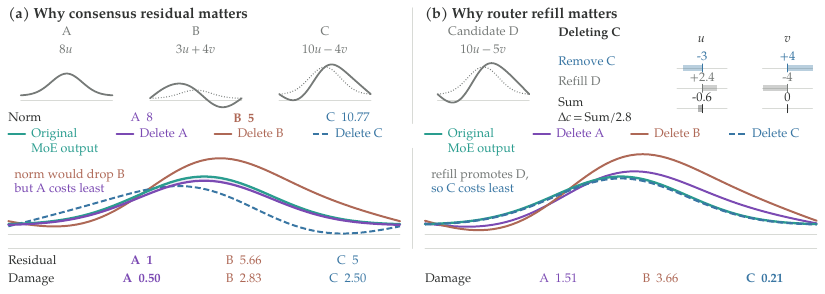}
  \caption{\textbf{Replaceability depends on output geometry and router refill.}
  A constructed single-token example without (a) and with (b) refill.
  Damage is L2 distance to the original mixture. Bold marks row minima.
  Appendix~\ref{app:geometry} gives the construction and plotting details.}
  \label{fig:razor-intuition}
\end{figure}

Deletion changes a routed mixture by renormalizing the survivors' weights and,
when the router refills the vacated slot, promoting an unselected expert. Figure~\ref{fig:razor-intuition} isolates both effects in a constructed
single-token example. Without refill, panel~(a) shows that removing the expert
closest to the original mixture is less harmful than removing the one with the
smallest output norm. With refill, panel~(b) shows that a promoted expert can
compensate for a removed balancing contribution and change the preferred
deletion. Output magnitude, fixed-support damage, and refill-aware damage can
therefore select different experts. More generally, identical gates and output
norms can correspond to different least-harmful deletions
(Proposition~\ref{prop:magnitude_insufficient}).

We study three consensus-residual criteria that account for different parts
of the fixed-input deletion counterfactual.
The reference is the original weighted MoE output, which we call the
\emph{consensus} without assuming expert agreement. \rcs measures each expert's
routing-weighted deviation from this mixture. \rcsloo accounts for survivor
renormalization, and \rcsrefill additionally includes router-selected
replacement. We call the refill criterion with conditional root mean square
(RMS) aggregation \method.
\rcsloo and \rcsrefill recover exact single-deletion output changes under their
respective fixed-input assumptions. Their comparison tests whether modeling
these additional effects improves the pruned model. We retain the
highest-scoring experts within each layer's budget using forward computation,
without gradients, subset search, or recovery
training. Local change remains a surrogate for the model-level objective
because joint deletions interact and altered activations
propagate across layers.

Our benchmark suite spans six capability categories, namely competition
mathematics, coding from function level to repository level, long-context
inference, graduate-level knowledge, verifiable constraint following, and tool
use. Scores assess final answers or executable outputs, providing evidence
about task performance rather than the correctness of intermediate reasoning.
We compare against pruning baselines on GLM-4.7-Flash and Qwen3.6-35B-A3B
at 25\% and 50\% expert removal,
where \method has the highest nine-task macro average of all evaluated methods
in every setting, exceeding REAP by 2.12--5.59 points, and also lowers reverse
KL relative to REAP in all four. DeepSeek-V4-Flash-0731 and Hy3 extend the
evaluation to two additional backbones, where \method also achieves the
highest nine-task macro average among the residual criteria.

The improvements have two limits. All four backbones fall below their
unpruned macro scores at 50\% removal, and better task performance does not
consistently preserve generation behavior. On Qwen3.6-35B-A3B, diversity does not
follow the task ranking, and \method at 50\% removal emits fewer stray
closing delimiters than REAP while reaching the output cap more often than the
unpruned model. These results motivate evaluating response behavior alongside
accuracy when pruning reasoning models.

\section{Methodology}
\label{sec:method}

\subsection{Pruning objective and routed mixture}
\label{sec:prelim}

We retain $B$ of $E$ routed experts per layer, with $k\leq B<E$, without
recovery training. Preserving the original output distribution at this budget
is the goal. We use local Euclidean output change as a tractable surrogate,
without assuming that minimizing it also minimizes model-level divergence.

For token representation $x$, the router selects $k$ experts
$\routeset(x)$. Expert $i$ produces $f_i(x)$ with normalized weight
$w_i(x)\geq0$, where $\sum_{i\in\routeset(x)}w_i(x)=1$. The routed output is
\begin{equation}
 \con(x)=\sum_{i\in\routeset(x)}w_i(x)f_i(x).
\label{eq:consensus}
\end{equation}
We call this mixture the \emph{consensus}, without assuming that its experts
agree. Unchanged residual-connection, dense, and shared-expert branches cancel
in the local comparison. A common routed-output scale does not change
within-layer rankings (Appendix~\ref{app:implementation}), and
Appendix~\ref{app:config} lists each backbone's scale.

The statistic $w_i\|f_i\|_2$ measures the magnitude of expert $i$'s weighted
output contribution. Deletion removes that contribution and renormalizes the
survivors, so its local effect depends on information that even the complete
set of gates and output norms cannot supply.

\begin{proposition}[Magnitude does not identify the least harmful deletion]
\label{prop:magnitude_insufficient}
Even with equal known gates, the expert-output norms do not in general determine
which single-expert deletion minimizes the local output shift after survivor
renormalization without refill.
\end{proposition}
\begin{proof}
Consider equally weighted scalar outputs $(8,3,10)$ and $(8,3,-10)$. Both have
the same output norms and gates. Their original mixtures are $7$ and $1/3$.
The absolute shifts after deleting the first, second, or third expert are
$(1/2,2,3/2)$ in the first case and $(23/6,4/3,31/6)$ in the second. The unique
least harmful deletion is therefore the first expert in one case and the second
in the other. A rule observing only the gates and norms cannot distinguish them.
\end{proof}

Reweighting output norms cannot recover the missing directional information.
Given $f_i$ and $w_i$, however, the shared vector $\con$ determines the survivor
output.

\subsection{From consensus residual to exact deletion effect}
\label{sec:residual}

To isolate one deletion, we hold the layer input and expert outputs fixed and
remove a selected expert with $w_i(x)<1$. We first model router refill, then
recover deletion without replacement as a special case.
\begin{assumption}[Refill]
\label{asm:refill}
Deleting $i\in\routeset(x)$ promotes the highest-ranked unselected expert $r$
under the router's selection rule. Its nonnegative mixture score, divided by
the original selected-score sum, defines the pseudo-weight $w_r(x)$.
The surviving scores and this promoted score are then renormalized over their
combined mass $1-w_i(x)+w_r(x)$.
\end{assumption}
\begin{assumption}[Fixed support]
\label{asm:support}
The surviving weights are renormalized over
$\routeset(x)\setminus\{i\}$ without promoting an unselected expert.
Equivalently, $w_r(x)\equiv0$ in Assumption~\ref{asm:refill}.
\end{assumption}
We write $\tilde\con^{-i}$ for the refilled mixture and $\con^{-i}$ for the
fixed-support one. With a fixed ranking, the same highest-ranked unselected
expert supplies the replacement for every single deletion at that token.
The first comparison therefore uses the routed outputs and one additional
expert output, whereas the second needs only the routed outputs.
Neither requires evaluating a separate end-to-end pruned model for each expert.
\begin{remark}
\label{rem:denominator}
Nonnegative $w_r(x)$ and $w_i(x)<1$ ensure
$D_i(x)\equiv1-w_i(x)+w_r(x)\geq1-w_i(x)>0$. This argument does not require
the selection ranking to match the ordering of mixture weights, which may
differ when routing uses a selection-only correction bias
(Appendix~\ref{app:implementation}).
\end{remark}

\begin{proposition}[Single-expert deletion under refill]
\label{prop:refill}
Under Assumption~\ref{asm:refill}, the token-level deletion damage is
\begin{equation}
 \delta_i(x)\equiv\big\|\con(x)-\tilde\con^{-i}(x)\big\|_2
 =\frac{\big\|w_i(x)\resid_i(x)-w_r(x)\resid_r(x)\big\|_2}{D_i(x)},
\label{eq:reroute}
\end{equation}
where $\resid_j(x)=f_j(x)-\con(x)$ is the consensus residual of expert $j$.
\end{proposition}

\begin{proof}
Suppressing $x$, the promoted slot carries mass $w_r$ while the surviving
selected mass is $1-w_i$, so $\tilde\con^{-i}=(\con-w_if_i+w_rf_r)/D_i$. Using
$D_i-1=-w_i+w_r$,
\[
 \con-\tilde\con^{-i}
 =\frac{D_i\con-\con+w_if_i-w_rf_r}{D_i}
 =\frac{w_i(f_i-\con)-w_r(f_r-\con)}{D_i}.
\]
Taking norms yields the result.
\end{proof}

Setting $w_r=0$ recovers the fixed-support case in closed form.

\begin{proposition}[Single-expert leave-one-out impact]
\label{prop:loo}
Under Assumption~\ref{asm:support}, the token-level deletion damage is
\begin{equation}
 \delta_i^{\mathrm{loo}}(x)\equiv\big\|\con(x)-\con^{-i}(x)\big\|_2
 =\frac{w_i(x)}{1-w_i(x)}\,\|f_i(x)-\con(x)\|_2 .
\label{eq:loo}
\end{equation}
\end{proposition}

\begin{proof}
Setting $w_r=0$ in Proposition~\ref{prop:refill} gives denominator
$1-w_i$ and numerator $w_i\|f_i-\con\|_2$.
\end{proof}

These identities give replaceability a concrete form. Fixed-support damage is
equivalently $w_i\|f_i-\con^{-i}\|_2$, so distance to the survivor output
measures how far the remaining computation is from reproducing the removed
expert, and routing mass converts that distance into mixture change. Small
damage is therefore ambiguous by construction, arising from either a small
distance or a small routing mass. The LOO factor $1/(1-w_i)$ corrects for
self-inclusion in $\con$ and follows from the deletion counterfactual rather
than serving as a tunable importance weight (Appendix~\ref{app:survivor}).
Refill subtracts the promoted expert's weighted residual from the removed
expert's residual and changes the normalizer. The promoted expert can therefore
compensate for the removed contribution, though by an
amount set by residual magnitudes and routing weights as well as direction.
An exact match $w_i\resid_i=w_r\resid_r$ leaves the local mixture unchanged.

Exact single-deletion effects are not additive. Without refill, removing a
routed subset $\remset\subseteq\routeset(x)$ with total weight
$W_{\remset}=\sum_{i\in\remset}w_i<1$ gives
\begin{equation}
 \con-\con^{-\remset}
 =\frac{\sum_{i\in\remset}w_i(f_i-\con)}{1-W_{\remset}}.
\label{eq:setshift}
\end{equation}
Residuals can reinforce or cancel, while removed mass changes the denominator
(Appendix~\ref{app:set}). Scalar ranking does not optimize this joint effect,
and refill does not restore additivity. Promoting several ranks at once couples
the removed experts through both the numerator and the set-dependent normalizer. If $W_{\remset}=1$, no
routing mass survives and the fixed-support counterfactual is undefined, even if
$B\geq k$ experts remain in the layer. Deployment reselects top-$k$ from the
retained pool and propagates changed activations across layers, so matching its
refill rule at a single token still does not guarantee checkpoint fidelity.

\subsection{Expert scoring and pruning}
\label{sec:rms}

A pruning rule requires one score per expert. \rcs uses $w_i\|\resid_i\|_2$, \rcsloo uses
$\delta_i^{\mathrm{loo}}$, and \rcsrefill uses $\delta_i$. Unless explicitly
varied, all three aggregate their token quantities by root mean square (RMS)
over calibration tokens routed to each expert, which weights variable damage
above steady damage of the same average size. \emph{We write \textup{\method}
for \rcsrefill under conditional RMS}. This aggregation is an empirical
choice rather than a consequence of the deletion identity.
Appendices~\ref{app:rms_detail} and~\ref{app:aggregation} analyze and test it,
and Table~\ref{tab:instances} summarizes experimental coverage.

We retain the highest-scoring experts within each layer's budget and deploy
top-$k$ routing over the retained pools. Appendix~\ref{app:implementation}
gives the implementation and pseudocode.

\begin{figure}[!t]
\centering
\includegraphics[width=\textwidth]{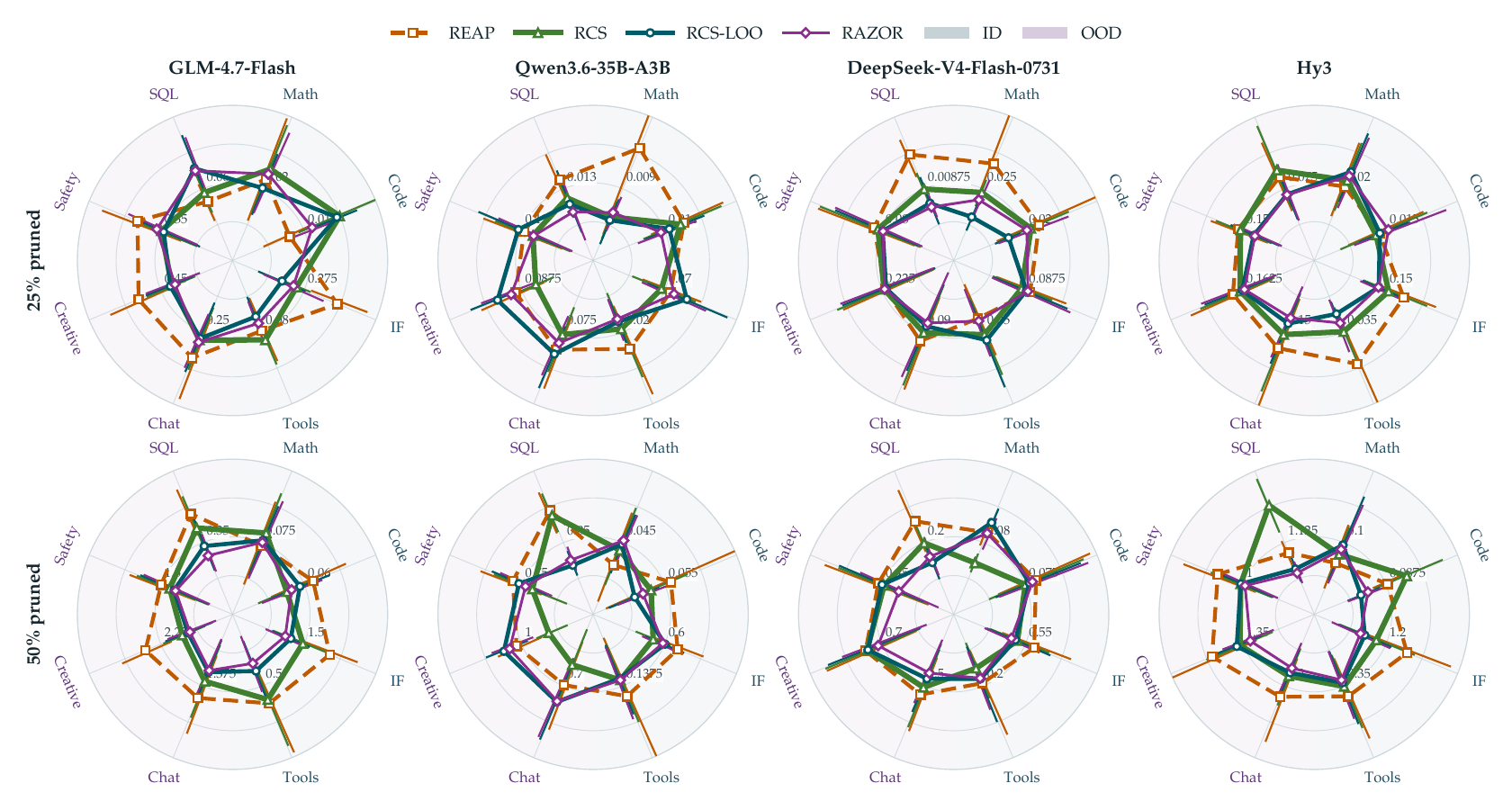}
\caption{\textbf{Domain-wise reverse KL at 25\% (top) and 50\% (bottom) removal.}
Values are in nats, with lower (inward) better. Compare methods along each spoke,
not by polygon area. Whiskers show pointwise 95\% paired-batch intervals.
Shading distinguishes ID/OOD domains.}
\label{fig:radar-four-criteria}
\end{figure}

\section{Experiments}
\label{sec:experiments}

Four research questions narrow from what a pruned model predicts to why the
scoring choices behind it matter.
\begin{enumerate}
\item[\textbf{RQ1.}] How well does pruning preserve output distributions across
domains and routing conditions?
\item[\textbf{RQ2.}] How much downstream task performance does each criterion
retain against prior scores, and does the ordering hold on further backbones?
\item[\textbf{RQ3.}] Which response properties shift even where task
accuracy is retained?
\item[\textbf{RQ4.}] Do refinements to the local deletion score improve pruning,
and how do scoring and aggregation choices affect expert selection and
deletion-set interactions?
\end{enumerate}
The experiments separately assess fidelity, task performance, response
behavior, and the effects of individual scoring choices.

\subsection{Experimental Setup}

\paragraph{Models and pruning methods.}
We compare \rcs, \rcsloo, and \method on GLM-4.7-Flash,
Qwen3.6-35B-A3B, DeepSeek-V4-Flash-0731, and Hy3 at 25\% and 50\% expert
removal, without recovery training. All three use conditional RMS unless an
aggregation ablation is specified, and each retains the highest-scoring
experts in learned-router layers. Fixed-hash layers use the common pruning and
route-remapping procedure in Appendix~\ref{app:implementation}. Matched checkpoint
and component studies use GLM-4.7-Flash and Qwen3.6-35B-A3B, whose routers
select 4/64 and 8/256 experts, respectively.

\paragraph{Baselines.}
We compare against routing frequency (Frequency), a conditional-mean
adaptation of expert activation norm (EAN) \citep{fantasticexperts2025}, and
REAP \citep{reap2026}, which averages routing-weighted output norms
conditional on use. Downstream benchmark comparisons against these three
baselines use GLM-4.7-Flash and Qwen3.6-35B-A3B, while separate mask-based
domain-wise and routing-stratified diagnostics compare against REAP on all four
backbones (Appendices~\ref{app:domain_fidelity} and~\ref{app:mechanism}).
REAP$^\star$ denotes an external checkpoint calibrated with 24,576 samples,
distinct from the matched REAP baseline.

\begin{table}[!t]
\definecolor{prune25}{HTML}{F2F6F8}
\definecolor{prune50}{HTML}{E2EBEF}
\definecolor{modelA}{HTML}{F7F0EE}
\definecolor{modelB}{HTML}{EFE1DD}
\centering
\small
\setlength{\tabcolsep}{3pt}
\renewcommand{\arraystretch}{1.05}
\caption{Nine-task benchmark results on GLM-4.7-Flash and Qwen3.6-35B-A3B at 25\% and 50\% removal.
Scores are scaled by 100. Bold marks the best pruned score in each column within a model--budget block.}
\label{tab:main}
\resizebox{\dimexpr\textwidth-2pt\relax}{!}{%
\begin{tabular}{>{\centering\arraybackslash}p{0.90cm}>{\centering\arraybackslash}p{0.60cm}c|c|ccc|c|c|c|cccc|c}
\toprule
\multicolumn{2}{c}{\multirow{2}{*}{\textbf{Model}}} & \multirow{2}{*}{\textbf{Method}} & \multicolumn{1}{c|}{Math} & \multicolumn{3}{c|}{Instruction Following} & \multicolumn{1}{c|}{Know.} & \multicolumn{1}{c|}{Tool} & \multicolumn{1}{c|}{LC} & \multicolumn{4}{c|}{Coding} & \multirow{2}{*}[-3pt]{\shortstack{\textit{Overall}\\\textit{Avg}}} \\
\cmidrule(lr){4-4} \cmidrule(lr){5-7} \cmidrule(lr){8-8} \cmidrule(lr){9-9} \cmidrule(lr){10-10} \cmidrule(lr){11-14}
& & & AIME'26 & IFEval & IFBench & \textit{Avg} & SuperGPQA & BFCL v4 & LongB. v2 & HE+ & LCB'26 & SWE & \textit{Avg} & \\
\midrule
\rowcolor{modelA}[\tabcolsep][\tabcolsep] & 0\% & --- & 31.7 & 79.7 & 40.0 & 59.9 & 38.5 & 65.3 & 25.3 & 79.9 & 37.4 & 5.2 & 40.8 & 44.8 \\
\rowcolor{prune25}[\tabcolsep][\tabcolsep] \cellcolor{modelA} &  & Frequency & 26.3 & 77.3 & 37.3 & 57.3 & 29.5 & 62.8 & 21.9 & 74.4 & 29.3 & \textbf{5.2} & 36.3 & 40.4 \\
\rowcolor{prune25}[\tabcolsep][\tabcolsep] \cellcolor{modelA} &  & EAN & 30.4 & 76.3 & 39.3 & 57.8 & 31.5 & 64.7 & 27.2 & 75.0 & 33.8 & 4.8 & 37.9 & 42.6 \\
\rowcolor{prune25}[\tabcolsep][\tabcolsep] \cellcolor{modelA} &  & REAP$^\star$ & 30.0 & 74.9 & 36.7 & 55.8 & 32.6 & 64.9 & 27.2 & 75.0 & 30.6 & 3.6 & 36.4 & 41.7 \\
\rowcolor{prune25}[\tabcolsep][\tabcolsep] \cellcolor{modelA} &  & REAP & 30.8 & 74.9 & 37.3 & 56.1 & 33.7 & 65.0 & 27.0 & 75.0 & 34.3 & 3.6 & 37.6 & 42.4 \\
\rowcolor{prune25}[\tabcolsep][\tabcolsep] \cellcolor{modelA} &  & \textbf{\rcs} & 29.6 & 77.1 & 36.0 & 56.6 & 33.3 & \textbf{65.7} & 26.0 & 73.2 & 33.8 & 4.2 & 37.1 & 42.1 \\
\rowcolor{prune25}[\tabcolsep][\tabcolsep] \cellcolor{modelA} &  & \textbf{\rcsloo} & 31.3 & 76.7 & 39.7 & 58.2 & \textbf{34.3} & 64.8 & 27.4 & 75.6 & 34.8 & 4.2 & 38.2 & 43.2 \\
\rowcolor{prune25}[\tabcolsep][\tabcolsep] \cellcolor{modelA} & \multirow{-7}{*}{25\%} & \textbf{\method} & \textbf{33.3} & \textbf{78.2} & \textbf{40.3} & \textbf{59.3} & 34.2 & 65.6 & \textbf{28.5} & \textbf{76.2} & \textbf{39.8} & 4.6 & \textbf{40.2} & \textbf{44.5} \\
\rowcolor{prune50}[\tabcolsep][\tabcolsep] \cellcolor{modelA} &  & Frequency & 0.0 & 62.8 & 30.3 & 46.6 & 15.7 & 51.1 & 18.7 & 54.9 & 17.8 & 0.2 & 24.3 & 27.9 \\
\rowcolor{prune50}[\tabcolsep][\tabcolsep] \cellcolor{modelA} &  & EAN & 15.8 & 70.6 & 29.0 & 49.8 & 25.4 & 54.8 & 23.7 & 59.2 & \textbf{33.9} & 2.6 & 31.9 & 35.0 \\
\rowcolor{prune50}[\tabcolsep][\tabcolsep] \cellcolor{modelA} &  & REAP & 23.3 & 61.0 & 25.3 & 43.2 & 27.8 & 56.0 & 22.1 & 67.1 & 26.7 & 1.2 & 31.7 & 34.5 \\
\rowcolor{prune50}[\tabcolsep][\tabcolsep] \cellcolor{modelA} &  & \textbf{\rcs} & 29.6 & 63.8 & 27.3 & 45.6 & 24.9 & \textbf{59.1} & 19.5 & 71.3 & 27.6 & 3.0 & 34.0 & 36.2 \\
\rowcolor{prune50}[\tabcolsep][\tabcolsep] \cellcolor{modelA} &  & \textbf{\rcsloo} & 26.3 & 67.5 & 30.7 & 49.1 & 29.0 & 58.6 & 20.9 & 72.0 & 29.4 & 3.2 & 34.9 & 37.5 \\
\rowcolor{prune50}[\tabcolsep][\tabcolsep] \cellcolor{modelA}\multirow{-14}{=}[1pt]{\rotatebox[origin=c]{90}{\textbf{GLM-4.7-Flash}}} & \multirow{-6}{*}{50\%} & \textbf{\method} & \textbf{30.8} & \textbf{72.0} & \textbf{32.3} & \textbf{52.2} & \textbf{33.7} & \textbf{59.1} & \textbf{25.9} & \textbf{73.2} & 29.8 & \textbf{4.0} & \textbf{35.7} & \textbf{40.1} \\
\midrule
\rowcolor{modelB}[\tabcolsep][\tabcolsep] & 0\% & --- & 76.3 & 83.9 & 33.7 & 58.8 & 62.8 & 67.6 & 50.1 & 91.5 & 71.9 & 17.0 & 60.1 & 61.6 \\
\rowcolor{prune25}[\tabcolsep][\tabcolsep] \cellcolor{modelB} &  & Frequency & 75.0 & 79.1 & 34.3 & 56.7 & 51.4 & \textbf{67.2} & 45.9 & 90.9 & 70.1 & 12.0 & 57.7 & 58.4 \\
\rowcolor{prune25}[\tabcolsep][\tabcolsep] \cellcolor{modelB} &  & EAN & 75.4 & 81.7 & 32.3 & 57.0 & 57.4 & 65.8 & 48.7 & 92.7 & 64.8 & 16.6 & 58.0 & 59.5 \\
\rowcolor{prune25}[\tabcolsep][\tabcolsep] \cellcolor{modelB} &  & REAP & 77.1 & 79.5 & 33.7 & 56.6 & 57.0 & 65.8 & 48.8 & 90.9 & 70.9 & 14.6 & 58.8 & 59.8 \\
\rowcolor{prune25}[\tabcolsep][\tabcolsep] \cellcolor{modelB} &  & \textbf{\rcs} & 76.7 & 80.0 & 38.0 & 59.0 & 57.3 & 66.3 & 48.5 & 92.1 & 71.9 & 18.0 & 60.7 & 61.0 \\
\rowcolor{prune25}[\tabcolsep][\tabcolsep] \cellcolor{modelB} &  & \textbf{\rcsloo} & 76.3 & 80.6 & 34.7 & 57.7 & 57.1 & 66.4 & 49.7 & 92.7 & \textbf{72.6} & 17.4 & 60.9 & 60.8 \\
\rowcolor{prune25}[\tabcolsep][\tabcolsep] \cellcolor{modelB} & \multirow{-6}{*}{25\%} & \textbf{\method} & \textbf{77.5} & \textbf{81.9} & \textbf{38.3} & \textbf{60.1} & \textbf{57.8} & 67.0 & \textbf{49.9} & \textbf{93.3} & \textbf{72.6} & \textbf{20.2} & \textbf{62.0} & \textbf{62.1} \\
\rowcolor{prune50}[\tabcolsep][\tabcolsep] \cellcolor{modelB} &  & Frequency & 68.8 & 71.0 & 31.7 & 51.4 & 41.1 & 43.1 & 42.5 & 89.6 & 55.0 & 7.2 & 50.6 & 50.0 \\
\rowcolor{prune50}[\tabcolsep][\tabcolsep] \cellcolor{modelB} &  & EAN & 72.1 & 74.1 & 30.3 & 52.2 & 46.1 & 58.7 & 46.1 & 88.4 & 47.1 & 8.4 & 48.0 & 52.4 \\
\rowcolor{prune50}[\tabcolsep][\tabcolsep] \cellcolor{modelB} &  & REAP & 71.7 & 73.2 & 31.0 & 52.1 & 46.3 & 63.9 & 47.5 & 88.4 & 66.4 & 12.6 & 55.8 & 55.7 \\
\rowcolor{prune50}[\tabcolsep][\tabcolsep] \cellcolor{modelB} &  & \textbf{\rcs} & 73.3 & 74.3 & \textbf{36.3} & 55.3 & 48.3 & 64.1 & \textbf{48.3} & 91.5 & 67.6 & 12.8 & 57.3 & 57.4 \\
\rowcolor{prune50}[\tabcolsep][\tabcolsep] \cellcolor{modelB} &  & \textbf{\rcsloo} & \textbf{75.4} & 78.2 & 32.0 & 55.1 & 47.0 & 64.5 & 47.9 & \textbf{92.1} & 67.9 & 13.6 & 57.9 & 57.6 \\
\rowcolor{prune50}[\tabcolsep][\tabcolsep] \cellcolor{modelB}\multirow{-13}{=}[1pt]{\rotatebox[origin=c]{90}{\textbf{Qwen3.6-35B-A3B}}} & \multirow{-6}{*}{50\%} & \textbf{\method} & 75.0 & \textbf{79.3} & 34.0 & \textbf{56.7} & \textbf{49.2} & \textbf{65.0} & \textbf{48.3} & \textbf{92.1} & \textbf{69.4} & \textbf{15.8} & \textbf{59.1} & \textbf{58.7} \\
\bottomrule
\end{tabular}%
}
\end{table}

\paragraph{Data and metrics.}
Calibration uses the 2,048-example, seven-domain \calib pool with
model-specific chat templates. Predictive fidelity is reverse KL,
$D_{\mathrm{KL}}(q\|p)$, from original ($p$) and pruned ($q$) predictions on
shared reference prefixes. Assistant-token NLL gives
$\Delta\mathrm{NLL}=\mathrm{NLL}_{q}-\mathrm{NLL}_{p}$ and relative excess PPL,
$\exp(\Delta\mathrm{NLL})-1$. Fidelity is reported over eight axes, four covered
by \calib (mathematics, code, instruction following, and tools) and four
held out from it (chat, creative writing, safety, and SQL), so
calibration-covered and held-out fidelity stay separable. This split is a
property of the fidelity axes and does not carry over to the benchmark suite
or the generation analyses. The matched
component sweep pools tokens within axes and weights axes equally, whereas
domain-wise and routing-stratified studies use grouped estimators.

The nine downstream tasks are AIME'26 \citep{maaAIME}, IFEval
\citep{ifeval2023}, IFBench \citep{ifbench2025}, SuperGPQA
\citep{supergpqa2025}, BFCL v4 \citep{bfcl2025,bfclv4}, LongBench v2
\citep{longbench2024}, HumanEval+ \citep{evalplus2023}, LiveCodeBench 2026
\citep{livecodebench2024}, and SWE-bench Verified
\citep{swebench2024}. They span competition mathematics,
coding from function level to repository level, long-context inference,
graduate-level knowledge, verifiable constraint following, and tool use. Each
is scored on its final answer or artifact, such as a numeric answer, a selected
option, or code that passes tests, rather than on the intermediate steps.
Appendix~\ref{app:benchmark_details} details task versions and score aggregation.
Response analyses measure diversity, stray closing thinking delimiters,
code-fence structure, length, and length-limit finishes. These are observable
properties of the generated text, not measures of reasoning quality, and they
qualify the task results without replacing them.

\paragraph{Hyperparameters and evaluation protocol.}
Matched comparisons use the same calibration inputs and layerwise budgets,
with the external REAP$^\star$ checkpoint reported separately.
For downstream benchmarks, we use temperature $0$, except for SWE-bench
Verified, where we use $0.7$. AIME'26 reports avg@8, and the other eight benchmarks report avg@3. Diversity and response-form
diagnostics follow the separate protocol of Section~\ref{sec:generation}.
Context and output limits, the thinking setting, and cross-study matching
are detailed in Appendices~\ref{app:razorcal}--\ref{app:benchmark_details}.
Diagnostic sampling and uncertainty are described in
Appendices~\ref{app:domain_fidelity}--\ref{app:behaviour}.

\subsection{Predictive Fidelity (RQ1)}
\label{sec:diagnostics}

We first test whether scores based on local output change preserve the
model's predictive distribution after joint pruning. Figure~\ref{fig:radar-four-criteria} compares REAP and the
three residual criteria per domain, Table~\ref{tab:loo_ablation} gives matched
equal-axis GLM-4.7-Flash and Qwen3.6-35B-A3B estimates, and
Figure~\ref{fig:routing-four-criteria} resolves the same comparison by routing
concentration under both fidelity measures, separating calibration-covered from
held-out domains. The three residual criteria use conditional RMS throughout,
so contrasts among them isolate the token quantity, whereas REAP uses the
conditional mean of its original definition, so comparisons against it vary
aggregation as well. Lower reverse KL indicates less drift,
and radar spokes are scaled independently from mostly nonzero origins, so
methods are comparable on the same spoke and not by polygon area.

\begin{table}[!t]
\definecolor{prune25}{HTML}{F2F6F8}
\definecolor{prune50}{HTML}{E2EBEF}
\definecolor{modelA}{HTML}{F7F0EE}
\definecolor{modelB}{HTML}{EFE1DD}
\centering
\small
\setlength{\tabcolsep}{3pt}
\renewcommand{\arraystretch}{1.05}
\caption{Nine-task benchmark results on DeepSeek-V4-Flash-0731 and Hy3
at 25\% and 50\% removal. These two backbones compare the residual criteria
against the unpruned reference, without the Frequency, EAN, and REAP baselines
of Table~\ref{tab:main}. Other conventions follow Table~\ref{tab:main}.}
\label{tab:additional}
\resizebox{\dimexpr\textwidth-2pt\relax}{!}{%
\begin{tabular}{>{\centering\arraybackslash}p{0.90cm}>{\centering\arraybackslash}p{0.60cm}c|c|ccc|c|c|c|cccc|c}
\toprule
\multicolumn{2}{c}{\multirow{2}{*}{\textbf{Model}}} & \multirow{2}{*}{\textbf{Method}} & \multicolumn{1}{c|}{Math} & \multicolumn{3}{c|}{Instruction Following} & \multicolumn{1}{c|}{Know.} & \multicolumn{1}{c|}{Tool} & \multicolumn{1}{c|}{LC} & \multicolumn{4}{c|}{Coding} & \multirow{2}{*}[-3pt]{\shortstack{\textit{Overall}\\\textit{Avg}}} \\
\cmidrule(lr){4-4} \cmidrule(lr){5-7} \cmidrule(lr){8-8} \cmidrule(lr){9-9} \cmidrule(lr){10-10} \cmidrule(lr){11-14}
& & & AIME'26 & IFEval & IFBench & \textit{Avg} & SuperGPQA & BFCL v4 & LongB. v2 & HE+ & LCB'26 & SWE & \textit{Avg} & \\
\midrule
\rowcolor{modelA}[\tabcolsep][\tabcolsep] & 0\% & --- & 67.9 & 86.1 & 37.7 & 61.9 & 60.3 & 72.6 & 33.8 & 83.5 & 69.3 & 21.2 & 58.0 & 59.2 \\
\rowcolor{prune25}[\tabcolsep][\tabcolsep] \cellcolor{modelA} &  & \textbf{\rcs} & 73.8 & 82.6 & 37.0 & 59.8 & 55.4 & 70.2 & 35.4 & 91.5 & 68.3 & 12.4 & 57.4 & 58.5 \\
\rowcolor{prune25}[\tabcolsep][\tabcolsep] \cellcolor{modelA} &  & \textbf{\rcsloo} & 76.7 & 83.7 & \textbf{38.7} & \textbf{61.2} & 55.6 & 70.4 & 35.2 & 89.6 & 68.8 & \textbf{18.8} & 59.1 & 59.7 \\
\rowcolor{prune25}[\tabcolsep][\tabcolsep] \cellcolor{modelA} & \multirow{-3}{*}{25\%} & \textbf{\method} & \textbf{81.3} & \textbf{84.7} & 36.3 & 60.5 & \textbf{56.0} & \textbf{72.8} & \textbf{37.4} & \textbf{93.3} & \textbf{69.1} & 18.2 & \textbf{60.2} & \textbf{61.0} \\
\rowcolor{prune50}[\tabcolsep][\tabcolsep] \cellcolor{modelA} &  & \textbf{\rcs} & 75.8 & 78.4 & 36.7 & 57.6 & \textbf{46.7} & 66.5 & 30.4 & 84.2 & 58.0 & 8.4 & 50.2 & 53.9 \\
\rowcolor{prune50}[\tabcolsep][\tabcolsep] \cellcolor{modelA} &  & \textbf{\rcsloo} & 76.7 & 79.8 & \textbf{39.0} & 59.4 & 46.5 & 68.1 & 29.6 & 85.4 & \textbf{58.4} & \textbf{10.4} & \textbf{51.4} & 54.9 \\
\rowcolor{prune50}[\tabcolsep][\tabcolsep] \cellcolor{modelA}\multirow{-7}{=}[1pt]{\rotatebox[origin=c]{90}{\footnotesize\shortstack{\textbf{DeepSeek-V4-}\\[-1pt]\textbf{Flash-0731}}}} & \multirow{-3}{*}{50\%} & \textbf{\method} & \textbf{78.8} & \textbf{80.4} & 38.7 & \textbf{59.6} & 46.4 & \textbf{72.1} & \textbf{30.8} & \textbf{86.0} & 58.2 & 9.2 & 51.1 & \textbf{55.6} \\
\midrule
\rowcolor{modelB}[\tabcolsep][\tabcolsep] & 0\% & --- & 82.9 & 90.4 & 46.7 & 68.6 & 64.7 & 64.2 & 50.0 & 90.8 & 68.3 & 15.4 & 58.2 & 63.7 \\
\rowcolor{prune25}[\tabcolsep][\tabcolsep] \cellcolor{modelB} &  & \textbf{\rcs} & 81.5 & 88.0 & 46.3 & 67.2 & 57.9 & 63.9 & 46.5 & 89.6 & 66.2 & 14.0 & 56.6 & 61.5 \\
\rowcolor{prune25}[\tabcolsep][\tabcolsep] \cellcolor{modelB} &  & \textbf{\rcsloo} & 83.0 & 89.1 & 46.0 & 67.6 & 58.3 & 64.2 & 47.3 & 90.2 & 65.7 & 13.4 & 56.4 & 61.9 \\
\rowcolor{prune25}[\tabcolsep][\tabcolsep] \cellcolor{modelB} & \multirow{-3}{*}{25\%} & \textbf{\method} & \textbf{88.9} & \textbf{89.5} & \textbf{47.3} & \textbf{68.4} & \textbf{59.2} & \textbf{64.3} & \textbf{48.2} & \textbf{90.8} & \textbf{66.9} & \textbf{15.8} & \textbf{57.8} & \textbf{63.4} \\
\rowcolor{prune50}[\tabcolsep][\tabcolsep] \cellcolor{modelB} &  & \textbf{\rcs} & 86.8 & 81.6 & 42.7 & 62.2 & 45.5 & 59.3 & 42.7 & 87.2 & 58.6 & 10.0 & 51.9 & 57.2 \\
\rowcolor{prune50}[\tabcolsep][\tabcolsep] \cellcolor{modelB} &  & \textbf{\rcsloo} & 91.6 & 83.7 & 44.7 & 64.2 & 46.1 & 59.6 & \textbf{44.7} & 89.0 & 58.2 & 10.0 & 52.4 & 58.6 \\
\rowcolor{prune50}[\tabcolsep][\tabcolsep] \cellcolor{modelB}\multirow{-7}{=}[1pt]{\rotatebox[origin=c]{90}{\textbf{Hy3}}} & \multirow{-3}{*}{50\%} & \textbf{\method} & \textbf{92.3} & \textbf{84.3} & \textbf{45.0} & \textbf{64.7} & \textbf{46.3} & \textbf{60.2} & 44.3 & \textbf{90.2} & \textbf{61.7} & \textbf{11.2} & \textbf{54.4} & \textbf{59.5} \\
\bottomrule
\end{tabular}%
}
\end{table}

\noindent\textbf{Obs 1. \method reduces predictive drift, but KL and likelihood
need not agree.}
Across the four matched GLM-4.7-Flash and Qwen3.6-35B-A3B settings, \method has lower reverse KL
than REAP in all four and than \rcsloo in three
(Table~\ref{tab:loo_ablation}). A separate routing-stratified comparison across
all four backbones illustrates the distinction between fidelity measures.
At 50\% removal, \method lowers KL relative to REAP in 73 of 80
model--domain-group--decile estimates, but keeps PPL closer to the original in
only 47. The corresponding \rcsloo counts are 72 and 47
(Appendix~\ref{app:ppl_view}).
Reduced distributional drift therefore need not imply better preservation of
reference-token likelihood.

\noindent\textbf{Obs 2. KL improvements are not uniform across domains.}
Of eight domains, the counts where \method lowers KL relative to REAP at
25\% and 50\% removal are, respectively, 5 and 7 for GLM-4.7-Flash, 6 and 5 for
Qwen3.6-35B-A3B, 6 and 8 for DeepSeek-V4-Flash-0731, and 6 and 7 for Hy3.
Yet Qwen3.6-35B-A3B's equal-domain held-out mean at 50\% is 0.2\% higher than REAP's,
and the relative-change interval includes zero. Aggregate gains thus coexist with domain-level
regressions. Appendix~\ref{app:domain_fidelity} gives the corresponding
\rcsloo results. These domain-wise and routing-stratified views reweight
overlapping predictions rather than provide independent replications.

\begin{figure}[!t]
\centering
\includegraphics[width=\textwidth]{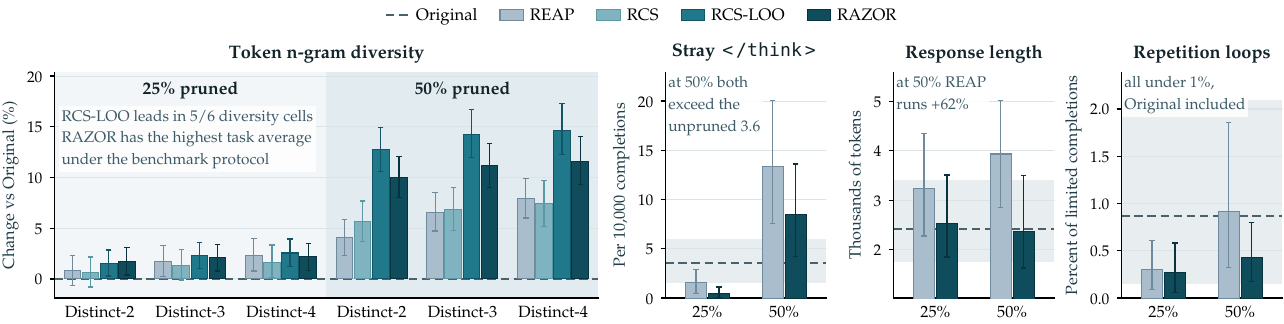}
\caption{\textbf{Qwen3.6-35B-A3B response diversity and response form.}
The left panel shows Distinct-$n$ changes from Original. The right panels show
completions containing \thinkclose per 10,000 LiveCodeBench completions, median token length,
and loop share among length-limited completions. Dashed lines mark Original,
with 95\% bands in the form panels. Whiskers show pointwise 95\% bootstrap intervals.}
\label{fig:diversity-grid}
\end{figure}

\subsection{Downstream Performance (RQ2)}
\label{sec:benchmarks}

Predictive fidelity on reference prefixes does not directly measure task
performance under decoding. Tables~\ref{tab:main} and~\ref{tab:additional} report the nine-task
suite, with unpruned references at 0\%. The first carries the baseline
comparison, and the second tests whether the ordering among the residual
criteria holds on two further routing architectures. Overall Avg
weights all nine tasks equally, while the two \textit{Avg} columns average
instruction-following and coding tasks. Macro gaps quoted below use unrounded
task-score averages.

\noindent\textbf{Obs 3. \method leads the baseline comparison in every setting
where baselines were run.}
Against Frequency, EAN, and REAP on GLM-4.7-Flash and Qwen3.6-35B-A3B,
\method has the highest macro score in all four model--budget settings.
Relative to REAP, it gains 2.12--5.59 points and wins all 36 paired task
comparisons. \rcsloo gains 0.80--3.01 points over REAP while losing three
of those 36 comparisons. DeepSeek-V4-Flash-0731 and Hy3 extend the comparison
among residual criteria, with the unpruned models as references and no
Frequency, EAN, or REAP benchmark runs. \method leads \rcsloo by 0.74--2.58 points
across all eight settings and wins 61 of 72 paired tasks against it, tying two
and losing nine, so the macro advantage does not imply a uniform taskwise
ordering. Appendix~\ref{app:benchmark_details}
gives the per-task breakdown and the scope of each comparison.

\noindent\textbf{Obs 4. Near-original performance at 25\% gives way to consistent
losses at 50\%.}
The 25\% \method checkpoints remain close to or above the originals in
macro score, whereas 50\% removal is lossy on all four backbones. For example,
GLM-4.7-Flash falls from 44.5 at 25\% to 40.1 at 50\%, against 44.8 unpruned.
Qwen3.6-35B-A3B falls from 62.1 to 58.7, against 61.6. These small gains above
the original macro scores do not establish that pruning improves the original
models.

\setcounter{topnumber}{1}

\begin{figure}[t]
\centering
\includegraphics[width=\textwidth]{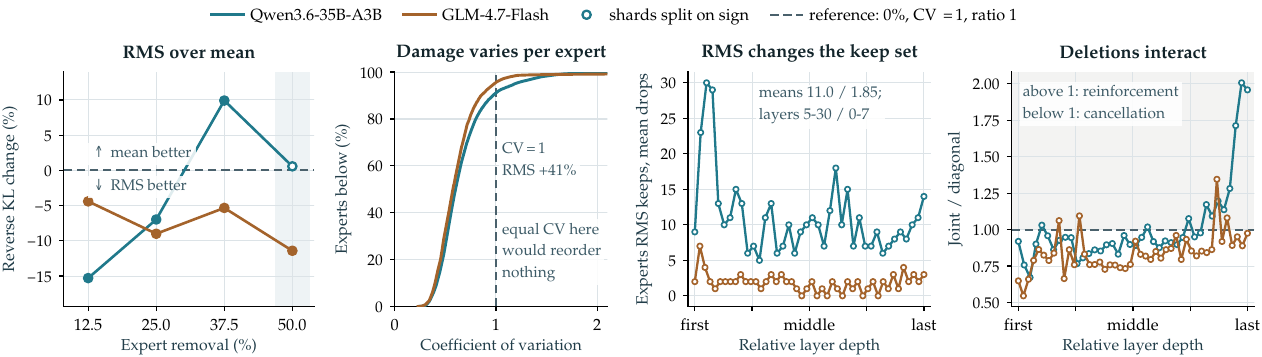}
\caption{\textbf{RMS aggregation and expert selection for \method.}
Left to right, panels show relative KL change from mean (negative favors RMS),
expert CV, RMS-only retained experts per layer, and deletion-set Gram ratios.
Shading marks 50\% removal, used in the last three panels.}
\label{fig:selection-diagnostics}
\end{figure}

\subsection{Generation Behavior (RQ3)}
\label{sec:generation}

A task score does not describe the diversity, format, length, or termination
of generated responses. We examine these properties on Qwen3.6-35B-A3B
(Figure~\ref{fig:diversity-grid}). Diversity compares all four criteria by
Distinct-$n$ change from Original ($n=2,3,4$) over 16 fixed responses per
question, weighting AIME'26, HumanEval+, and LiveCodeBench equally. Response
form compares REAP and \method on LiveCodeBench over 256 shared sample
positions per question (44,800 completions per variant). Sampling uses
temperature $0.7$, top-$p=0.95$, top-$k=20$, an 8,192-token cap, and thinking
disabled, so these results probe response behavior rather than explicit
reasoning traces. Diversity intervals use paired-question resampling and form
intervals resample questions within each variant, both conditioning on the
fixed generations
(Appendices~\ref{app:distinct_n_protocol} and~\ref{app:behaviour}).

\noindent\textbf{Obs 5. Higher diversity does not identify the best pruning
criterion.}
\rcsloo has the highest Distinct-$n$ in five of the six combinations of
removal budget and n-gram order, and exceeds REAP in all six by a median of
3.71 percentage points in relative change from Original. It does so despite
\method's stronger macro task performance. The exception is Distinct-2 at 25\%, where
\method changes by $+1.73\%$ from Original versus $+1.53\%$ for
\rcsloo. Changes are smaller at 25\% than at 50\%. In a separate four-task
analysis, Qwen3.6-35B-A3B's \rcsloo--REAP Distinct-4 gap reverses under a 128-token
prefix control that also changes question eligibility, so the reversal cannot
be attributed to length alone (Appendix~\ref{app:response_diversity}). Higher
diversity therefore establishes neither correctness nor a ranking robust to
this joint control.

\noindent\textbf{Obs 6. Closer response form does not imply preserved
termination rates.}
At 50\% removal, the rate of completions containing a stray \thinkclose
falls from REAP's 13.4 to \method's 8.5 per 10,000 completions,
whereas \rcsloo reaches 68.5
(Table~\ref{tab:behaviour-lcb}). Median length is also closer to
Original, at 2.37k tokens for \method versus 2.43k unpruned and 3.94k for REAP,
as plotted in Figure~\ref{fig:diversity-grid}. That table reports median
character lengths for the same responses. Their proportional changes differ
from those of median token lengths
(Appendix~\ref{app:behaviour}).
These response-form comparisons draw on a different sample pool from the
diversity result above and are not adjusted for length.

Length-limited finishes nevertheless reach 25.8\% for \method and 24.2\%
for REAP, against 16.0\% unpruned. Few capped responses satisfy the
high-repetition criterion. Among
length-limited completions, loop shares under the rightmost panel's criterion
(fewer than one quarter of token 4-grams unique) are 0.27\%/0.43\% for
\method and 0.30\%/0.91\% for REAP at 25\%/50\% removal, all below 1\%
(Appendix~\ref{app:behaviour}).

\subsection{Discussion (RQ4)}
\label{sec:discussion}

The checkpoint comparisons combine several scoring choices. We next vary
individual components to assess their effects on predictive fidelity and
expert selection.

\subsubsection{Scoring components}
\label{sec:main_results}

The three criteria separate components of the local deletion counterfactual.
\rcs scores the
routing-weighted consensus residual, \rcsloo adds survivor renormalization,
and \method adds router-selected replacement.
Table~\ref{tab:loo_ablation} compares them against REAP on
matched GLM-4.7-Flash and Qwen3.6-35B-A3B checkpoints at both budgets, and the
component contrasts of Appendix~\ref{app:matched_components} hold the
remaining choices fixed. The factorial study
(Appendix~\ref{app:factorial}) comes from a separate evaluation collection and
is therefore read on its own.

\noindent\textbf{Obs 7. Consensus subtraction helps, while counterfactual
refinements have mixed effects.}
Consensus subtraction lowers KL in all four model--budget settings when
routing weights and mean aggregation are held fixed, though the margin is
0.2\% in one of them, so the direction is consistent while the size is not
always meaningful. This matched contrast
isolates the reference change, unlike REAP-to-\rcs, which also changes aggregation. Subsequent
refinements are model-dependent. Under conditional RMS, adding the LOO factor
lowers KL by 0.6\%/5.6\% on GLM-4.7-Flash but raises it by 7.1\%/11.8\% on
Qwen3.6-35B-A3B at 25\%/50\% removal. Refill then improves on \rcsloo in three
of four settings, with GLM-4.7-Flash at 25\% the exception ($+0.5\%$ KL).
Thus a more complete single-deletion counterfactual need not produce a better
jointly pruned checkpoint. The choice of \method rests on the benchmark
comparisons as well as these fidelity results, rather than local exactness alone.

\subsubsection{Expert selection}
\label{sec:selection}

To relate aggregation to the retained experts,
Figure~\ref{fig:selection-diagnostics} compares RMS-versus-mean KL across four budgets,
per-expert coefficient of variation (CV), per-layer keep-set swaps, and
deletion-set Gram ratios, with Table~\ref{tab:cv_stats} placing \method and
\rcsloo statistics side by side. Hollow markers in the KL panel indicate
disagreement in sign across its four evaluation shards, and the latter three
panels use 50\% removal.

\noindent\textbf{Obs 8. RMS changes retained sets, with model-dependent fidelity
benefits.}
RMS yields lower KL than mean aggregation at all four budgets on GLM-4.7-Flash,
but higher KL above 25\% removal on Qwen3.6-35B-A3B. This differs from
\rcsloo, whose RMS-over-mean gains at 25\% and 50\% removal are small on Qwen3.6-35B-A3B
(Appendix~\ref{app:matched_components}).
Variation in the RMS premium $\sqrt{1+\mathrm{CV}^2}$ reweights experts rather
than uniformly rescaling them. At 50\% removal, RMS retains an average of
11.0 experts per layer that mean would discard on Qwen3.6-35B-A3B and 1.85 on
GLM-4.7-Flash, from retention budgets of 128 and 32, respectively. The per-layer
ranges are 5--30 and 0--7. This establishes changed selection,
not that the swaps cause benchmark gains. Appendix~\ref{app:aggregation} gives
further aggregation contrasts.

\noindent\textbf{Obs 9. Local diagnostics expose interactions that scalar
rankings do not resolve.}
The Gram-ratio panel shows net reinforcement in four GLM-4.7-Flash layers for
\method, whereas \rcsloo-selected sets show net cancellation in all 46 GLM-4.7-Flash layers
and 29 of 40 Qwen3.6-35B-A3B layers (Appendix~\ref{app:joint_numerator}).
These interactions concern the fixed-support numerator, constrained by
$\sum_i w_i(f_i-c)=0$. The denominator and refill remain outside this probe.
Single-deletion rankings also show no consistent advantage over REAP, despite
favorable \rcsloo comparisons in concentrated-routing bins
(Appendices~\ref{app:selection_conditions} and~\ref{app:per_expert_ablation}).
Thus the probes expose interactions without establishing optimal joint
selection or the cause of the benchmark gains.

\section{Related Work}
\label{sec:related}

MoE compression methods differ in the removal unit, the importance signal, and
the selection procedure applied afterwards. \method varies only the second,
removing whole experts under a fixed layerwise budget without updating
the learned expert or router weights. We compare it with methods that score
whole experts independently, model relations among experts, or search over
retained sets.
Appendix~\ref{app:related_work} treats the families that change the removal
unit or the surviving computation, including merging, router adaptation, budget
allocation, and finer-grained compression.

\paragraph{Expert scoring.}
At the individual-expert level, pruning methods use routing, activation, or
weight statistics \citep{muzio2024seer,fantasticexperts2025,aimer2026}.
REAP analyzes promoted substitution and survivor renormalization, but its
pruning score averages the removed expert's routing-weighted output norm over
tokens selecting that expert \citep{reap2026}. A unified formulation separates
the removed contribution from a rerouting residual and adopts a magnitude
proxy, motivated by the cost of direct rerouting evaluation and the high overlap
between proxy-based and exact-damage selections in its experiments
\citep{unified2026}. Its score family varies routing frequency, gate weighting,
and activation strength, recommending gate-free conditional-mean activation
scores for task-agnostic pruning. Under our fixed-input routing assumptions,
consensus residuals express the refill-aware perturbation using the routed
outputs and one additional expert output per token. This avoids constructing
a separate pruned mixture for every candidate deletion. Conditional RMS
aggregation remains a separate design choice.

\paragraph{Expert relations.}
Independent scores can miss structure shared across experts. STUN and HC-SMoE
cluster router behavior or expert outputs, while SHAPE, ConMoE, and MAESTRO use
routing co-occurrence, parameter-space replaceability, or cross-layer
transitions \citep{stun2024,hcsmoe2024,shape2026,conmoe2026,maestro2026}.
These relations support selection, merging, or path modeling, whereas \method
scores deletion with survivor renormalization and refill, without recovery.

\paragraph{Candidate-set search.}
Beyond relational ranking, other methods evaluate retained sets through
reconstruction, joint pruning--merging, or layerwise and blockwise search
\citep{lu2024notallexperts,eep2024,moeii2024}. \method performs no subset
search. Its exactness for one fixed-input deletion does not solve joint pruning,
because deletions interact and a refill candidate may itself be removed.

\section{Conclusion}

\method uses consensus residuals to score expert replaceability under
survivor renormalization and router refill. The resulting single-deletion
identities support training-free pruning from forward computation.
Across two backbones and two removal budgets, \method achieves the highest
nine-task macro average among the evaluated baselines and residual criteria,
and lowers reverse KL relative to REAP in all four settings. It also achieves
the highest macro task score among the residual criteria on two additional
backbones. These results support including survivor and replacement outputs
in expert scoring, while leaving
the connection between local damage and joint pruning empirical.
In particular, more complete local counterfactuals do not consistently improve
fidelity, and residual criteria yield less stable selections than REAP in the
Qwen3.6-35B-A3B calibration-budget study
(Appendix~\ref{app:calibration_budget}). Changes in
diversity, formatting, and termination further show why task performance and
response behavior require separate evaluation.

\paragraph{Limitations.}
Joint pruning can remove the refill candidate assumed by a single-deletion
score. Interacting deletions and changes propagated across layers also remain
outside the local counterfactual. The benchmark comparisons cover fixed
checkpoints and report no task-level uncertainty. Physical pruning reduces
stored parameters, but retaining the original top-$k$ does not imply a
proportional reduction in active computation. Scoring runtime, serving
latency, energy use, and peak memory have not been measured.

\section*{AI use statement}

Generative AI tools assisted with the refinement of academic prose, focusing
on clarity, coherence, and precision of expression. They also supported
literature search and coding.

\section*{Ethics statement}

This work studies expert pruning of post-trained MoE language models, aiming
to reduce model storage requirements while retaining their capabilities.
Removing experts can also alter behaviors shaped by post-training, including
safety-related responses, even when aggregate task performance is preserved.
Agreement with the original model on safety-related inputs measures predictive
fidelity, but does not establish safety or preservation of alignment. We
therefore emphasize independent safety and alignment evaluation before
deployment, particularly in safety-sensitive applications. The study uses
previously released calibration data (Appendix~\ref{app:razorcal}) and involves
no human-subject research.

\section*{Reproducibility statement}

Section~\ref{sec:method} and Appendix~\ref{app:geometry} provide the scoring
criterion and derivations. Appendix~\ref{app:implementation} gives the algorithm
and implementation details. Appendix~\ref{app:evaluation} documents calibration
data, model configurations, evaluation protocols, score aggregation, and
uncertainty estimates. Appendix~\ref{app:benchmark_details} states the
evaluator settings and the scope of each cross-study comparison. The scoring
and pruning implementation, the \calib calibration corpus, and the eight pruned
checkpoints evaluated here are publicly available. Each checkpoint includes
the retained expert indices. The public calibration release replaces
potentially identifying strings with placeholders, so it need not reproduce
the original expert selection exactly.

\bibliography{razor}

\begin{thebibliography}{46}
\providecommand{\natexlab}[1]{#1}
\providecommand{\url}[1]{\texttt{#1}}
\expandafter\ifx\csname urlstyle\endcsname\relax
  \providecommand{\doi}[1]{doi: #1}\else
  \providecommand{\doi}{doi: \begingroup \urlstyle{rm}\Url}\fi

\bibitem[Bai et~al.(2025)Bai, Tu, Zhang, Peng, Wang, Lv, Cao, Xu, Hou, Dong,
  Tang, and Li]{longbench2024}
Yushi Bai, Shangqing Tu, Jiajie Zhang, Hao Peng, Xiaozhi Wang, Xin Lv, Shulin
  Cao, Jiazheng Xu, Lei Hou, Yuxiao Dong, Jie Tang, and Juanzi Li.
\newblock {LongBench} v2: Towards deeper understanding and reasoning on
  realistic long-context multitasks.
\newblock In \emph{Proceedings of the 63rd Annual Meeting of the Association
  for Computational Linguistics (Volume 1: Long Papers)}, pp.\  3639--3664.
  Association for Computational Linguistics, 2025.
\newblock \doi{10.18653/v1/2025.acl-long.183}.
\newblock URL \url{https://aclanthology.org/2025.acl-long.183/}.

\bibitem[Chen et~al.(2025)Chen, Liu, Sun, Chao, Hsu, and Lee]{hcsmoe2024}
I-Chun Chen, Hsu-Shen Liu, Wei-Fang Sun, Chen-Hao Chao, Yen-Chang Hsu, and
  Chun-Yi Lee.
\newblock Retraining-free merging of sparse {MoE} via hierarchical clustering.
\newblock In \emph{Proceedings of the 42nd International Conference on Machine
  Learning}, pp.\  8594--8620. PMLR, 2025.
\newblock URL \url{https://proceedings.mlr.press/v267/chen25aq.html}.

\bibitem[Chowdhury et~al.(2024)Chowdhury, Wang, El~Maghraoui, Wang, Chen, and
  Carothers]{provablepruning2024}
Mohammed Nowaz~Rabbani Chowdhury, Meng Wang, Kaoutar El~Maghraoui, Naigang
  Wang, Pin-Yu Chen, and Christopher Carothers.
\newblock A provably effective method for pruning experts in fine-tuned sparse
  mixture-of-experts.
\newblock In \emph{Proceedings of the 41st International Conference on Machine
  Learning}, volume 235 of \emph{Proceedings of Machine Learning Research},
  pp.\  8815--8847, 2024.

\bibitem[{DeepSeek-AI}(2026)]{deepseekv4_2026}
{DeepSeek-AI}.
\newblock {DeepSeek-V4}: Towards highly efficient million-token context
  intelligence.
\newblock \emph{arXiv preprint arXiv:2606.19348}, 2026.
\newblock URL \url{https://arxiv.org/abs/2606.19348}.

\bibitem[Du et~al.(2025)Du, Yao, Ma, Wang, Zheng, Zhu, Liu, Liang, Jin, Wei,
  Zheng, Deng, Gavin, Jia, Jiang, Liao, Li, Li, Li, Li, Li, Ma, Ni, Que, Wang,
  Wen, Wu, Hsing, Xu, Yang, Wang, Zhou, Bai, Bu, Cai, Chen, Chen, Cheng, Cheng,
  Ding, Huang, Huang, Li, Li, Li, Liang, Lin, Lin, Ma, Pang, Peng, Peng, Qi,
  Qiu, Qu, Quan, Tan, Wang, Wang, Wang, Wang, Wang, Xu, Yang, Yuan, Yue, Zhan,
  Zhang, Zhang, Zhang, Zhang, Zhang, Zhao, Zheng, Zhong, Gao, Li, Liu, Liu,
  Liu, Ni, Peng, Qin, Su, Wang, Wang, Yang, Yang, Cao, Yue, Zhang, Zhou, Liu,
  Lin, Huang, and Zhang]{supergpqa2025}
Xinrun Du, Yifan Yao, Kaijing Ma, Bingli Wang, Tianyu Zheng, King Zhu, Minghao
  Liu, Yiming Liang, Xiaolong Jin, Zhenlin Wei, Chujie Zheng, Kaixin Deng,
  Shawn Gavin, Shian Jia, Sichao Jiang, Yiyan Liao, Rui Li, Qinrui Li, Sirun
  Li, Yizhi Li, Yunwen Li, David Ma, Yuansheng Ni, Haoran Que, Qiyao Wang,
  Zhoufutu Wen, Siwei Wu, Tyshawn Hsing, Ming Xu, Zhenzhu Yang, Zekun~Moore
  Wang, Junting Zhou, Yuelin Bai, Xingyuan Bu, Chenglin Cai, Liang Chen, Yifan
  Chen, Chengtuo Cheng, Tianhao Cheng, Keyi Ding, Siming Huang, Yun Huang,
  Yaoru Li, Yizhe Li, Zhaoqun Li, Tianhao Liang, Chengdong Lin, Hongquan Lin,
  Yinghao Ma, Tianyang Pang, Zhongyuan Peng, Zifan Peng, Qige Qi, Shi Qiu,
  Xingwei Qu, Shanghaoran Quan, Yizhou Tan, Zili Wang, Chenqing Wang, Hao Wang,
  Yiya Wang, Yubo Wang, Jiajun Xu, Kexin Yang, Ruibin Yuan, Yuanhao Yue,
  Tianyang Zhan, Chun Zhang, Jinyang Zhang, Xiyue Zhang, Xingjian Zhang, Yue
  Zhang, Yongchi Zhao, Xiangyu Zheng, Chenghua Zhong, Yang Gao, Zhoujun Li,
  Dayiheng Liu, Qian Liu, Tianyu Liu, Shiwen Ni, Junran Peng, Yujia Qin, Wenbo
  Su, Guoyin Wang, Shi Wang, Jian Yang, Min Yang, Meng Cao, Xiang Yue,
  Zhaoxiang Zhang, Wangchunshu Zhou, Jiaheng Liu, Qunshu Lin, Wenhao Huang, and
  Ge~Zhang.
\newblock {SuperGPQA}: Scaling {LLM} evaluation across 285 graduate
  disciplines.
\newblock In \emph{Advances in Neural Information Processing Systems}, volume
  38, Main Conference. Curran Associates, Inc., 2025.
\newblock \doi{10.52202/085713-3766}.
\newblock URL
  \url{https://proceedings.neurips.cc/paper_files/paper/2025/file/a3c5af1f56fc73eef1ba0f442739f5ca-Paper-Datasets_and_Benchmarks_Track.pdf}.

\bibitem[Fedus et~al.(2022)Fedus, Zoph, and Shazeer]{fedus2022switch}
William Fedus, Barret Zoph, and Noam Shazeer.
\newblock Switch transformers: Scaling to trillion parameter models with simple
  and efficient sparsity.
\newblock \emph{Journal of Machine Learning Research}, 23\penalty0
  (120):\penalty0 1--39, 2022.
\newblock URL \url{https://jmlr.org/papers/v23/21-0998.html}.

\bibitem[{GLM-4.5 Team}(2025)]{glm45_2025}
{GLM-4.5 Team}.
\newblock {GLM-4.5}: Agentic, reasoning, and coding ({ARC}) foundation models.
\newblock \emph{arXiv preprint arXiv:2508.06471}, 2025.
\newblock URL \url{https://arxiv.org/abs/2508.06471}.

\bibitem[Goel et~al.(2026)Goel, Maheshwari, and Chakraborty]{maestro2026}
Palaash Goel, Ayush Maheshwari, and Tanmoy Chakraborty.
\newblock It takes a {MAESTRO} to prune bad experts.
\newblock \emph{arXiv preprint arXiv:2607.08601}, 2026.
\newblock URL \url{https://arxiv.org/abs/2607.08601}.

\bibitem[He et~al.(2026)He, Zou, Jiang, Ding, Qu, Li, and
  Miller]{fishermoe2026}
Haoze He, Xinkai Zou, Xuan Jiang, Xingyuan Ding, Ao~Qu, Juncheng~Billy Li, and
  Heather Miller.
\newblock Less is {MoE}: Trimming experts in domain-specialist language models.
\newblock \emph{arXiv preprint arXiv:2606.05538}, 2026.

\bibitem[{Hugging Face Smol Models Research}(2025)]{smoltalk2}
{Hugging Face Smol Models Research}.
\newblock {SmolTalk2}.
\newblock Hugging Face dataset card, 2025.
\newblock URL \url{https://huggingface.co/datasets/HuggingFaceTB/smoltalk2}.
\newblock Accessed 2026-09-14.

\bibitem[Hyeon \& Do(2026)Hyeon and Do]{routercalib2026}
Sieun Hyeon and Jaeyoung Do.
\newblock Is retraining-free enough? the necessity of router calibration for
  efficient {MoE} compression.
\newblock \emph{arXiv preprint arXiv:2603.02217}, 2026.
\newblock URL \url{https://arxiv.org/abs/2603.02217}.

\bibitem[Jain et~al.(2025)Jain, Han, Gu, Li, Yan, Zhang, Wang, Solar-Lezama,
  Sen, and Stoica]{livecodebench2024}
Naman Jain, King Han, Alex Gu, Wen-Ding Li, Fanjia Yan, Tianjun Zhang, Sida~I.
  Wang, Armando Solar-Lezama, Koushik Sen, and Ion Stoica.
\newblock {LiveCodeBench}: Holistic and contamination free evaluation of large
  language models for code.
\newblock In \emph{The Thirteenth International Conference on Learning
  Representations}, 2025.
\newblock URL
  \url{https://proceedings.iclr.cc/paper_files/paper/2025/hash/94074dd5a072d28ff75a76dabed43767-Abstract-Conference.html}.

\bibitem[Jaiswal et~al.(2025)Jaiswal, Wang, Li, Li, Chen, Wang, Wang, Pang, and
  Du]{fantasticexperts2025}
Ajay Jaiswal, Jianyu Wang, Yixiao Li, Pingzhi Li, Tianlong Chen, Zhangyang
  Wang, Chong Wang, Ruoming Pang, and Xianzhi Du.
\newblock Finding fantastic experts in {MoEs}: A unified study for expert
  dropping strategies and observations.
\newblock \emph{arXiv preprint arXiv:2504.05586}, 2025.
\newblock URL \url{https://arxiv.org/abs/2504.05586}.

\bibitem[Jha(2026)]{halfexperts2026}
Anik Jha.
\newblock Half the experts, all the code: One-shot domain pruning of
  mixture-of-experts {LLMs} for coding.
\newblock \emph{arXiv preprint arXiv:2607.16721}, 2026.

\bibitem[Jha et~al.(2026)Jha, Hashemzadeh, Pasand, Parviz, Lee, and
  Knyazev]{ream2026}
Saurav Jha, Maryam Hashemzadeh, Ali~Saheb Pasand, Ali Parviz, Min-Joong Lee,
  and Boris Knyazev.
\newblock {REAM}: Merging improves pruning of experts in {LLMs}.
\newblock \emph{arXiv preprint arXiv:2604.04356}, 2026.
\newblock URL \url{https://arxiv.org/abs/2604.04356}.

\bibitem[Jimenez et~al.(2024)Jimenez, Yang, Wettig, Yao, Pei, Press, and
  Narasimhan]{swebench2024}
Carlos~E. Jimenez, John Yang, Alexander Wettig, Shunyu Yao, Kexin Pei, Ofir
  Press, and Karthik~R. Narasimhan.
\newblock {SWE-bench}: Can language models resolve real-world {GitHub} issues?
\newblock In \emph{The Twelfth International Conference on Learning
  Representations}, 2024.
\newblock URL \url{https://openreview.net/forum?id=VTF8yNQM66}.

\bibitem[Lasby et~al.(2026)Lasby, Lazarevich, Sinnadurai, Lie, Ioannou, and
  Thangarasa]{reap2026}
Mike Lasby, Ivan Lazarevich, Nish Sinnadurai, Sean Lie, Yani Ioannou, and
  Vithursan Thangarasa.
\newblock {REAP} the experts: Why pruning prevails for one-shot {MoE}
  compression.
\newblock In \emph{The Fourteenth International Conference on Learning
  Representations}, 2026.
\newblock URL \url{https://openreview.net/forum?id=ukGxWd2aDG}.

\bibitem[Lee et~al.(2025)Lee, Hwang, Qiao, Campos, Yao, and He]{stun2024}
Jaeseong Lee, Seung-won Hwang, Aurick Qiao, Daniel~F Campos, Zhewei Yao, and
  Yuxiong He.
\newblock {STUN}: Structured-then-unstructured pruning for scalable {MoE}
  pruning.
\newblock In \emph{Proceedings of the 63rd Annual Meeting of the Association
  for Computational Linguistics (Volume 1: Long Papers)}, pp.\  13660--13676.
  Association for Computational Linguistics, July 2025.
\newblock \doi{10.18653/v1/2025.acl-long.671}.
\newblock URL \url{https://aclanthology.org/2025.acl-long.671/}.

\bibitem[Lepikhin et~al.(2021)Lepikhin, Lee, Xu, Chen, Firat, Huang, Krikun,
  Shazeer, and Chen]{lepikhin2021gshard}
Dmitry Lepikhin, HyoukJoong Lee, Yuanzhong Xu, Dehao Chen, Orhan Firat, Yanping
  Huang, Maxim Krikun, Noam Shazeer, and Zhifeng Chen.
\newblock {GShard}: Scaling giant models with conditional computation and
  automatic sharding.
\newblock In \emph{International Conference on Learning Representations}, 2021.
\newblock URL \url{https://openreview.net/forum?id=qrwe7XHTmYb}.

\bibitem[Lewkowycz et~al.(2022)Lewkowycz, Andreassen, Dohan, Dyer, Michalewski,
  Ramasesh, Slone, Anil, Schlag, Gutman-Solo, Wu, Neyshabur, Gur-Ari, and
  Misra]{minerva2022}
Aitor Lewkowycz, Anders Andreassen, David Dohan, Ethan Dyer, Henryk
  Michalewski, Vinay Ramasesh, Ambrose Slone, Cem Anil, Imanol Schlag, Theo
  Gutman-Solo, Yuhuai Wu, Behnam Neyshabur, Guy Gur-Ari, and Vedant Misra.
\newblock Solving quantitative reasoning problems with language models.
\newblock In \emph{Advances in Neural Information Processing Systems},
  volume~35, pp.\  3843--3857. Curran Associates, Inc., 2022.
\newblock \doi{10.52202/068431-0278}.
\newblock URL
  \url{https://proceedings.neurips.cc/paper_files/paper/2022/file/18abbeef8cfe9203fdf9053c9c4fe191-Paper-Conference.pdf}.

\bibitem[Li et~al.(2016)Li, Galley, Brockett, Gao, and Dolan]{li2016diversity}
Jiwei Li, Michel Galley, Chris Brockett, Jianfeng Gao, and Bill Dolan.
\newblock A diversity-promoting objective function for neural conversation
  models.
\newblock In \emph{Proceedings of the 2016 Conference of the North American
  Chapter of the Association for Computational Linguistics: Human Language
  Technologies}, pp.\  110--119. Association for Computational Linguistics,
  2016.
\newblock \doi{10.18653/v1/N16-1014}.
\newblock URL \url{https://aclanthology.org/N16-1014/}.

\bibitem[Li et~al.(2026)Li, Yang, Zhou, Xue, Jiang, and Wang]{heapr2025}
Ke~Li, Zheng Yang, Zhongbin Zhou, Feng Xue, Zhonglin Jiang, and Wenxiao Wang.
\newblock {HEAPr}: Hessian-based efficient atomic expert pruning in output
  space.
\newblock In \emph{The Fourteenth International Conference on Learning
  Representations}, 2026.
\newblock URL \url{https://openreview.net/forum?id=JAbMgS7gl6}.

\bibitem[Li et~al.(2024)Li, Zhang, Yadav, Sung, Cheng, Bansal, and
  Chen]{mcsmoe2024}
Pingzhi Li, Zhenyu Zhang, Prateek Yadav, Yi-Lin Sung, Yu~Cheng, Mohit Bansal,
  and Tianlong Chen.
\newblock Merge, then compress: Demystify efficient {SMoE} with hints from its
  routing policy.
\newblock In \emph{The Twelfth International Conference on Learning
  Representations}, 2024.
\newblock URL \url{https://openreview.net/forum?id=eFWG9Cy3WK}.

\bibitem[Liu et~al.(2024)Liu, Zhu, Lin, Ning, Blaschko, Yan, Dai, Yang, and
  Wang]{eep2024}
Enshu Liu, Junyi Zhu, Zinan Lin, Xuefei Ning, Matthew~B. Blaschko, Shengen Yan,
  Guohao Dai, Huazhong Yang, and Yu~Wang.
\newblock Efficient expert pruning for sparse mixture-of-experts language
  models: Enhancing performance and reducing inference costs.
\newblock \emph{arXiv preprint arXiv:2407.00945}, 2024.
\newblock URL \url{https://arxiv.org/abs/2407.00945}.

\bibitem[Liu et~al.(2023)Liu, Xia, Wang, and Zhang]{evalplus2023}
Jiawei Liu, Chunqiu~Steven Xia, Yuyao Wang, and Lingming Zhang.
\newblock Is your code generated by {ChatGPT} really correct? rigorous
  evaluation of large language models for code generation.
\newblock In \emph{Advances in Neural Information Processing Systems},
  volume~36, pp.\  21558--21572. Curran Associates, Inc., 2023.
\newblock \doi{10.52202/075280-0943}.
\newblock URL
  \url{https://proceedings.neurips.cc/paper_files/paper/2023/file/43e9d647ccd3e4b7b5baab53f0368686-Paper-Conference.pdf}.

\bibitem[Liu et~al.(2026{\natexlab{a}})Liu, Chen, Tang, Shen, Wang, and
  Yuan]{aimer2026}
Zongfang Liu, Guangyi Chen, Shengkun Tang, Yifan Shen, Huan Wang, and Xin Yuan.
\newblock {AIMER}: Calibration-free task-agnostic {MoE} expert pruning.
\newblock \emph{arXiv preprint arXiv:2603.18492}, 2026{\natexlab{a}}.
\newblock URL \url{https://arxiv.org/abs/2603.18492}.

\bibitem[Liu et~al.(2026{\natexlab{b}})Liu, Zhang, Ma, Chen, and
  Yuan]{unified2026}
Zongfang Liu, Jinghui Zhang, Zijian Ma, Guangyi Chen, and Xin Yuan.
\newblock How to score experts for one-shot {MoE} expert pruning: A unified
  formulation and selection principle.
\newblock \emph{arXiv preprint arXiv:2606.15716}, 2026{\natexlab{b}}.
\newblock URL \url{https://arxiv.org/abs/2606.15716}.

\bibitem[Lu et~al.(2024)Lu, Liu, Xu, Zhou, Huang, Zhang, Yan, and
  Li]{lu2024notallexperts}
Xudong Lu, Qi~Liu, Yuhui Xu, Aojun Zhou, Siyuan Huang, Bo~Zhang, Junchi Yan,
  and Hongsheng Li.
\newblock Not all experts are equal: Efficient expert pruning and skipping for
  mixture-of-experts large language models.
\newblock In \emph{Proceedings of the 62nd Annual Meeting of the Association
  for Computational Linguistics (Volume 1: Long Papers)}, pp.\  6159--6172.
  Association for Computational Linguistics, August 2024.
\newblock \doi{10.18653/v1/2024.acl-long.334}.
\newblock URL \url{https://aclanthology.org/2024.acl-long.334/}.

\bibitem[Mao et~al.(2025)Mao, Tsao, Zhou, Patil, and Gonzalez]{bfclv4}
Huanzhi Mao, Raymond Tsao, Jingzhuo Zhou, Shishir~G. Patil, and Joseph~E.
  Gonzalez.
\newblock {BFCL V4}: Agentic part 1: Web search, July 2025.
\newblock URL
  \url{https://gorilla.cs.berkeley.edu/blogs/15_bfcl_v4_web_search.html}.

\bibitem[{Mathematical Association of America}(2026)]{maaAIME}
{Mathematical Association of America}.
\newblock American invitational mathematics examination ({AIME}) 2026.
\newblock Competition, 2026.
\newblock URL
  \url{https://maa.org/math-competitions/american-invitational-mathematics-examination-aime}.

\bibitem[Merity et~al.(2016)Merity, Xiong, Bradbury, and Socher]{wikitext2017}
Stephen Merity, Caiming Xiong, James Bradbury, and Richard Socher.
\newblock Pointer sentinel mixture models.
\newblock \emph{arXiv preprint arXiv:1609.07843}, 2016.
\newblock URL \url{https://arxiv.org/abs/1609.07843}.

\bibitem[Muzio et~al.(2024)Muzio, Sun, and He]{muzio2024seer}
Alexandre Muzio, Alex Sun, and Churan He.
\newblock {SEER-MoE}: Sparse expert efficiency through regularization for
  mixture-of-experts.
\newblock \emph{arXiv preprint arXiv:2404.05089}, 2024.
\newblock URL \url{https://arxiv.org/abs/2404.05089}.

\bibitem[{NVIDIA}(2026)]{nemotron3super2026}
{NVIDIA}.
\newblock {Nemotron 3 Super}: Open, efficient mixture-of-experts hybrid
  {Mamba-Transformer} model for agentic reasoning.
\newblock \emph{arXiv preprint arXiv:2604.12374}, 2026.
\newblock URL \url{https://arxiv.org/abs/2604.12374}.

\bibitem[Patil et~al.(2025)Patil, Mao, Yan, Ji, Suresh, Stoica, and
  Gonzalez]{bfcl2025}
Shishir~G Patil, Huanzhi Mao, Fanjia Yan, Charlie Cheng-Jie Ji, Vishnu Suresh,
  Ion Stoica, and Joseph~E. Gonzalez.
\newblock The berkeley function calling leaderboard ({BFCL}): From tool use to
  agentic evaluation of large language models.
\newblock In \emph{Proceedings of the 42nd International Conference on Machine
  Learning}, volume 267 of \emph{Proceedings of Machine Learning Research},
  pp.\  48371--48392. PMLR, 2025.
\newblock URL \url{https://proceedings.mlr.press/v267/patil25a.html}.

\bibitem[Pyatkin et~al.(2025)Pyatkin, Malik, Graf, Ivison, Huang, Dasigi,
  Lambert, and Hajishirzi]{ifbench2025}
Valentina Pyatkin, Saumya Malik, Victoria Graf, Hamish Ivison, Shengyi Huang,
  Pradeep Dasigi, Nathan Lambert, and Hannaneh Hajishirzi.
\newblock Generalizing verifiable instruction following.
\newblock In \emph{Advances in Neural Information Processing Systems}, volume
  38, Main Conference. Curran Associates, Inc., 2025.
\newblock \doi{10.52202/085713-1645}.
\newblock URL
  \url{https://proceedings.neurips.cc/paper_files/paper/2025/file/46499a0622ecf568b72d17b61e45dbd5-Paper-Datasets_and_Benchmarks_Track.pdf}.

\bibitem[{Qwen Team}(2026)]{qwen35_2026}
{Qwen Team}.
\newblock {Qwen3.5}: Towards native multimodal agents, February 2026.
\newblock URL \url{https://huggingface.co/Qwen/Qwen3.5-35B-A3B#citation}.

\bibitem[Shazeer et~al.(2017)Shazeer, Mirhoseini, Maziarz, Davis, Le, Hinton,
  and Dean]{shazeer2017}
Noam Shazeer, Azalia Mirhoseini, Krzysztof Maziarz, Andy Davis, Quoc Le,
  Geoffrey Hinton, and Jeff Dean.
\newblock Outrageously large neural networks: The sparsely-gated
  mixture-of-experts layer.
\newblock In \emph{International Conference on Learning Representations}, 2017.
\newblock URL \url{https://openreview.net/forum?id=B1ckMDqlg}.

\bibitem[{Team Olmo} et~al.(2025){Team Olmo}, Ettinger, Bertsch, Kuehl, Graham,
  Heineman, Groeneveld, Brahman, Timbers, Ivison, Morrison, Poznanski, Lo,
  Soldaini, Jordan, Chen, Noukhovitch, Lambert, Walsh, Dasigi, Berry, Malik,
  Shah, Geng, Arora, Gupta, Anderson, Xiao, Murray, Romero, Graf, Asai, Bhagia,
  Wettig, Liu, Rangapur, Anastasiades, Huang, Schwenk, Trivedi, Magnusson,
  Lochner, Liu, Miranda, Sap, Morgan, Schmitz, Guerquin, Wilson, Huff, Bras,
  Xin, Shao, Skjonsberg, Shen, Li, Wilde, Pyatkin, Merrill, Chang, Gu, Zeng,
  Sabharwal, Zettlemoyer, Koh, Farhadi, Smith, and Hajishirzi]{olmo3_2025}
{Team Olmo}, Allyson Ettinger, Amanda Bertsch, Bailey Kuehl, David Graham,
  David Heineman, Dirk Groeneveld, Faeze Brahman, Finbarr Timbers, Hamish
  Ivison, Jacob Morrison, Jake Poznanski, Kyle Lo, Luca Soldaini, Matt Jordan,
  Mayee Chen, Michael Noukhovitch, Nathan Lambert, Pete Walsh, Pradeep Dasigi,
  Robert Berry, Saumya Malik, Saurabh Shah, Scott Geng, Shane Arora, Shashank
  Gupta, Taira Anderson, Teng Xiao, Tyler Murray, Tyler Romero, Victoria Graf,
  Akari Asai, Akshita Bhagia, Alexander Wettig, Alisa Liu, Aman Rangapur, Chloe
  Anastasiades, Costa Huang, Dustin Schwenk, Harsh Trivedi, Ian Magnusson,
  Jaron Lochner, Jiacheng Liu, Lester James~V. Miranda, Maarten Sap, Malia
  Morgan, Michael Schmitz, Michal Guerquin, Michael Wilson, Regan Huff,
  Ronan~Le Bras, Rui Xin, Rulin Shao, Sam Skjonsberg, Shannon~Zejiang Shen,
  Shuyue~Stella Li, Tucker Wilde, Valentina Pyatkin, Will Merrill, Yapei Chang,
  Yuling Gu, Zhiyuan Zeng, Ashish Sabharwal, Luke Zettlemoyer, Pang~Wei Koh,
  Ali Farhadi, Noah~A. Smith, and Hannaneh Hajishirzi.
\newblock {Olmo 3}, 2025.
\newblock URL \url{https://arxiv.org/abs/2512.13961}.
\newblock Dataset card:
  \url{https://huggingface.co/datasets/allenai/Dolci-Instruct-SFT}.

\bibitem[Teknium(2023)]{openhermes25}
Teknium.
\newblock {OpenHermes 2.5}: An open dataset of synthetic data for generalist
  {LLM} assistants.
\newblock Hugging Face dataset card, 2023.
\newblock URL \url{https://huggingface.co/datasets/teknium/OpenHermes-2.5}.

\bibitem[{Tencent Hy Team}(2026)]{hy3_2026}
{Tencent Hy Team}.
\newblock {Hy3}.
\newblock Hugging Face model card, 2026.
\newblock URL \url{https://huggingface.co/tencent/Hy3}.

\bibitem[Xie et~al.(2024)Xie, Zhang, Zhou, Xie, Song, Liu, Wang, Lin, and
  Xu]{moepruner2024}
Yanyue Xie, Zhi Zhang, Ding Zhou, Cong Xie, Ziang Song, Xin Liu, Yanzhi Wang,
  Xue Lin, and An~Xu.
\newblock {MoE-Pruner}: Pruning mixture-of-experts large language model using
  the hints from its router.
\newblock \emph{arXiv preprint arXiv:2410.12013}, 2024.
\newblock URL \url{https://arxiv.org/abs/2410.12013}.

\bibitem[Yang et~al.(2024)Yang, Sui, Xiao, Huang, Gong, Duan, Jia, Yin, Cheng,
  and Yuan]{moeii2024}
Cheng Yang, Yang Sui, Jinqi Xiao, Lingyi Huang, Yu~Gong, Yuanlin Duan, Wenqi
  Jia, Miao Yin, Yu~Cheng, and Bo~Yuan.
\newblock {MoE-I\textsuperscript{2}}: Compressing mixture of experts models
  through inter-expert pruning and intra-expert low-rank decomposition.
\newblock In \emph{Findings of the Association for Computational Linguistics:
  EMNLP 2024}, pp.\  10456--10466. Association for Computational Linguistics,
  2024.
\newblock \doi{10.18653/v1/2024.findings-emnlp.612}.
\newblock URL \url{https://aclanthology.org/2024.findings-emnlp.612/}.

\bibitem[Yao et~al.(2026)Yao, Pan, Dai, Cong, Li, and Yang]{conmoe2026}
Yilun Yao, Jiaming Pan, Elsie Dai, Peizhuang Cong, Yaoming Li, and Tong Yang.
\newblock {ConMoE}: Expert-pool consolidation via prototype reassignment for
  {MoE} compression.
\newblock \emph{arXiv preprint arXiv:2605.29350}, 2026.
\newblock URL \url{https://arxiv.org/abs/2605.29350}.

\bibitem[Zhang(2026)]{shape2026}
Yuhao Zhang.
\newblock {SHAPE}: Coalition-aware expert pruning for sparse mixture-of-experts
  {LLMs}.
\newblock \emph{arXiv preprint arXiv:2606.09886}, 2026.
\newblock URL \url{https://arxiv.org/abs/2606.09886}.

\bibitem[Zhang et~al.(2026)Zhang, Ghosh, Liu, Yu, and Liu]{globalprune2026}
Zeliang Zhang, Nikhil Ghosh, Jiani Liu, Bin Yu, and Xiaodong Liu.
\newblock Does a global perspective help prune sparse {MoEs} elegantly?
\newblock \emph{arXiv preprint arXiv:2604.06542}, 2026.
\newblock URL \url{https://arxiv.org/abs/2604.06542}.

\bibitem[Zhou et~al.(2023)Zhou, Lu, Mishra, Brahma, Basu, Luan, Zhou, and
  Hou]{ifeval2023}
Jeffrey Zhou, Tianjian Lu, Swaroop Mishra, Siddhartha Brahma, Sujoy Basu,
  Yi~Luan, Denny Zhou, and Le~Hou.
\newblock Instruction-following evaluation for large language models.
\newblock \emph{arXiv preprint arXiv:2311.07911}, 2023.
\newblock URL \url{https://arxiv.org/abs/2311.07911}.

\end{thebibliography}
\bibliographystyle{iclr2027_conference}

\setcounter{topnumber}{3}

\FloatBarrier
\appendix

Appendix~\ref{app:geometry} provides additional derivations and implementation
details, and Appendix~\ref{app:evaluation} specifies the data and evaluation
protocols. Appendices~\ref{app:ablations} and~\ref{app:calibration_robustness}
analyze scoring components, expert selection, and calibration sensitivity.
Appendix~\ref{app:axis_results} reports detailed fidelity and generation
results, and Appendix~\ref{app:related_work} discusses related compression methods.

\noindent\textbf{Scope of the reported estimates.}
All reported performance intervals condition on fixed checkpoints and exclude
variability over calibration draws, pruning runs, and independent training runs.
Matching a scoring rule across two studies does not imply that they
share calibration inputs or retained expert sets, so fidelity and benchmark
results are compared within a study rather than across studies.

\section{Additional Derivations and Implementation}
\label{app:geometry}

We derive the geometry of consensus residuals, relate them to exact
single-deletion damage, and examine interactions under joint removal.
Appendix~\ref{app:implementation} describes score accumulation and physical
pruning.

\subsection{Residual decomposition}
The consensus residual depends on the angle between an expert output and the
mixture, which is the information that gates and output norms omit. Expanding
it gives
\begin{align}
 \|f_i-\con\|_2^2
 &=\|f_i\|_2^2-2\langle f_i,\con\rangle+\|\con\|_2^2,\\
 \langle f_i,\con\rangle
 &=\|f_i\|_2\,\|\con\|_2\cos\phi_i.
\end{align}
Here $\phi_i$ is the angle between nonzero $f_i$ and $\con$.
If either vector is zero, their inner product is zero without defining an angle.
At fixed norms, a positive inner product reduces the residual and a negative
one increases it. Even without collinear or near-duplicate outputs, an aligned
expert may therefore be replaceable, while one balancing the others may be
costly to lose. The mixture still includes expert $i$, so a further
self-inclusion correction is needed to obtain distance to the survivors.
Because mixture norms, gates, and routed sets vary across calibration tokens,
this geometric identity constrains each token's damage without determining
which retained set performs better.

\paragraph{Constructed example in Figure~\ref{fig:razor-intuition}.}
Let $u,v$ be orthonormal and let $E_A,E_B,E_C$ have outputs $8u$,
$3u+4v$, and $10u-4v$, with equal initial weights. Their consensus is $c=7u$,
the routed-expert mixture at one layer, not the model's final output.
Residuals are $\|f_i-c\|_2$. Without refill, deleting $E_A,E_B,E_C$ and
renormalizing the survivors gives damage $0.5$, $\sqrt{8}\approx2.83$, and
$2.5$, respectively. With refill, the same initially unselected
$E_D=10u-5v$ enters after every deletion, with survivor and replacement
masses $1,1,0.8$. For deletion of $E_C$, the bars show $-(f_C-c)$,
$0.8(f_D-c)$, and their unnormalized sum. The $v$ terms cancel, giving
$\Delta c=-0.6u/2.8$ and damage $3/14\approx0.21$, versus $1.51$ for deleting
$E_A$ and $3.66$ for deleting $E_B$. Thus magnitude, fixed-support damage,
and refill damage favor $E_B,E_A,E_C$, respectively.

In the panels, curves span output coordinates, with mixture curves using
$2.5\times$ the expert vertical scale. Dotted curves isolate $u$ components,
and shading marks coordinate-wise differences. The construction illustrates
the role of output geometry and refill rather than a measured
checkpoint-level gain.

\subsection{Survivor-distance identity}
\label{app:survivor}
For $w_i<1$, starting from the fixed-support post-removal output gives
\begin{align}
 \con^{-i}&=\frac{\con-w_if_i}{1-w_i},\\
 f_i-\con^{-i}
 &=f_i-\frac{\con-w_if_i}{1-w_i}
 =\frac{f_i-\con}{1-w_i}.
\end{align}
Taking norms gives the survivor-distance identity used in the main text,
\begin{equation}
 \|f_i-\con\|_2=(1-w_i)\,\|f_i-\con^{-i}\|_2 ,
\label{eq:survivor}
\end{equation}
and combining \eqref{eq:survivor} with Proposition~\ref{prop:loo} yields
\begin{equation}
 \|\con-\con^{-i}\|_2=w_i\|f_i-\con^{-i}\|_2,
\label{eq:loo_alt}
\end{equation}
which expresses damage as routing mass times distance to the survivors.
The distance describes local functional replaceability at the same input,
while the weight converts that distance into mixture change. Neither factor
alone measures redundancy in the final retained model.

The factor $1/(1-w_i)$ is what converts the observed consensus residual into
this survivor distance. Omitting it yields $(1-w_i)\delta_i^{\mathrm{loo}}$ at
each token, which is a different proxy rather than the same deletion damage
under a rescaling, because $w_i$ varies across tokens. The identity also makes
the local comparison translation invariant. With inputs and gates fixed, adding
a common vector $g$ to all routed outputs shifts both mixtures by $g$ and
leaves deletion damage unchanged, although the output norms that magnitude
scores rely on do change. This invariance holds for the local damage, not for
model predictions.

For $\delta_i^{\mathrm{loo}}>0$, expanding the refill numerator gives
\begin{equation}
 \frac{\delta_i}{\delta_i^{\mathrm{loo}}}
 =\frac{1-w_i}{1-w_i+w_r}\,
  \sqrt{1-2\rho_i\cos\theta_i+\rho_i^{2}},
 \qquad
 \rho_i=\frac{w_r\|\resid_r\|_2}{w_i\|\resid_i\|_2},
\label{eq:refill_expand}
\end{equation}
where $\theta_i$ is the angle between nonzero residuals. When the promoted
weighted residual vanishes, the square-root factor is defined to be one.
The denominator factor is at most one and strictly
smaller when $w_r>0$, whereas the numerator factor exceeds one exactly when
$\rho_i>0$ and $\rho_i>2\cos\theta_i$ both hold. The condition on $\rho_i$
cannot be dropped, since $\rho_i=0$ makes the factor exactly one.
Orthogonal residuals thus already raise the numerator factor for any
$\rho_i>0$, yet routing mass, residual magnitudes, and alignment jointly decide
whether refill increases the damage ratio itself. Appendix~\ref{app:refill}
reports score ratios and partial geometric summaries, which measure the
observed effect without decomposing it into these three factors.

\subsection{Exact perturbation for a selected set}
\label{app:set}
Joint removal couples the experts it deletes, so the exact single-deletion
identities do not extend additively to a selected set.
For a removed routed subset $\remset$ with $W_{\remset}<1$, the survivor mixture
is $\con^{-\remset}=(\con-\sum_{j\in\remset}w_j f_j)/(1-W_{\remset})$.
Subtracting it from $\con$ gives Eq.~\plaineqref{eq:setshift}. All residuals
$\resid_j=f_j-\con$ share the original consensus, so the squared numerator
expands as
\begin{equation}
 \Big\|\sum_{j\in\remset}w_j\resid_j\Big\|_2^2
 =\underbrace{\sum_{j\in\remset}w_j^2\|\resid_j\|_2^2}_{\text{diagonal}}
 +\underbrace{\sum_{j\ne l\in\remset}w_jw_l\langle\resid_j,\resid_l\rangle}_{
   \text{interference}} .
\label{eq:gram}
\end{equation}
Interference cross terms and the set-dependent denominator block an additive
decomposition into independent removal costs. The diagonal term also sums
experts within a token, whereas RMS aggregates tokens within an expert, so the
expansion neither derives RMS nor gives an exact joint objective obtained by
summing expert scores.

Under refill, let $\mathcal R$ contain the $|\remset|$ highest-ranked
unselected experts that remain in the retained pool. The numerator becomes
$\|\sum_{j\in\remset}w_j\resid_j-\sum_{r\in\mathcal R}w_r\resid_r\|_2$ and the
denominator becomes $1-W_{\remset}+W_{\mathcal R}$, where
$W_{\mathcal R}=\sum_{r\in\mathcal R}w_r$. Both now depend on the whole
set, and the expansion acquires cross terms within $\mathcal R$ and between
$\remset$ and $\mathcal R$ on top of those in Eq.~\plaineqref{eq:gram}.
Summing scalar single-expert scores captures none of these interactions, which
is why our selection remains a ranking heuristic under a fixed budget.

\subsection{Implementation and complexity}
\label{app:implementation}

\begin{algorithm}[htbp]
\caption{Chunked scoring and routing-aware physical pruning for the RCS family}
\label{alg:razor}
\begin{algorithmic}[1]
\REQUIRE Model, calibration data $\mathcal D$ with valid-position masks, budgets $B_\ell$, chunk size $C$, criterion $a\in\{\mathrm{RCS},\mathrm{RCS\mbox{-}LOO},\mathrm{RCS\mbox{-}Refill}\}$, denominator clamp $\epsilon=10^{-6}$
\STATE Initialize per-shard $q_{\ell i}\leftarrow0$, $n_{\ell i}\leftarrow0$ for learned-router layers
\FOR{learned-router layer input from an unpruned forward over $\mathcal D$}
  \STATE Obtain routed sets $\routeset_t$ and normalized selected weights $w_t$
  \IF{$a=\mathrm{RCS\mbox{-}Refill}$}
    \STATE Obtain the rank-$(k{+}1)$ expert $r_t$ and its pseudo-weight $w_r(t)$
  \ENDIF
  \FOR{token chunk $X$ of size at most $C$}
    \STATE Reconstruct routed outputs $f_i(t)$ and form $\con_t=\sum_{i\in\routeset_t}w_i(t)f_i(t)$ and $\resid_i(t)=f_i(t)-\con_t$
    \IF{$a=\mathrm{RCS}$}
      \STATE $d_i(t)\leftarrow w_i(t)\|\resid_i(t)\|_2$
    \ELSIF{$a=\mathrm{RCS\mbox{-}LOO}$}
      \STATE $d_i(t)\leftarrow w_i(t)\|\resid_i(t)\|_2/\max(1-w_i(t),\epsilon)$
    \ELSE
      \STATE Reconstruct $f_{r_t}(t)$, set $\resid_{r_t}(t)=f_{r_t}(t)-\con_t$
      \STATE $d_i(t)\leftarrow\|w_i(t)\resid_i(t)-w_r(t)\resid_{r_t}(t)\|_2/\max(1-w_i(t)+w_r(t),\epsilon)$
    \ENDIF
    \STATE For valid routed pairs, accumulate $q_{\ell i}\mathrel{+}=d_i(t)^2$, $n_{\ell i}\mathrel{+}=1$; discard chunk-local outputs
  \ENDFOR
\ENDFOR
\STATE Sum $q,n$ across disjoint shards; set $s_{\ell i}=\sqrt{q_{\ell i}/\max(n_{\ell i},1)}$
\STATE Retain the $B_\ell$ highest-scoring experts in each learned-router layer
\STATE For fixed-hash layers, select by table counts and remap to distinct surviving routes
\RETURN Model with consistently sliced expert tensors, router rows, and updated route indices
\end{algorithmic}
\end{algorithm}

\subsubsection{Online statistics and bounded working buffers}
Algorithm~\ref{alg:razor} changes only the token-level quantity across the three
criteria, which share chunking, conditional RMS aggregation, ranking, and
physical pruning. The \rcs and \rcsloo branches reconstruct only the routed expert
outputs, whereas \rcsrefill additionally reconstructs the rank-$(k{+}1)$
promoted expert. Selecting \rcsrefill with conditional RMS gives \method.
For each learned-router layer, the scoring pass obtains the routing context for the
batch, partitions its token positions, and locally reconstructs the required
expert outputs for each chunk. These outputs form the consensus and residuals
before advancing to the next chunk. The cost is therefore
additional local computation alongside one original forward pass, rather than a
separate end-to-end model evaluation for every candidate deletion.

Chunking bounds the working buffers.
With $T$ packed token positions, top-$k$ routing, and expert-output width $d$,
storing the routed outputs together with the promoted one requires $O(T(k+1)d)$
space. Restricting the working set to $C$ positions replaces this with
$O(C(k+1)d)$, plus $O(Cd)$ for the
consensus. Only $q_i=\sum_t d_i(t)^2$, where $d_i(t)$ is whichever of the three
token quantities the criterion selects, and the valid routed-token count $n_i$
persist for RMS scoring, requiring $O(E)$ state per layer. Disjoint calibration
shards accumulate these quantities independently. Summing them before taking
$\sqrt{q_i/\max(n_i,1)}$ gives the pooled conditional RMS, unlike averaging
shard-level RMS scores. Unobserved experts receive zero score, and we remove
the lowest $E-B$ scores without a secondary criterion for exact ties. These
bounds cover only the storage controlled by chunking, excluding model weights
and ordinary forward activations.

One implementation choice affects the computed numbers. Each denominator is
clamped from below by $10^{-6}$, giving $\max(1-w_i+w_r,10^{-6})$ for the
refill score and $\max(1-w_i,10^{-6})$ for its fixed-support special case.
Because the identities describe the unclamped quantity, an active clamp makes
the computed score differ from exact local damage.

\subsubsection{Layer-wise execution}
Output chunking bounds activation buffers but leaves the memory occupied by
model weights untouched. For models that exceed the available device capacity,
a layer-wise schedule loads one decoder layer, propagates the current
calibration subset through it, collects its statistics, and releases its weights
before loading the next layer. This is what makes the largest backbones here
scorable on a fixed device budget. The schedule preserves the score definition,
although floating-point differences can affect ties or nearly tied rankings.

\subsubsection{Architecture-specific routing}
The derivation requires normalized selected weights, not a particular router
parameterization. We preserve each architecture's native selection rule and
score function, using softmax probabilities for Qwen3.6-35B-A3B, sigmoid scores
for GLM-4.7-Flash, DeepSeek-V4-Flash-0731's native score transformation, and
sigmoid scores for Hy3. Where a correction bias determines selection, it is not
added to the mixture weights. Selected unmodified scores are renormalized
before evaluating the damage formula. The implementation retains the
routed-output scale $\lambda$, although omitting this common positive factor
from the displayed score leaves within-layer rankings unchanged.

DeepSeek-V4-Flash-0731 includes three fixed-hash layers whose expert identities
come from a frozen token-to-expert table, so they fall outside learned-router
saliency and need a separate rule. To retain $B$ experts, we select the most
frequent experts by counting occurrences over the entire table, without
weighting by calibration-token occurrence. Each table row then preserves its
surviving expert identities, and removed entries are replaced, in slot order,
by the cosine-nearest surviving router-weight vector not already used in that
row. With $B\geq k$, the result has $k$ distinct routes, and reindexing follows
the same survivor order as parameter slicing. The procedure repairs routing
rather than merging expert weights or solving an assignment problem. The
routing-stratified study shares this hash-layer selection across criteria, so
its comparisons isolate learned-router saliency
(Appendix~\ref{app:benchmark_details}).

\subsubsection{Physical pruning}
We slice expert-indexed tensors and corresponding router rows with the same
keep indices, leaving dense blocks and shared experts unchanged. Under a common
retained width, the result is a smaller MoE checkpoint requiring neither runtime
masks nor custom sparse kernels. For Qwen3.6-35B-A3B, reducing 256 to 128 routed experts per
layer yields an approximately 19B-parameter model.

Removing experts reduces stored parameters, but preserving top-$k$ does not
proportionally reduce active expert computation. We do not measure serving
latency, energy use, or end-to-end acceleration.

\section{Data and Evaluation Protocol}
\label{app:evaluation}

We use separate protocols for calibration, predictive fidelity, downstream
performance, and response behavior. The
calibration pool \calib scores experts (Appendix~\ref{app:razorcal}), the
eight-axis reference pool measures predictive fidelity
(Appendix~\ref{app:eval_data}), the nine-task suite measures downstream
performance (Appendix~\ref{app:benchmark_details}), and a separate
repeated-generation collection measures response behavior
(Appendix~\ref{app:response_diversity}). Only the first influences which
experts are retained. The fidelity pool alone splits its eight axes into four
that \calib covers and four that it does not, so calibration-covered and
held-out fidelity can be reported separately throughout
Appendix~\ref{app:axis_results}. That ID/OOD distinction applies to the
fidelity axes only. The benchmark suite and the repeated-generation collection
are not partitioned this way, since they draw on public benchmarks rather than
on axes defined relative to \calib. The calibration pool and the fidelity pool
draw mainly on open Nemotron post-training datasets, whose development and
reuse across the model family are described in the Nemotron 3 Super technical
report \citep{nemotron3super2026}, together with a few non-Nemotron
instruction sources named below. The nine benchmark tasks, and the generation
analyses using those tasks, come from their own public releases
(Appendix~\ref{app:benchmark_details}). The specific dataset versions are
listed below.

\subsection{Calibration data}
\label{app:razorcal}

\calib contains 2,048 chat-format examples across seven domains
(Table~\ref{tab:calib}), whose embedding spread and per-domain length coverage
are shown in Figure~\ref{fig:calib}. It is a calibration mixture assembled from
existing instruction datasets, not a new independently collected corpus. Coding receives
twice the allocation of each other domain to cover function completion,
competitive programming, class-based solutions, software engineering, and tool
use. Source-specific sampling applies quotas, length filters, and deduplication.
Preprocessing removes dataset-specific harness prefixes where appropriate.
The 32,000-character cap counts message content and serialized tool calls,
excluding separately stored reasoning fields and tool schemas. Each
checkpoint's chat template renders the preserved conversation roles before
tokenization.

\begin{table}[!t]
\centering
\small
\setlength{\tabcolsep}{2.8pt}
\renewcommand{\arraystretch}{1.16}
\caption{\calib domain allocations and calibration sources.
Counts denote source examples, not tokens.}
\label{tab:calib}
\begin{tabularx}{\columnwidth}{@{}>{\raggedright\arraybackslash}p{3.0cm}>{\centering\arraybackslash}p{1.25cm}>{\raggedright\arraybackslash}X@{}}
\toprule
\rowcolor{tblHeader}
\thead{Domain} & \thead{Samples} & \thead{Source} \\
\midrule
Coding & \textbf{512} & Dolci-Instruct-SFT and Nemotron competitive programming, SWE, and OpenCode \\
Mathematics & 256 & Nemotron-Math-v2, high-reasoning AoPS subset \\
Science & 256 & Nemotron-Science-v1 \\
Chinese-STEM & 256 & Nemotron-SFT-Multilingual-v1 \\
Instruction following & 256 & Dolci constraint-following mixture \\
Tool calling & 256 & Nemotron-Agentic-v1 \\
World knowledge & 256 & OpenHermes-2.5 \\
\bottomrule
\end{tabularx}
\end{table}

\paragraph{Sources and composition.}
Coding includes 205 Python-algorithm examples from Dolci-Instruct-SFT, the
instruction-tuning mixture accompanying Olmo~3 \citep{olmo3_2025}. It adds
141 Python and 115 C++ examples from Nemotron-SFT-Competitive-Programming-v2,
40 agentless software-engineering examples from Nemotron-SFT-SWE-v2, and
11 general tool-use examples from Nemotron-SFT-OpenCode-v1.
Nemotron-SFT-SWE-v2 covers code localization, repair, and test generation,
whereas Nemotron-SFT-OpenCode-v1 covers agent interactions.

Mathematics contains 256 problem--solution examples from the high-reasoning
AoPS subset of Nemotron-Math-v2, including 96 selected for competition-style
problems. This source pairs problems from Art of Problem Solving with
model-generated solutions. High denotes the teacher's reasoning mode,
not problem difficulty.
Science combines 176 chemistry reasoning questions and 80 multiple-choice
questions from Nemotron-Science-v1. Chinese-STEM
uses translated mathematics, code, and STEM examples from
Nemotron-SFT-Multilingual-v1 (86/85/85), which translates prompts and final
answers rather than all reasoning traces.
Instruction following draws 256 examples from a Dolci constraint-following pool
\citep{olmo3_2025}.
Tool calling uses 256 multi-turn agent--tool conversations from
Nemotron-Agentic-v1. The world-knowledge allocation
contains 256 OpenHermes-2.5 examples \citep{openhermes25}, drawn from the
cleaned, non-thinking subset in SmolTalk2 \citep{smoltalk2}. OpenHermes is a
general-purpose synthetic instruction corpus rather than a dedicated knowledge
test, and all quotas balance source examples rather than token counts.

\begin{figure}[tbp]
\centering
\includegraphics[width=\textwidth]{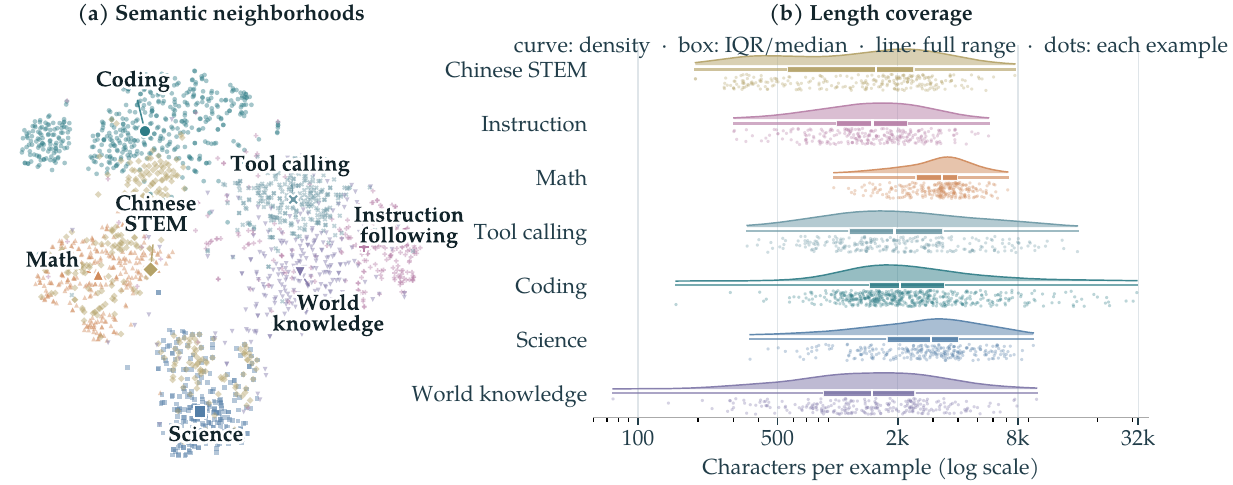}
\caption{\textbf{Overview of the \calib calibration pool.} (a) Cosine
t-SNE projection of Qwen3-Embedding-8B representations of the rendered
conversations, with color and marker shape encoding domain and labels marking
domain medians. (b) Per-domain character-length coverage over the full observed
range, counting message content only. The curve is a kernel-density estimate in log-character space, the bar
is the interquartile range with the median marked, the thin line spans minimum
to maximum, and each dot is one of the 2,048 examples. Domain counts are given in
Table~\ref{tab:calib}.}
\label{fig:calib}
\end{figure}

\paragraph{Scored positions and coverage.}
For the diagnostic collection, calibration examples are rendered with thinking
disabled, separately supplied reasoning fields omitted, and packed into
4,096-token rows. Calibration statistics include all
nonpadding conversation positions, including prompts and role headers, unlike
the assistant-only fidelity targets in Appendix~\ref{app:eval_data}.
Packing fills each row in order and stops at the row boundary, so an example
that crosses it is truncated and its tail is discarded rather than continued
into the next row. Long examples therefore contribute only their leading
tokens, and the collection stops once the token budget is filled, which is why
the counts below cover fewer than the 2,048 pool examples.
Qwen3.6-35B-A3B scores 785,703 valid tokens from 921 source examples, while
GLM-4.7-Flash scores 785,632 tokens from 992. We call these the
\emph{diagnostic calibration collections} below. They provide the
score-distribution and Gram-matrix statistics, as well as the saliency scores
used in the single-expert ablations in Appendix~\ref{app:ablations}. The
ablation damage is measured on separate evaluation inputs.
Because the two tokenizers give
different token counts and source coverage at the same row capacity, these
statistics are compared within a backbone rather than across backbones.

\subsection{Fidelity evaluation data}
\label{app:eval_data}

Reverse KL and reference-token perplexity use a pool of 8,192 conversations,
with 1,024 examples on each of eight task axes. These compare predictions on
shared reference responses rather than scoring benchmark answers. Four axes
cover calibration task types.

\noindent\textbf{Mathematics.} Mathematical problem-solving responses from the
\texttt{medium} split of Nemotron-Math-v2.
Medium denotes the teacher's reasoning mode, not problem difficulty.

\noindent\textbf{Code.} Python solutions from the \texttt{competitive\_coding\_python}
split of Nemotron-SFT-Competitive-Programming-v2.

\noindent\textbf{Instruction following.} Constraint-following conversations from the
\texttt{chat\_if} split of Nemotron-Instruction-Following-Chat-v1,
selected by its instruction-following capability label.

\noindent\textbf{Tools.} Agent--tool interaction trajectories from the
\texttt{tool\_calling} split of Nemotron-Agentic-v1.

The other four axes hold out task types.

\noindent\textbf{Chat.} Examples carrying the chat capability label in the same
Chat-v1 \texttt{chat\_if} split.

\noindent\textbf{Creative writing.} Examples selected by prompt keywords from the
\texttt{reasoning\_off} split of Nemotron-SFT-Instruction-Following-Chat-v2.
Creative writing is our selection, not an official split.

\noindent\textbf{Safety.} Reference safety-aligned responses from the \texttt{train}
split of Nemotron-SFT-Safety-v1.

\noindent\textbf{SQL.} Examples from the \texttt{text\_to\_sql} split of the
competitive-programming source, separate from its Python subset,
pairing natural-language tasks and database schemas with SQL.
ID/OOD denotes calibration task-type coverage, not disjoint source datasets or
a formal distribution shift, and both groups contain synthetic data.

Each axis retains 1,024 examples that fit the row capacity and carry enough
assistant content to score. Conversations end on an assistant response, and
plain-text assistant turns may be truncated while structured tool calls remain
intact.

Evaluation packs examples into 4,096-token rows and scores assistant positions,
excluding role headers and turn terminators. Chat rendering disables thinking and
omits separately supplied reasoning fields, so fidelity is measured on fixed
reference prefixes rather than on generated reasoning trajectories. The routing
study uses six batches per axis (Appendix~\ref{app:mechanism}), and its
evaluated batches form a subset of this pool. The fidelity pool has no
exact-message overlap with calibration, although near-duplicates and benchmark
contamination have not been ruled out.

\subsection{Evaluation configuration}
\label{app:config}

Table~\ref{tab:config} summarizes the four main backbones.
Qwen3.6-35B-A3B uses the hybrid attention--MoE architecture described for
Qwen3.5 \citep{qwen35_2026}. The GLM and DeepSeek model families are described
in their technical reports \citep{glm45_2025,deepseekv4_2026}. The configurations
below refer to the evaluated checkpoints. Hy3 is the public release
\citep{hy3_2026}. The architecture and routing settings describe the language-model decoder.
Layer counts exclude auxiliary prediction modules. GLM-4.7-Flash and Hy3
each begin with one dense feed-forward layer. All 43 DeepSeek-V4-Flash-0731 layers are
MoE layers, of which three use fixed-hash routing and the other 40 use
learned-router selection. Its score transformation is $\sqrt{\operatorname{softplus}(z)}$.
The nominal parameter counts exclude auxiliary MTP and DSpark modules.

Both removal ratios retain the original top-$k$ and shared experts. Calibration
uses the 2,048-example source pool, while the scored coverage in
Appendix~\ref{app:razorcal} refers specifically to the Qwen3.6-35B-A3B and GLM-4.7-Flash diagnostic
collection. For downstream benchmarks, we use temperature $0$, except for
SWE-bench Verified, where we use $0.7$. AIME'26 reports avg@8, and the other eight benchmarks report avg@3.
For each benchmark, sampling parameters are shared across all four backbones.
We set the maximum input length to 224K tokens for GLM-4.7-Flash and
Qwen3.6-35B-A3B, and 128K tokens for DeepSeek-V4-Flash-0731 and Hy3.
All four use a 32K-token output limit. These are evaluation limits, not the
models' maximum context capacities. The Qwen3.6-35B-A3B and GLM-4.7-Flash
benchmark settings enable thinking. Fidelity and response diversity use their
own protocols, with generation analysis using temperature $0.7$.
Cross-study comparability is discussed in Appendix~\ref{app:benchmark_details}.

Comparisons among the residual criteria match calibration inputs and
layerwise budgets within each study. Benchmark and diagnostic protocols use
separate checkpoints and calibration draws, as specified in
Appendix~\ref{app:benchmark_details}.

\begin{table}[!t]
\centering
\small
\setlength{\tabcolsep}{4pt}
\renewcommand{\arraystretch}{1.10}
\caption{Architectures and pruning settings of the four main backbones.
Parameter counts are nominal backbone sizes, and expert counts are per MoE layer.}
\label{tab:config}
\begin{tabularx}{\columnwidth}{@{}>{\raggedright\arraybackslash}p{4.9cm}*{4}{>{\centering\arraybackslash}X}@{}}
\toprule
& \thead{Qwen3.6-} & \thead{GLM-4.7-} & \thead{DeepSeek-V4-} &
\multirow{2}{*}{\thead{Hy3}} \\
& \thead{35B-A3B} & \thead{Flash} & \thead{Flash-0731} & \\
\midrule
\multicolumn{5}{@{}l}{\textit{Architecture}} \\
Total parameters & 35B & 30B & 284B & 295B \\
Active parameters per token & $\sim$3B & $\sim$3B & $\sim$13B & $\sim$21B \\
Decoder layers & 40 & 47 & 43 & 80 \\
MoE layers & 40 & 46 & 43 & 79 \\
Routed experts & 256 & 64 & 256 & 192 \\
Expert intermediate width & 512 & 1,536 & 2,048 & 1,536 \\
Shared experts & 1 & 1 & 1 & 1 \\
\multicolumn{5}{@{}l}{\textit{Routing}} \\
Top-$k$ & 8 & 4 & 6 & 8 \\
Router scores & softmax & sigmoid & sqrtsoftplus & sigmoid \\
Routed scale $\lambda$ & 1.0 & 1.8 & 1.5 & 2.826 \\
Fixed-hash layers & 0 & 0 & 3 & 0 \\
\multicolumn{5}{@{}l}{\textit{Pruning}} \\
Pruned fractions & \multicolumn{4}{c}{25\% and 50\%} \\
Retained routed experts (25\%) & 192 & 48 & 192 & 144 \\
Retained routed experts (50\%) & 128 & 32 & 128 & 96 \\
\bottomrule
\end{tabularx}
\end{table}

\subsection{Metric definitions}
\label{app:metrics}

Let $p_t$ and $q_t$ be the unpruned and pruned next-token distributions on the
same reference prefix. Reverse KL is
\begin{equation}
 D_{\mathrm{KL}}(q_t\|p_t)=\sum_v q_t(v)\log\frac{q_t(v)}{p_t(v)}.
\end{equation}
The main fidelity sweep weights assistant tokens within each evaluation axis
and balances the eight axes equally.

For reference token $y_t$, define
$\ell_p(t)=-\log p_t(y_t)$ and $\ell_q(t)=-\log q_t(y_t)$. Averaging these
losses under the specified evaluation weights gives reference-token NLL and
\begin{equation}
\begin{aligned}
 \mathrm{PPL}_{p}&=\exp(\mathrm{NLL}_{p}),\qquad
 \mathrm{PPL}_{q}=\exp(\mathrm{NLL}_{q}),\\
 \Delta\mathrm{NLL}&=\mathrm{NLL}_{q}-\mathrm{NLL}_{p},\qquad
 \frac{\mathrm{PPL}_{q}}{\mathrm{PPL}_{p}}-1
 =\exp(\Delta\mathrm{NLL})-1.
\end{aligned}
\end{equation}
The final quantity is excess perplexity. These likelihood metrics use observed
reference tokens, not samples from $p_t$ or $q_t$. Negative
$\Delta\mathrm{NLL}$ indicates higher reference-token likelihood under the
pruned model, not necessarily better downstream capability.

For normalized selected routing weights and $k>1$, routing entropy is
\begin{equation}
 H_r(x)=-\frac{\sum_{j\in\routeset(x)}w_j(x)\log w_j(x)}{\log k}.
\end{equation}
Smaller $H_r$ means more concentrated routing. Predictive entropy instead
summarizes the vocabulary distribution,
\begin{equation}
 H(p_t)=-\sum_v p_t(v)\log p_t(v),\qquad
 \Delta H_t=H(q_t)-H(p_t).
\end{equation}
Absolute entropy drift averages predictive entropies within routing cells
and across valid layers, takes the absolute pruned--original difference in
each batch--decile cell, and then averages these differences. Taking the
absolute value after cell aggregation distinguishes this measure from the
mean per-token $|\Delta H_t|$. Predictive entropy also differs from
reference-token NLL, so its exponential is not the PPL reported here.
Appendix~\ref{app:axis_results} specifies the sampling units and uncertainty
estimates for each output analysis.

\subsection{Downstream benchmark protocol and scope}
\label{app:benchmark_details}

The benchmark evaluates four backbones at two removal ratios each.
Table~\ref{tab:main} covers GLM-4.7-Flash and Qwen3.6-35B-A3B, and
Table~\ref{tab:additional} covers DeepSeek-V4-Flash-0731 and Hy3. The full method
comparison includes Frequency, EAN, REAP, and all three residual scores on
GLM-4.7-Flash and Qwen3.6-35B-A3B. \rcs, \rcsloo, and \method use
conditional RMS. DeepSeek-V4-Flash-0731 and Hy3 compare the three residual
scores against their unpruned references only, so they enter the
\rcsloo-relative and unpruned-relative benchmark comparisons rather than the
REAP-relative ranges and win counts.

\paragraph{Baseline scores.}
For $n_i=|\mathcal D_i|>0$ routed calibration tokens, the baseline scores
are, up to common positive layer scales that do not affect ranking,
\begin{equation}
 s_i^{\mathrm{Frequency}}=n_i,\qquad
 s_i^{\mathrm{EAN}}=\frac{1}{n_i}\sum_{x\in\mathcal D_i}\|f_i(x)\|_2,\qquad
 s_i^{\mathrm{REAP}}=\frac{1}{n_i}\sum_{x\in\mathcal D_i}w_i(x)\|f_i(x)\|_2.
\end{equation}
Frequency counts hard routing selections, as in frequency-based pruning
\citep{muzio2024seer,fantasticexperts2025}. EAN is our conditional-mean
adaptation of activation-norm pruning \citep{fantasticexperts2025,reap2026}.
We divide by each expert's routed-token count, whereas the accumulated-norm
EAN in REAP's baseline definition does not. This expert-specific normalization
is not a common rescaling and can change the ranking. REAP follows the
routed-token conditional mean in its original definition \citep{reap2026}.
The conditional-mean EAN score is algebraically identical to MAN in
\citet{unified2026}. We retain the EAN label in the tables. All three scores
are computed without recovery training, and zero-observation experts receive
zero score.

\paragraph{Mathematics and instruction following.}
The nine tasks cover six capability categories. Each is graded on its final
answer or artifact rather than on the reasoning that produced it, so a task
that requires many dependent steps is still scored by whether the endpoint is
correct. AIME'26 uses the 2026
American Invitational Mathematics Examination \citep{maaAIME} to assess
competition mathematics. IFEval \citep{ifeval2023} checks compliance with
programmatically verifiable instructions. IFBench \citep{ifbench2025}
extends this evaluation to additional verifiable constraints, testing
instruction-following generalization beyond IFEval.

\paragraph{Knowledge, tools, and long context.}
SuperGPQA \citep{supergpqa2025} evaluates graduate-level knowledge across
285 disciplines. The Berkeley Function Calling Leaderboard, version 4
(BFCL v4, \citealp{bfcl2025,bfclv4}), evaluates function calling and agentic
tool use. LongBench v2 \citep{longbench2024} tests comprehension and reasoning
over long contexts.

\paragraph{Coding tasks and versions.}
The three coding tasks cover function-level generation, competition programming,
and repository-level issue resolution. HumanEval+ \citep{evalplus2023} uses
EvalPlus's expanded tests to assess the functional correctness of generated
Python functions. LiveCodeBench \citep{livecodebench2024} draws from time-indexed
programming competitions. We evaluate LiveCodeBench 2026, abbreviated
LCB'26 in the tables, using the latest version in our evaluation setup.
This evaluation is distinct from the
175-question subset used for the generation diagnostics
(Appendix~\ref{app:distinct_n_protocol}). SWE-bench Verified
\citep{swebench2024}, abbreviated SWE, is the 500-issue
subset of SWE-bench screened by software developers for well-specified issues
and appropriately scoped tests.\footnote{\url{https://huggingface.co/datasets/princeton-nlp/SWE-bench_Verified}}
Scores report the percentage of resolved
issues averaged over three evaluations.

\paragraph{Score aggregation.}
Instruction-following and coding averages give equal weight to their two and
three tasks, respectively, while Overall weights all nine equally. Averages and
paired macro gaps are computed from task scores rather than rounded displayed means.
The task scores yield a 33/0/3 win/tie/loss count for
\rcsloo against REAP and 36/0/0 for \method across GLM-4.7-Flash and
Qwen3.6-35B-A3B at two budgets. Two properties qualify these counts. Both sets
of 36 comparisons reuse one task suite rather than replicate it, and neither
extends to the two backbones without a REAP benchmark run. The benchmark
comparisons report no task-level uncertainty. AIME'26 reports avg@8 and the other eight
benchmarks report avg@3. The three residual criteria and REAP use matched \calib
inputs, whereas REAP$^\star$ is an external GLM-4.7-Flash checkpoint calibrated
with 24,576 samples.

\paragraph{Additional task contrasts.}
On Hy3, \method scores 63.4 and 59.5 at 25\% and 50\% removal, respectively,
versus 63.7 unpruned. On DeepSeek-V4-Flash-0731, the corresponding scores are
61.0 and 55.6 versus 59.2 unpruned (Table~\ref{tab:additional}). The macro
ordering is neither uniform across tasks nor monotonic across scoring
refinements. At 25\% on Qwen3.6-35B-A3B, \rcs slightly exceeds
\rcsloo in macro average (61.0 versus 60.8). At 50\% on GLM-4.7-Flash, EAN exceeds
\method on LCB'26 (33.9 versus 29.8).

For \rcsloo versus REAP, both backbones carrying that comparison show larger
macro gaps at 50\% than at 25\% (0.80 to 3.01 on GLM-4.7-Flash and 1.02 to
1.96 on Qwen3.6-35B-A3B), leaving open whether the widening is general.
On DeepSeek-V4-Flash-0731 at 25\%, \rcsloo gains on SWE
(12.4 to 18.8) but loses on HE+ (91.5 to 89.6) relative to \rcs, the
nearest measured criterion on that backbone.
At 50\%, SWE suffers the largest proportional loss from the original for
\rcsloo on both Hy3 and DeepSeek-V4-Flash-0731. Thus even favorable macro comparisons
coexist with capability-specific trade-offs.

\paragraph{Comparability.}
Benchmark and diagnostic results are compared within a study rather than
across studies, since the two use separate checkpoints and calibration
draws even when they share a scoring criterion. On DeepSeek-V4-Flash-0731,
where the routing study matches fixed-hash selections across criteria, its
benchmark gap therefore reflects the full pruned model rather than
learned-router saliency alone. The LiveCodeBench comparison uses the same
release across methods. Comparisons with external results require matching
the task release and evaluation protocol. Input and output limits and
shared sampling settings appear in Appendix~\ref{app:config}.

\section{Scoring Proxies and Expert Selection}
\label{app:ablations}

A scoring rule makes two choices that the main text reports jointly, namely
which quantity to measure at a routed token and how to reduce that quantity
across tokens. This appendix separates them. Appendices~\ref{app:rms_detail}
through~\ref{app:factorial} compare these choices and their interactions.
Conditional RMS yields lower KL than the conditional mean for most tested
token quantities, while the benefit of the residual reference depends on
aggregation and backbone.
Appendix~\ref{app:selection} then asks how the scoring choices change the
retained set itself, through three probes that measure selection consequences
rather than score values.

The analyses below use distinct evaluation collections. Comparisons are
interpreted within each collection, except for the explicitly identified
mean-refill contrasts in Appendix~\ref{app:factorial}, which are not fully
matched.

\subsection{Scoring criteria and budget comparisons}

Table~\ref{tab:instances} lists the three consensus-residual scoring rules and
the magnitude baseline \reaprms, all with conditional RMS, alongside their
experimental coverage.

\begin{table}[!t]
\centering
\small
\setlength{\tabcolsep}{4pt}
\renewcommand{\arraystretch}{1.14}
\caption{\textbf{Scoring criteria and experimental coverage.}}
\label{tab:instances}
\begin{tabularx}{\textwidth}{@{}p{2.0cm}p{3.5cm}>{\raggedright\arraybackslash}X@{}}
\toprule
\thead{Instance} & \thead{Token quantity} & \thead{Evidence in this paper} \\
\midrule
\rcsrefill\newline (\method) &
$\dfrac{\|w_i\resid_i-w_r\resid_r\|_2}{1-w_i+w_r}$ &
Nine-task benchmarks on all four backbones, with baselines on two
(Tables~\ref{tab:main}, \ref{tab:additional}). Reverse KL
(Figure~\ref{fig:radar-four-criteria}, Table~\ref{tab:loo_ablation}),
score distributions (Table~\ref{tab:cv_stats}), and
refill mechanism (Table~\ref{tab:refill}). Mean--RMS reduction and factorial
comparisons (Figure~\ref{fig:aggregation}, Table~\ref{tab:factorial_complete}).
Per-axis, routing-stratified and PPL fidelity, selection probes, and
Qwen3.6-35B-A3B response behavior (Table~\ref{tab:behaviour-lcb}). \\
\rcsloo &
$\dfrac{w_i\|\resid_i\|_2}{1-w_i}$ &
Nine-task benchmarks on all four backbones, with baselines on two
(Tables~\ref{tab:main}, \ref{tab:additional}). Output-fidelity and diversity figures, response behavior,
reverse KL, score distributions, and calibration-budget and
corpus diagnostics. \\
\rcs &
$w_i\|\resid_i\|_2$ &
Nine-task benchmarks on all four backbones, with baselines on two
(Tables~\ref{tab:main}, \ref{tab:additional}). Reverse KL
(Figure~\ref{fig:radar-four-criteria}, Table~\ref{tab:loo_ablation}) and
the reduction axis (Figure~\ref{fig:aggregation}). \\
\reaprms &
$w_i\|f_i\|_2$ &
Reverse KL in the separate factorial collection
(Appendix~\ref{app:factorial}). No benchmarks. \\
\bottomrule
\end{tabularx}
\end{table}

Each criterion fixes a token quantity, and every criterion here conditions its
aggregation on tokens routed to the expert. \rcs uses the routing-weighted
consensus residual, \rcsloo adds survivor renormalization, and \rcsrefill
additionally models the promoted expert. The token quantities of \rcsloo and
\rcsrefill are exact under their respective deletion assumptions, yet the
fixed-support identity fixes only the combination of consensus residual and LOO
factor. It leaves both the aggregation rule and the performance of the selected
checkpoint open, so the three rules need not form a performance hierarchy.

Magnitude and residual criteria can be compared under the same aggregation.
\reaprms uses the output norm and differs from REAP solely in replacing the
conditional mean with RMS, isolating aggregation from the reference change.
Our EAN adaptation averages $\|f_i\|_2$ and REAP averages $w_i\|f_i\|_2$, with
source definitions and the EAN normalization difference in
Appendix~\ref{app:benchmark_details}. The \reaprms results belong to the
factorial collection of Appendix~\ref{app:factorial}, whose
Table~\ref{tab:factorial_complete} crosses $f_i$ against $f_i-\con$ and $w_i$
against $w_i/(1-w_i)$ over the available sum, mean, and RMS reductions, with
refill damage as an additional quantity.

Table~\ref{tab:loo_ablation} compares REAP, \rcs, \rcsloo, and \method
under the same checkpoint-evaluation protocol. Moving from REAP to \rcs
lowers KL on both models, but changes both the reference and mean-to-RMS
aggregation, so it does not isolate consensus subtraction. The matched
mean-aggregation contrast appears in Appendix~\ref{app:matched_components}.
Adding the LOO denominator raises KL on Qwen3.6-35B-A3B, while adding refill
lowers it in three of four settings. These estimates pool tokens within each
axis and weight the eight axes equally. The $\Delta$ vs REAP column gives
\rcsloo's relative change from REAP, while refill gain gives \method's
relative reduction from \rcsloo. Negative values favor \rcsloo in both
columns, and percentages use unrounded KL.

\begin{table}[!t]
\centering
\small
\setlength{\tabcolsep}{4pt}
\renewcommand{\arraystretch}{1.14}
\caption{\textbf{Reverse KL across scoring rules in nats (lower is better).}}
\label{tab:loo_ablation}
\begin{tabularx}{\textwidth}{@{}>{\raggedright\arraybackslash}X*{6}{>{\raggedleft\arraybackslash}X}@{}}
\toprule
\thead{Removal} & \thead{REAP} & \thead{\rcs} & \thead{\rcsloo} &
\thead{\method} &
\thead{$\Delta$ vs REAP} & \thead{Refill gain} \\
\midrule
\multicolumn{7}{@{}l}{\textit{Qwen3.6-35B-A3B}} \\
25\% & 0.07461 & 0.06669 & 0.07141 & 0.06872 & $-4.3\%$ & $+3.8\%$ \\
50\% & 0.45654 & 0.40337 & 0.45108 & 0.44397 & $-1.2\%$ & $+1.6\%$ \\
\midrule
\multicolumn{7}{@{}l}{\textit{GLM-4.7-Flash}} \\
25\% & 0.22850 & 0.20481 & 0.20351 & 0.20459 & $-10.9\%$ & $-0.5\%$ \\
50\% & 1.01690 & 0.92310 & 0.87103 & 0.85457 & $-14.3\%$ & $+1.9\%$ \\
\bottomrule
\end{tabularx}
\end{table}

\subsection{RMS aggregation and score distributions}
\label{app:rms_detail}

For routed calibration tokens $\mathcal D_i=\{x\in\mathcal D:i\in\routeset(x)\}$,
\method uses the conditional root mean square of refill damage,
\begin{equation}
 s_i
 =\sqrt{\frac{1}{|\mathcal{D}_i|}
   \sum_{x\in\mathcal{D}_i}
   \left(
     \frac{\|w_i(x)\resid_i(x)-w_r(x)\resid_r(x)\|_2}{1-w_i(x)+w_r(x)}
   \right)^2}.
\label{eq:score}
\end{equation}
This definition applies when $|\mathcal D_i|>0$, with unseen experts handled
as specified in Appendix~\ref{app:implementation}. The token damage follows
from Proposition~\ref{prop:refill}, but conditioning and RMS are aggregation
choices, not consequences of that identity. The same aggregation is applied
to $w_i\|\resid_i\|_2$ for \rcs and $\delta_i^{\mathrm{loo}}$ for \rcsloo.

RMS differs from the conditional mean by an explicit variance term, which is
what gives larger observed shifts more weight.
With empirical conditional mean $\mu_i$ and population standard deviation
$\sigma_i$ of token-level damage over $\mathcal D_i$,
\begin{equation}
 s_i=\sqrt{\mu_i^2+\sigma_i^2}
 =\mu_i\sqrt{1+\mathrm{CV}_i^2},
 \qquad \mathrm{CV}_i=\sigma_i/\mu_i \quad (\mu_i>0).
\label{eq:cv}
\end{equation}
For $\mu_i<\mu_j$, RMS ranks $i$ above $j$ exactly when
$\sigma_i^2-\sigma_j^2>\mu_j^2-\mu_i^2$. Such reversals change the retained
set only across the budget boundary. This preference for variable damage
motivates RMS as an empirical design choice rather than a guarantee of better
pruning. Appendix~\ref{app:matched_components} tests that choice by comparing
mean and RMS while holding the token quantity fixed.

For the score-distribution analysis, we additionally accumulate the first
moment of token-level damage. Together with the routed-token count and second
moment used for RMS scoring, this yields each expert's mean and standard
deviation with $O(E)$ additional state per layer. We collect these statistics
for both fixed-support and refill damage.
Table~\ref{tab:cv_stats} compares \rcsloo and \method. Equal-expert quantiles
cover all observed experts without a count threshold, comprising 10,240 on
Qwen3.6-35B-A3B and 2,944 on GLM-4.7-Flash, and the CV panel of
Figure~\ref{fig:selection-diagnostics} uses the same experts. Within-layer
dispersion is
$\operatorname{median}_{\ell}[\operatorname{std}_i(\mathrm{CV}_{\ell i})/
\operatorname{mean}_i(\mathrm{CV}_{\ell i})]$, using population standard
deviations.

What changes a ranking is the premium's variation across experts rather than its
absolute size. The median fixed-support premium is 1.189 on Qwen3.6-35B-A3B and
1.155 on GLM-4.7-Flash. We quantify the resulting boundary swaps by comparing
direct top-$B$ mean and RMS rankings on the same experts and averaging swap
counts equally across layers, with each entry reporting the RMS-only retained
count over the retained-set size. Qwen3.6-35B-A3B has more exchanges at both
budgets. Under refill, the median premium is 1.166 on Qwen3.6-35B-A3B and 1.153
on GLM-4.7-Flash, and the layer-averaged swap counts differ from the
fixed-support averages by less than one expert. The two counterfactuals therefore show similar relative
variability and similar mean--RMS selection differences, which is compatible
with unequal damage magnitudes and unequal second moments between them.

\begin{table}[!t]
\centering
\small
\setlength{\tabcolsep}{4.5pt}
\renewcommand{\arraystretch}{1.12}
\caption{Damage variability and mean--RMS selection differences for
\rcsloo and \method.}
\label{tab:cv_stats}
\begin{tabularx}{\columnwidth}{@{}>{\raggedright\arraybackslash}p{5.4cm}*{4}{>{\centering\arraybackslash}X}@{}}
\toprule
& \multicolumn{2}{c}{\thead{Qwen3.6-35B-A3B}}
& \multicolumn{2}{c}{\thead{GLM-4.7-Flash}} \\
& \thead{\rcsloo}
& \thead{\method}
& \thead{\rcsloo}
& \thead{\method} \\
\midrule
\multicolumn{5}{@{}l}{\textit{Sample}} \\
Scored tokens & \multicolumn{2}{c}{785,703} & \multicolumn{2}{c}{785,632} \\
\multicolumn{5}{@{}l}{\textit{Coefficient of variation}} \\
$\mathrm{CV}$ median & 0.643 & 0.599 & 0.577 & 0.574 \\
$\mathrm{CV}$ $p_{5}$ / $p_{95}$ & 0.370 / 1.255 & 0.373 / 1.184
& 0.350 / 1.002 & 0.362 / 0.978 \\
Within-layer dispersion of $\mathrm{CV}$ & 0.317 & 0.305 & 0.221 & 0.211 \\
\multicolumn{5}{@{}l}{\textit{RMS premium $\sqrt{1+\mathrm{CV}^2}$}} \\
Median & 1.189 & 1.166 & 1.155 & 1.153 \\
$p_{95}$ & 1.605 & 1.550 & 1.416 & 1.398 \\
\multicolumn{5}{@{}l}{\textit{Boundary swaps}} \\
25\% removal & 7.95 / 192 & 7.53 / 192 & 1.39 / 48 & 1.33 / 48 \\
50\% removal & 10.80 / 128 & 11.00 / 128 & 1.70 / 32 & 1.85 / 32 \\
\bottomrule
\end{tabularx}
\end{table}

One boundary crossing shows the mechanism concretely. We take Qwen3.6-35B-A3B
layer 0, the first layer with a mean--RMS disagreement at 25\% removal in the
4,096-token-row diagnostic, and within it the opposite-status pair with the
closest conditional means, selected without reference to downstream outcomes.
Expert 208 has mean/CV/RMS 0.03323/0.741/0.04136 over 32,041 routed tokens,
against 0.03359/0.637/0.03982 over 23,766 tokens for expert 38. Their
mean-to-RMS ranks move from 194 to 190 and from 189 to 201, respectively,
crossing the 192-expert retention boundary. Higher variability thus flips
selection across a conditional-mean margin of $0.00036$, for one pair whose
utility and presence in the benchmark checkpoints are not established here.

\subsection{Conditional RMS, mean, and corpus-sum reductions}
\label{app:aggregation}

The previous subsection compared conditional RMS against the conditional mean,
which isolates the second moment. A second and independent choice remains,
namely whether to condition on expert use at all, and the two are easily
conflated because both appear to make RMS ``weight large damage more.''
Separating them requires a third reduction that drops the conditioning. Let $g_i$
be any non-negative token quantity with conditional RMS $s_i$, let
$\pi_i=|\mathcal D_i|/|\mathcal D|$, and set $g_i$ to zero on tokens not routed
to $i$. The empirical average squared damage over \emph{all} calibration
tokens is
\begin{equation}
 \frac{1}{|\mathcal D|}\sum_{x\in\mathcal D}
 \mathbf 1\{i\in\routeset(x)\}g_i(x)^2
 =\pi_i s_i^2.
\label{eq:severity_exposure}
\end{equation}
An all-token RMS criterion would therefore rank $\sqrt{\pi_i}s_i$, which
differs from ranking $s_i$ by the exposure factor $\sqrt{\pi_i}$. Ranking
$s_i$ measures severity conditional on use, so an infrequent but severe
deletion is not discounted for being rare. This score differs from the summed
squared token quantity over all calibration tokens. Neither aggregation
guarantees preservation of capabilities absent from the calibration data.

To isolate aggregation, we fix the \rcs token quantity
$g_i(t)=w_i(t)\|f_i(t)-\con(t)\|_2$ and compare three reductions of the observed
damage at four pruning budgets,
\begin{equation}
\begin{aligned}
 \text{conditional RMS}\ &\ \sqrt{\tfrac{1}{n_i}\textstyle\sum_{t\in\mathcal D_i}g_i(t)^2},
 &\qquad \text{conditional mean}\ &\ \tfrac{1}{n_i}\textstyle\sum_{t\in\mathcal D_i}g_i(t),\\
 \text{corpus sum}\ &\ \textstyle\sum_{t\in\mathcal D}g_i(t)^2=|\mathcal D|\,\pi_is_i^2,
\end{aligned}
\label{eq:reductions}
\end{equation}
where $\mathcal D_i$ contains tokens routed to expert $i$, $n_i=|\mathcal D_i|$,
and $g_i(t)=0$ elsewhere. The corpus sum differs from the all-token mean squared
damage in Eq.~\plaineqref{eq:severity_exposure} only by the common factor
$|\mathcal D|$, so they induce the same ranking. We exclude selection frequency
$n_i$ because it discards $g_i$ entirely. Including it would conflate a change
in token quantity with a change in aggregation. Table~\ref{tab:main} compares
frequency as a criterion in its own right.

For the \rcs quantity, reverse KL follows the strict ordering
$\text{RMS}<\text{mean}<\text{sum}$ in all eight model--budget settings
(Figure~\ref{fig:aggregation}). Both choices thus point the same way under this
protocol, as conditioning on use beats weighting squared severity by exposure,
and squaring within the conditional reduction beats not squaring. The
all-token quantity measures summed squared residual scores across calibration
tokens. It remains a proxy and does not equal the error of
jointly pruning the selected set.

How much the reduction matters is easier to appreciate against the spread among
criteria, which the main text reports as the headline comparison. Between the
extreme reductions, KL spans $2.4\times$ and $1.7\times$
on Qwen3.6-35B-A3B and $1.15\times$ and $1.05\times$ on GLM-4.7-Flash at 25\%
and 50\% removal. The corresponding spans among REAP, \rcs, and \rcsloo in
Table~\ref{tab:loo_ablation} are $1.12\times$, $1.13\times$, $1.12\times$, and
$1.17\times$. The reduction span is thus larger on Qwen3.6-35B-A3B, comparable
on GLM-4.7-Flash at 25\%, and smaller there at 50\%. These spans describe
different sets of tested configurations. Because the REAP--\rcs--\rcsloo
comparison changes aggregation as well as token quantity, they do not establish
which factor contributes more. The factorial grid in
Appendix~\ref{app:factorial} separates these choices.

The ordering largely carries over to the two exact deletion quantities.
These comparisons use equal layerwise budgets on the same four shards as
Table~\ref{tab:loo_ablation}, applying the two conditional reductions to the
leave-one-out and refill residuals. For refill damage, RMS lowers reverse KL
relative to the conditional mean at all four GLM-4.7-Flash budgets, by
4.4\%--11.4\%, and at 12.5\% and 25\% on Qwen3.6-35B-A3B, while raising it at
37.5\% and 50\% on Qwen3.6-35B-A3B (0.20744 versus 0.18874, and 0.44397 versus
0.44143), for six of eight settings improved. Applied to the leave-one-out
residual, RMS lowers KL in seven of eight settings, failing only at 37.5\% on
Qwen3.6-35B-A3B (0.20504 versus 0.19196). RMS therefore improves KL in most
settings for both deletion quantities, with exceptions on Qwen3.6-35B-A3B
at higher removal budgets.

\begin{figure}[!t]
\centering
\includegraphics[width=\textwidth]{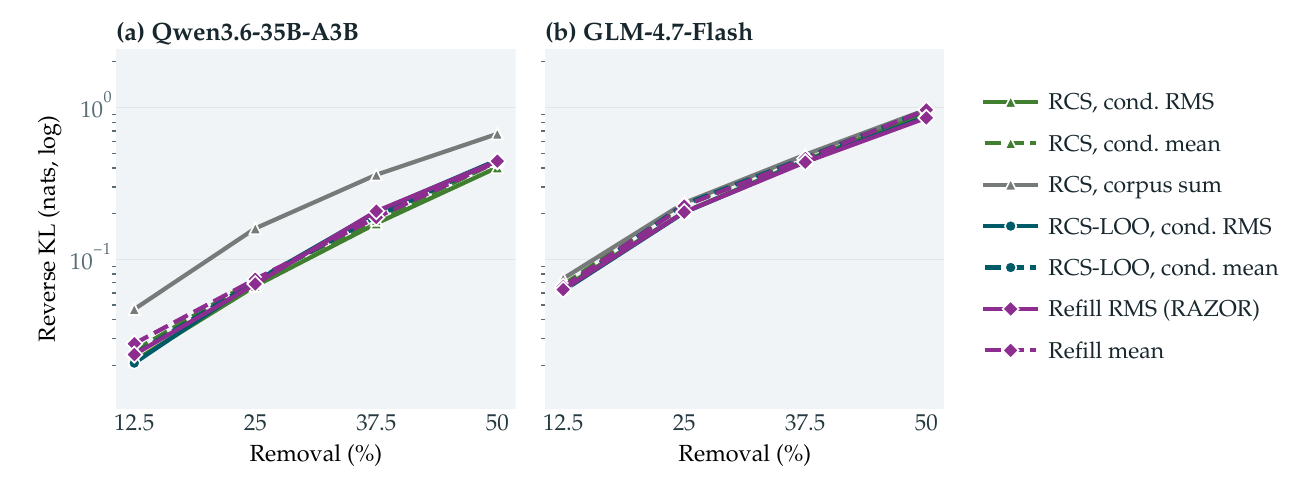}
\caption{\textbf{Reverse KL under different calibration reductions (lower is better).}
The \rcs quantity uses conditional RMS, conditional mean, and corpus sum.
The \rcsloo and \rcsrefill quantities each use the two conditional reductions.
Selection frequency is excluded because it discards the token quantity rather
than reducing it.}
\label{fig:aggregation}
\end{figure}

\subsection{What the refill term changes}
\label{app:refill}

Refill improves macro task performance across the four evaluated backbones.
To examine its local effect, Table~\ref{tab:refill} compares fixed-support and
refill damage on the two diagnostic backbones using the calibration statistics
of Table~\ref{tab:cv_stats}. Refill changes score scale in opposite directions
while preserving high token-profile similarity for most experts. Router averages are
token-weighted over routed events, whereas score ratios and token-axis cosines
use unweighted expert quantiles. The routed-event-weighted mean $w_i=1/k$ is a
renormalization consistency check, since each token's $k$ selected weights sum
to one.

\paragraph{Score scale and residual geometry.}
The median refill-to-fixed-support score ratio corresponds to a $2.81\%$
increase on Qwen3.6-35B-A3B and a $2.56\%$ decrease on GLM-4.7-Flash. Refill raises scores for
$85.2\%$ of Qwen3.6-35B-A3B experts and $14.3\%$ of GLM-4.7-Flash experts. Substituting the
reported mean routing weights into $(1-w_i)/(1-w_i+w_r)$ gives $0.922$ and
$0.850$, respectively. These values illustrate the denominator's attenuation
at average weights. They are neither token-averaged factors nor a decomposition
of the median score ratio.

For the reported residual-cosine statistics, the denominator
$\|\resid_i\|_2\|\resid_r\|_2$ is clamped below at $10^{-9}$.
The statistic is zero when either residual vanishes.
The mean residual cosines are $0.108$ on Qwen3.6-35B-A3B and $0.103$ on
GLM-4.7-Flash. By Eq.~\plaineqref{eq:refill_expand}, such a low positive cosine
is compatible with either an increase or a decrease in the numerator, depending
on the relative residual magnitude $\rho_i$. These aggregate statistics do
not identify $\rho_i$ or separate the contributions of direction and residual
magnitude. They also cannot recover the nonlinear factors, because
the score ratios summarize experts whereas the routing statistics summarize
routed events.

\paragraph{Damage-profile similarity and selection.}
Treating $\delta_i^{\mathrm{loo}}$ and $\delta_i$ as vectors over each expert's
routed tokens, their cosine has median $0.9942$ on Qwen3.6-35B-A3B and $0.9952$ on GLM-4.7-Flash,
with fifth percentiles $0.9630$ and $0.9715$. Most profiles are therefore close
in direction, although the minima of $0.7062$ and $0.5440$ show substantial
changes for some experts. Directional similarity alone does not fix the
ranking, since nearly proportional profiles with different scale factors still
reorder experts. The direction of the median score change does not determine
the KL change.
Despite the opposite median score shifts on the two models, refill lowers KL at
both Qwen3.6-35B-A3B budgets and at only one GLM-4.7-Flash budget, so these
local statistics do not predict its contribution to checkpoint quality.

\begin{table}[!t]
\centering
\small
\setlength{\tabcolsep}{4pt}
\renewcommand{\arraystretch}{1.14}
\caption{\textbf{Calibration statistics of routing weights and refill-induced score changes.}}
\label{tab:refill}
\begin{tabularx}{\columnwidth}{@{}>{\raggedright\arraybackslash}p{6.4cm}*{2}{>{\raggedleft\arraybackslash}X}@{}}
\toprule
\thead{Quantity} & \thead{Qwen3.6-35B-A3B} & \thead{GLM-4.7-Flash} \\
\midrule
\multicolumn{3}{@{}l}{\textit{Routing geometry (token-weighted)}} \\
Selected weight $w_i$ ($=1/k$) & 0.1250 & 0.2500 \\
Promoted pseudo-weight $w_r$ & 0.0740 & 0.1325 \\
$\cos(\resid_i,\resid_r)$ & 0.1080 & 0.1034 \\
\midrule
\multicolumn{3}{@{}l}{\textit{Score ratio $s_i^{\mathrm{rf}}/s_i^{\mathrm{loo}}$ (per expert)}} \\
Median & 1.0281 & 0.9744 \\
Experts with ratio $>1$ & 85.2\% & 14.3\% \\
\midrule
\multicolumn{3}{@{}l}{\textit{Token-axis $\cos(\delta_i^{\mathrm{loo}},\delta_i)$ (per expert)}} \\
Median & 0.9942 & 0.9952 \\
$p_5$ & 0.9630 & 0.9715 \\
Minimum & 0.7062 & 0.5440 \\
\midrule
\multicolumn{3}{@{}l}{\textit{Collection}} \\
Top-$k$ / experts per layer & 8 / 256 & 4 / 64 \\
Layers / scored experts & 40 / 10,240 & 46 / 2,944 \\
Routed token positions & 785,703 & 785,632 \\
\bottomrule
\end{tabularx}
\end{table}

\subsection{Reference, routing factor, and aggregation}
\label{app:factorial}

Table~\ref{tab:factorial_complete} reports fourteen scoring configurations for
two models at 25\% and 50\% removal. Its core grid crosses two output
references with two routing factors. Rows within each block compare token
quantities, while blocks compare reductions of the same quantity. The
conditional mean of $w_i\|f_i\|_2$ is REAP, and the conditional RMS of
$\frac{w_i}{1-w_i}\|f_i-c\|_2$ is \rcsloo. Adding refill damage gives five
token quantities across the corpus-sum $\sum_t g_i(t)^2$, conditional-mean,
and conditional-RMS blocks, giving fourteen rows per model.

One row is not exactly matched to the rest. The conditional-mean refill row
comes from the main sweep of Figure~\ref{fig:aggregation} rather than this
grid's evaluation collection, so comparisons involving it are cross-collection
contrasts, unlike the matched comparisons among the other thirteen rows.

\begin{table}[!t]
\centering
\small
\setlength{\tabcolsep}{3pt}
\renewcommand{\arraystretch}{1.15}
\caption{\textbf{Fourteen scoring configurations varying reference, routing factor, and aggregation.}
Reverse KL is in nats at 25\% and 50\% removal, with lower values better.
Thirteen rows share four evaluation shards. The conditional-mean refill
row comes from the main sweep in Figure~\ref{fig:aggregation}, which
differs in one shard. Comparisons involving this row are not fully matched.
Sum denotes $\sum_t g_i(t)^2$. No corpus-sum result is available for
$\frac{w_i}{1-w_i}\|f_i\|_2$. Refill damage with RMS is \method.}
\label{tab:factorial_complete}
\begin{tabularx}{\textwidth}{@{}p{3.5cm}>{\centering\arraybackslash}p{1.7cm}*{2}{>{\raggedleft\arraybackslash}X}@{\hspace{12pt}}*{2}{>{\raggedleft\arraybackslash}X}@{}}
\toprule
\thead{Token quantity} & \thead{Aggregate} & \multicolumn{2}{c}{\thead{Qwen3.6-35B-A3B}} & \multicolumn{2}{c}{\thead{GLM-4.7-Flash}} \\
\cmidrule(lr){3-4}\cmidrule(l){5-6}
 & & \thead{25\%} & \thead{50\%} & \thead{25\%} & \thead{50\%} \\
\midrule
$w_i\|f_i\|_2$ & Sum & 0.14712 & 0.67725 & 0.22309 & 0.89425 \\
$w_i\|f_i-c\|_2$ & Sum & 0.16005 & 0.67173 & 0.23513 & 0.96880 \\
$\frac{w_i}{1-w_i}\|f_i-c\|_2$ & Sum & 0.17659 & 0.69294 & 0.23339 & 0.91110 \\
$\frac{\|w_i\resid_i-w_r\resid_r\|_2}{1-w_i+w_r}$ & Sum & 0.18798 & 0.70264 & 0.23176 & 0.88086 \\
\midrule
$w_i\|f_i\|_2$ & Mean & 0.07469 & 0.45691 & 0.22852 & 1.01692 \\
$\frac{w_i}{1-w_i}\|f_i\|_2$ & Mean & 0.07882 & 0.48975 & 0.22922 & 0.97162 \\
$w_i\|f_i-c\|_2$ & Mean & 0.07357 & 0.44075 & 0.22804 & 0.94483 \\
$\frac{w_i}{1-w_i}\|f_i-c\|_2$ & Mean & 0.07151 & 0.45341 & 0.23017 & 0.95748 \\
$\frac{\|w_i\resid_i-w_r\resid_r\|_2}{1-w_i+w_r}$ & Mean & 0.07386 & 0.44143 & 0.22481 & 0.96456 \\
\midrule
$w_i\|f_i\|_2$ & RMS & 0.06869 & 0.43948 & 0.20813 & 0.86866 \\
$\frac{w_i}{1-w_i}\|f_i\|_2$ & RMS & 0.09535 & 0.64640 & 0.20865 & 0.90296 \\
$w_i\|f_i-c\|_2$ & RMS & 0.06674 & 0.40355 & 0.20483 & 0.92318 \\
$\frac{w_i}{1-w_i}\|f_i-c\|_2$ & RMS & 0.07144 & 0.45134 & 0.20353 & 0.87107 \\
$\frac{\|w_i\resid_i-w_r\resid_r\|_2}{1-w_i+w_r}$ & RMS & 0.06874 & 0.44409 & 0.20461 & 0.85460 \\
\bottomrule
\end{tabularx}
\end{table}

\paragraph{Evaluation and aggregation.}
Within each model, thirteen of the fourteen rows share four evaluation shards
with matched domain sequences and token counts. Let $L_{s,d}$ denote mean
reverse KL in shard $s$ and domain $d$, with $n_{s,d}$ scored tokens. The
estimator pools tokens within domains and then weights the eight domains
equally,
\begin{equation}
 L=\frac{1}{8}\sum_{d=1}^{8}
 \frac{\sum_{s=0}^{3} n_{s,d}L_{s,d}}{\sum_{s=0}^{3} n_{s,d}}.
\label{eq:factorial_estimator}
\end{equation}
The entries are shard-based point estimates, not estimates of variability
across independent runs.

\paragraph{Relation to the matched component study.}
The factorial and matched component studies use overlapping evaluation data
but differ in token coverage within one of their four subsets. We therefore
report each study separately and do not treat their agreement as an
independent replication. The mean-refill row is explicitly identified as a
cross-study comparison.

\paragraph{No single scoring component is uniformly beneficial.}
The effect of each component depends on the other scoring choices.
Even with mean aggregation fixed, the effect of the routing factor depends on
the output reference. On GLM-4.7-Flash at 50\% removal, replacing $w_i$ with
$w_i/(1-w_i)$ lowers reverse KL from 1.01692 to 0.97162 for output magnitudes,
but raises it from 0.94483 to 0.95748 for consensus residuals. For fixed-support
deletion damage, RMS lowers reverse KL relative to the conditional mean in all four
displayed model--budget settings, although the Qwen3.6-35B-A3B differences are small.

\paragraph{The reduction ordering depends on the token quantity.}
The grid reproduces the \rcs ordering
$\text{RMS}<\text{mean}<\text{sum}$ from Appendix~\ref{app:aggregation} at both
displayed budgets on both models, but this ordering does not hold for every
quantity. Conditional RMS has the lowest KL in 15 of the 16 cells with all
three reductions, including four cross-collection comparisons involving mean
refill. The exception is refill at 50\% removal on Qwen3.6-35B-A3B.
On GLM-4.7-Flash, corpus sum yields lower KL than conditional mean for output
magnitudes at both budgets (0.22309 versus 0.22852, and 0.89425 versus 1.01692)
and for both residual quantities at 50\% removal (0.91110 versus 0.95748, and
0.88086 versus 0.96456), reversing the middle two terms in four of the sixteen
cells. On Qwen3.6-35B-A3B, mean yields lower KL than sum in all eight available
comparisons, with sum yielding 1.48--2.55$\times$ the KL of mean. RMS beats sum
in all sixteen matched comparisons, but the mean--sum ordering depends on the
quantity and backbone.

\paragraph{The reference and the reduction interact.}
For the four non-refill quantities in the core grid, RMS lowers reverse KL
relative to mean in 14 of the 16 quantity--model--budget cells. Both exceptions are the routing factor without
the consensus reference, $\frac{w_i}{1-w_i}\|f_i\|_2$, on Qwen3.6-35B-A3B, where RMS raises
KL from 0.07882 to 0.09535 at 25\% removal and from 0.48975 to 0.64640 at 50\%.
Read in the other direction, the benefit of consensus subtraction depends on
which reduction is in use. With
the routing factor held at $w_i$, subtracting the consensus lowers KL in all
four settings under the mean but in three under RMS, raising it from 0.86866
to 0.92318 on GLM-4.7-Flash at 50\% removal. With the leave-one-out factor,
it lowers KL in all four settings under RMS but in three under the mean,
raising it from 0.22922 to 0.23017 on GLM-4.7-Flash at 25\% removal. The two
components are therefore not separable additive improvements.

\paragraph{An aggregation-matched baseline narrows the residual advantage.}
Comparing our criteria against REAP confounds the reference with the
reduction, so the grid also aggregates REAP's token quantity by conditional
RMS. Under that matched baseline, \reaprms gives lower KL than \rcsloo in
both Qwen3.6-35B-A3B settings and slightly lower KL on GLM-4.7-Flash at 50\%,
so the benefit of RMS over mean aggregation for fixed-support deletion damage
does not imply that the complete \rcsloo criterion outperforms \reaprms.
This comparison changes both the output reference and the
survivor-renormalization factor. \rcsrefill fares
better against the same baseline. It attains the lowest KL of all fourteen
rows on GLM-4.7-Flash at 50\% (0.85460) and beats \reaprms on GLM-4.7-Flash at
both budgets, while trailing it on Qwen3.6-35B-A3B at both budgets (0.06874
versus 0.06869 and 0.44409 versus 0.43948), where \rcs remains lowest.
With both quantities aggregated by the conditional mean instead, refill yields
lower KL than REAP in all four settings (0.07386 versus 0.07469 and 0.44143
versus 0.45691 on Qwen3.6-35B-A3B, 0.22481 versus 0.22852 and 0.96456 versus
1.01692 on GLM-4.7-Flash), although mean \rcsloo still edges out mean refill
on Qwen3.6-35B-A3B at 25\% (0.07151). These mean-refill contrasts use the
separate main-sweep shards rather than the identical batches used elsewhere in
this table.

Taken together, the grid shows that local exactness of a token quantity does
not determine its empirical ordering after joint pruning, which is the same
backbone dependence visible in the main-sweep comparison of
Table~\ref{tab:loo_ablation}. This is why the choice of \method rests on the
benchmark results rather than on these checkpoint-fidelity contrasts, which
do not identify the source of the downstream gains.

\subsection{Matched component contrasts at two removal budgets}
\label{app:matched_components}

The matched component study uses four evaluation subsets at 25\% and 50\%
removal, pools tokens within each
axis, weights the eight axes equally, and varies one scoring choice at a time
with the others held fixed. This separates the effects of the output reference,
survivor renormalization, refill, and aggregation under a common protocol.

Consensus subtraction lowers reverse KL in every matched setting. With routing
weights and mean aggregation fixed, it does so by 1.5\%/3.5\% on
Qwen3.6-35B-A3B and 0.2\%/7.1\% on GLM-4.7-Flash at 25\%/50\% removal. With
residuals and the LOO factor fixed, RMS likewise lowers KL in all four
settings, by 11.6\%/9.0\% on GLM-4.7-Flash but by only 0.05\%/0.44\% on
Qwen3.6-35B-A3B, which is too small a margin to read as separation. Beyond
these two components the picture splits by backbone. Relative to RMS without
LOO, \rcsloo lowers GLM-4.7-Flash's reverse KL but raises Qwen3.6-35B-A3B's at
both budgets (Table~\ref{tab:loo_ablation}). For refill damage, RMS changes
reverse KL by $-7.0\%$/$+0.6\%$ on Qwen3.6-35B-A3B against
$-9.0\%$/$-11.4\%$ on GLM-4.7-Flash.

The two fidelity metrics also disagree, which matters because each is a
defensible summary of the same predictions. In the RMS contrast at fixed
residuals and LOO factor, RMS raises Qwen3.6-35B-A3B's excess perplexity by
16.5\%/2.7\% while lowering its reverse KL, whereas it lowers
GLM-4.7-Flash's by 12.5\%/13.0\% in agreement with KL. The refill contrast
shows the same split, changing excess perplexity by $-8.5\%$/$+6.9\%$ on
Qwen3.6-35B-A3B and $-8.9\%$/$-13.2\%$ on GLM-4.7-Flash.
Appendix~\ref{app:ppl_view} examines this disagreement directly on shared
routing groups. Here it means that no ordering of the proxies holds under both
metrics at once.

Whether the gains reach beyond calibration-covered task types is a separate
question, which we address by averaging the four covered and four held-out axes
separately with equal weight within each group, using excess perplexity
$\exp(\Delta\mathrm{NLL})-1$. REAP's held-out reverse KL is $2.1$--$2.9\times$ its in-domain
value across the two budgets and its held-out excess perplexity is
$3.4$--$3.8\times$ its in-domain value. Domain-specific comparisons appear in
Appendices~\ref{app:per_axis} and~\ref{app:domain_fidelity}.

Across all eight axes, GLM-4.7-Flash improves on both metrics, with \rcsloo reducing KL by
10.9\%/14.3\% and excess perplexity by 8.9\%/17.1\% relative to REAP, and
\method reducing them by 10.5\%/16.0\% and 8.7\%/19.3\%. Qwen3.6-35B-A3B again
splits, as \rcsloo lowers its all-axis KL by 4.3\%/1.2\% while changing excess
perplexity by $-0.7\%$/$+12.6\%$, and \method lowers KL by 7.9\%/2.8\% while
changing excess perplexity by $-9.5\%$/$+6.6\%$. Refill is the one component
that helps that model on both metrics, lowering its excess perplexity by
8.9\%/5.3\% relative to \rcsloo.

\subsection{Selection diagnostics}
\label{app:selection}

The comparisons above measure score values and checkpoint fidelity, neither of
which shows what changed in the retained set. Three probes address that
question from different angles. Conditioning on keep-set disagreements asks
whether the tokens where two criteria retain different experts are where their
outputs diverge. Joint residual statistics ask whether the residuals inside a
selected deletion set cancel or reinforce. Single-expert ablations ask whether
local score order survives deployment, once the router reselects from the
remaining pool. These probes characterize selection and local interactions.
They do not measure redundancy among the final retained experts or establish
the mechanism behind the benchmark gains.

\subsubsection{Output differences conditioned on selection disagreements}
\label{app:selection_conditions}

We compare \rcsloo and REAP at 50\% removal on
16 batches spanning eight axes for each of Qwen3.6-35B-A3B and GLM-4.7-Flash. This collection is
separate from the four-shard main sweep and is subject to the checkpoint-matching
limits in Appendix~\ref{app:benchmark_details}. It contains no \method
predictions, so this probe is not repeated for refill.

Each conditioning variable uses routed events from a specified layerwise
keep-set disagreement. For routing concentration, eligible experts are retained
by \rcsloo but not by \rcs. We select the event with
the largest $1/(1-w_i)$ across layers for each token and use its normalized
routing entropy. For output consensus, eligible experts are retained by residual
RMS without LOO but not by REAP. We select the event with the largest residual
damage relative to the expert's calibration mean and use its
$\kappa=\|c\|_2^2/\sum_j w_j\|f_j\|_2^2$ and relative residual norm.
For relative impact, eligible experts are retained by \rcsloo but not by
conditional-mean fixed-support scoring. We take the largest
$\delta_i^{\mathrm{loo}}/\mathbb{E}_{\mathrm{cal}}[\delta_i^{\mathrm{loo}}\mid i\in\routeset]$ across their
routed events, or zero if none is observed. This ratio measures shift relative
to an expert's mean damage, not event rarity or preservation of rare capabilities.

All diagnostics use model-wide quantiles on a common finite-token subset.
Within each bin, we average \rcsloo-minus-REAP reverse KL within batches and
then equally across batches, with 5,000 paired-batch bootstrap resamples.
The most concentrated routing, highest-consensus, and highest-relative-impact
bins give Qwen3.6-35B-A3B and GLM-4.7-Flash differences of $-0.14/-0.33$, $-0.18/-0.28$, and
$-0.25/-0.47$, respectively. Tokens where disagreeing experts carry the most
concentrated routing, highest consensus, or highest relative impact tend to
show the largest criterion gaps. These associations are conditional on
selection disagreements and do not generalize to unconditioned tokens.

\subsubsection{Interactions within the selected deletion sets}
\label{app:joint_numerator}

The joint identity in Eq.~\plaineqref{eq:setshift} motivates a separate test
of whether cross-expert residuals cancel or reinforce within score-selected
deletion sets. We select sets with \rcsloo at 25\% and 50\% removal and
accumulate their interaction Gram matrices on the diagnostic calibration
collections of Appendix~\ref{app:razorcal}. Define
$u_i(x)=w_i(x)(f_i(x)-c(x))$ for routed experts and zero otherwise, and let
$G_{ij}=\sum_x\langle u_i(x),u_j(x)\rangle$ within a layer. For deletion set
$\mathcal A_\ell$, the interaction ratio is
\begin{equation}
 \gamma_\ell=
 \frac{\sum_{i,j\in\mathcal A_\ell}G_{ij}}
      {\sum_{i\in\mathcal A_\ell}G_{ii}},
\label{eq:joint_ratio}
\end{equation}
provided the denominator is positive. The common layer output scale cancels
in the ratio. Values below one indicate
net cancellation, while values above one indicate reinforcement. We compute
one ratio per layer and report equal-layer quantiles over 40 Qwen3.6-35B-A3B and 46 GLM-4.7-Flash
layers.

At 25\% and 50\% removal, median ratios are 1.080/0.921 on Qwen3.6-35B-A3B and
0.992/0.811 on GLM-4.7-Flash. At 50\%, 29 of 40 Qwen3.6-35B-A3B layers and all 46 GLM-4.7-Flash layers
show net cancellation. Sets selected by \method from the same accumulators
give median ratios of 1.077/0.925 on Qwen3.6-35B-A3B and 0.997/0.831 on
GLM-4.7-Flash, with cancellation in 29 of 40 and 42 of 46 layers at 50\%.
The degree of cancellation is therefore similar across criteria at both
budgets. These ratios characterize only the squared numerator of the
fixed-support perturbation. They omit the token-dependent factor
$(1-W_{\mathcal A}(x))^{-2}$ and router refill, and the identity
$\sum_i u_i(x)=0$ constrains their cross terms. They therefore do not
measure joint deletion loss. The fixed-support counterfactual is undefined
when no routing mass survives.

\subsubsection{Single-expert removal with deployed routing}
\label{app:per_expert_ablation}

This probe tests whether local saliency transfers to end-to-end single-expert
damage under deployed routing. We uniformly sample
32 experts and intervene on each in six layers spanning shallow, middle, and
deep positions, giving 192 ablations per model. Each intervention masks one
expert's router score to $-\infty$ and reselects and renormalizes the top-$k$
from the remaining pool, potentially admitting previously inactive experts.

We compare the intervened model $q$ with the unpruned model $p$ on 16 shared
batches, two per evaluation axis, using assistant-token reverse KL
$D_{\mathrm{KL}}(q\|p)$. The resulting damage estimate pools 78,091 scored
tokens on Qwen3.6-35B-A3B and 79,975 on GLM-4.7-Flash, weighting tokens rather
than axes equally. For each scoring rule, we
compute Spearman correlation over the 32 interventions within each layer
and average the six correlations equally. The saliency statistics come from the
diagnostic calibration collections of Appendix~\ref{app:razorcal}.

Mean correlations for \rcsloo, REAP, and EAN are 0.31, 0.34, and 0.31 on
Qwen3.6-35B-A3B, and 0.19, 0.18, and 0.10 on GLM-4.7-Flash, respectively. Every Qwen3.6-35B-A3B correlation
is positive, but several GLM-4.7-Flash values are near zero and EAN is negative in
two layers. The refill score \method gives 0.31 on Qwen3.6-35B-A3B and 0.18 on
GLM-4.7-Flash. The point estimates show no consistent advantage across models.
This probe measures how well local output shifts under single-expert removal
transfer to sample-wide prediction KL after router reselection. Better ordering
of all single deletions does not guarantee better budgeted selection, and the
probe accounts for neither the fixed-support identity nor joint deletion sets.

\section{Calibration Stability and Sensitivity}
\label{app:calibration_robustness}

We examine how expert selection changes with calibration budget and how
performance varies across calibration corpora. The budget study compares
expert sets selected from disjoint token subsets for REAP and all three
residual criteria. The corpus study evaluates reference-token PPL for
24 \rcsloo checkpoints. Its configurations also differ in token budget
and conversation structure, so the PPL contrasts do not isolate corpus choice.

\subsection{Calibration budget and selection stability}
\label{app:calibration_budget}

The Qwen3.6-35B-A3B collection contains 32 shards and 2,151,112 tokens. At each
budget, we draw pairs of disjoint subsets with equal shard counts, merge
statistics within each subset, and select experts. Agreement is the sum of
layerwise retained-set intersection sizes divided by the total number of
retained experts across scored layers. We report averages over four subset-pair
draws. Additive sufficient statistics allow scores to be recomputed for each
subset without additional model evaluations. All four criteria use the same subsets.
Subset agreement compares the two disjoint selections. Full-set agreement
compares each selection with that from all 32 shards and averages the two
values. Only the former measures agreement from disjoint calibration data,
because the full-set reference includes both subsets.

\begin{figure}[!t]
\centering
\includegraphics[width=\textwidth]{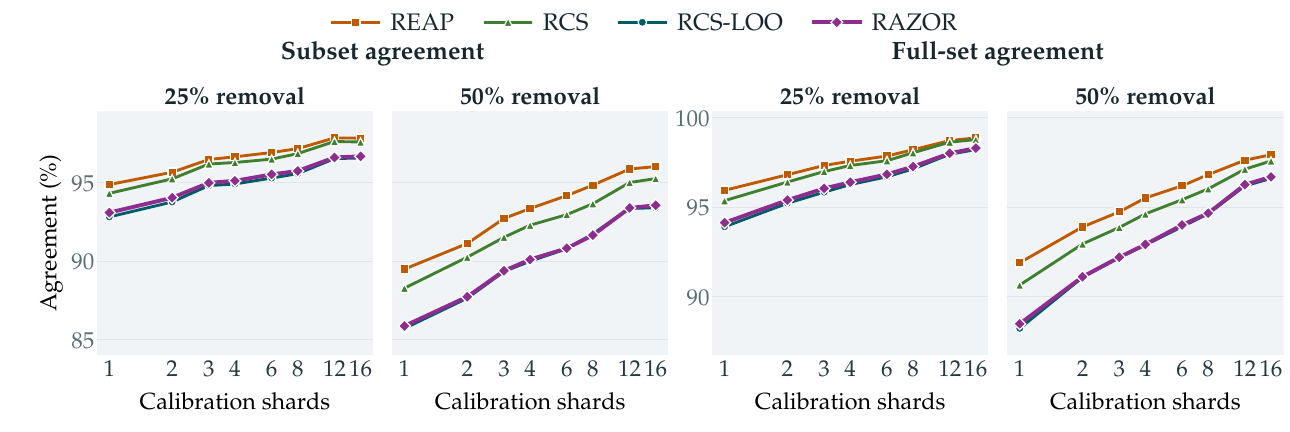}
\caption{\textbf{Calibration budget and selection agreement on Qwen3.6-35B-A3B.}
Subset agreement (left pair) compares two disjoint calibration subsets,
while full-set agreement (right pair) compares each against the full 32-shard
selection. Tokens per shard count are given in the text.}
\label{fig:calibration_budget}
\end{figure}

Selection agreement generally increases with calibration budget for all four
criteria (Figure~\ref{fig:calibration_budget}). For \rcsloo, disjoint-subset
agreement rises from 85.72\% to 93.39\% at 50\% removal as the budget grows from
approximately 65,500 to 1.08 million tokens. At 25\% removal, it rises from
92.82\% to 96.58\%. In the final budget step, disjoint-subset agreement
increases by only 0.01 percentage points for \rcsloo at 50\% removal.
At 25\%, the corresponding changes across criteria range from $-0.03$ to
$+0.07$ percentage points. These small changes across four draws do not
establish whether agreement has plateaued.

Across eight calibration budgets, two removal ratios, and two agreement
measures, REAP has the highest selection agreement and \rcs the second
highest in all 32 settings. \method exceeds \rcsloo in 31 settings.
At the smallest budget and 50\% removal, disjoint-subset agreement is 89.50\%
for REAP and 85.72\% for \rcsloo. At approximately 1.08 million tokens,
their disjoint-subset gap narrows to 2.64 percentage points, alongside a
narrowing full-set agreement gap. The residual criteria therefore produce less
stable retained sets than REAP throughout this collection, and more calibration
data reduces but does not close that difference.

Selection stability and task performance measure different properties. These
means over four subset-pair draws do not establish whether the lower stability
of residual criteria affects fidelity or downstream performance. That question
requires evaluating the checkpoints selected at each calibration budget.

\subsection{Corpus configurations and evaluation comparability}

The composition study compares full \calib, a tool-filtered subset,
self-generated text, WikiText, and two complementary halves of the filtered
subset at 25\% and 50\% removal. Filtering retains 1,790 of 2,048 examples by
excluding 249 tool-calling and nine coding conversations with a tool-role
message before the final assistant turn. The complementary halves combine
alternating shards, $\{0,2\}$ and $\{1,3\}$. The filtered corpus retains
\calib's mixed-source provenance and is not a human-authored control.
Self-generated variants use the corresponding unpruned backbone to continue
the prefix before the first assistant turn. Since later dialogue turns are not
preserved, these variants change conversation structure as well as response
content. For article-style language-modeling text, we use the raw WikiText-2
training split of curated Wikipedia articles \citep{wikitext2017}.

\paragraph{Perplexity across configurations.}
We evaluate 24 \rcsloo checkpoints, one per configuration, model and removal
ratio. Each is scored on 128 sequences of 2,048 tokens, yielding 262,016
next-token targets. PPL exponentiates the mean reference-token NLL over all
next-token positions, not only assistant turns. The unpruned references are
$5.781$ on Qwen3.6-35B-A3B and $3.563$ on GLM-4.7-Flash. Comparisons across
configurations are subject to the input-matching limits below.

\paragraph{Input matching.}
The Qwen3.6-35B-A3B full-\calib and self-generated runs share the original model's
tokenization and evaluation rows. Their conversation pool differs from
the eight-axis fidelity pool. It contains the first 128 conversations in source
order that fill the token window, each truncated to its first 2,048 tokens.
On these matched inputs, self-generated calibration yields lower PPL than
full \calib at both removal ratios ($6.468$ versus $7.544$ at 25\%, and
$6.991$ versus $9.028$ at 50\%). This contrast still changes calibration
token budgets and conversation structure as well as response content.

Across the larger collection, self-generated calibration yields lower PPL than
tool-filtered calibration on Qwen3.6-35B-A3B but higher PPL on GLM-4.7-Flash at both
budgets. WikiText calibration produces a 50\%-removal Qwen3.6-35B-A3B checkpoint with
PPL $4.831$, below the unpruned reference. The complementary half-splits differ
by $0.295$ on Qwen3.6-35B-A3B at 25\% removal and by $0.170$ on GLM-4.7-Flash at 50\%.
Outside the matched Qwen3.6-35B-A3B contrast above, identical evaluation
inputs have not been established across all configurations. Together with
differences in calibration token budget and conversation structure, this
limits interpretation of the PPL values as effects of calibration source.
The study does not compare calibration robustness against REAP.

\section{Output Fidelity and Response Diversity}
\label{app:axis_results}

We examine fidelity across evaluation axes and routing conditions, and
compare it with response diversity and form. The analyses show domain-level
regressions alongside aggregate KL gains (Appendix~\ref{app:per_axis}), differences between KL and
reference-token likelihood, and generation changes that do not follow the
task-performance ranking.

Resampling units differ across analyses, so their intervals are not
interchangeable. Per-axis fidelity resamples paired evaluation shards,
domain-wise and routing-stratified analyses resample paired batches, and
response diversity resamples questions. All reported 95\% intervals are
pointwise and unadjusted for multiple comparisons.

\subsection{Per-axis fidelity results}
\label{app:per_axis}
The main sweep pools four disjoint stratified shards, weighting tokens within
axes and the eight axes equally. Table~\ref{tab:loo_ablation} reports its
absolute and relative reverse KL at 25\% and 50\% removal. The \rcsloo
checkpoints have lower reverse KL than REAP on both models at both budgets,
with larger relative reductions on GLM-4.7-Flash.

At 25\% and 50\% removal, the axis-wise comparison separates four
calibration-covered axes (\emph{math, code, instruction following, tool use})
from four held-out axes (\emph{chat, creative, safety, SQL}). \rcsloo has lower
reverse-KL point estimates in 11 of 16 axis--ratio cells on Qwen3.6-35B-A3B and 13 of 16
on GLM-4.7-Flash. To distinguish favorable and adverse directions, we enumerate all
$4^4=256$ ordered resamples of the four paired shards. In each resample we
recompute the token-weighted KL of both methods within the axis and take
$100(L_{\mathrm{method}}/L_{\mathrm{REAP}}-1)$. The 2.5th and 97.5th percentiles
give pointwise intervals, of which 23 lie below zero, four lie above zero,
and five include zero. Per setting, the below/above/includes-zero counts are 5/1/2 for
Qwen3.6-35B-A3B at both budgets, 6/1/1 for GLM-4.7-Flash at 25\%, and 7/1/0 for GLM-4.7-Flash at 50\%.
The four adverse cells are
creative writing on Qwen3.6-35B-A3B at both budgets, SQL on GLM-4.7-Flash at 25\%, and mathematics
on GLM-4.7-Flash at 50\%. Thus the aggregate improvement coexists with identifiable
axis-level regressions. The same procedure applied to \method gives lower
point estimates in 11 of 16 cells on Qwen3.6-35B-A3B and 14 of 16 on GLM-4.7-Flash,
with 24 intervals below zero, three above zero, and five including zero
(6/0/2 and 5/1/2 on Qwen3.6-35B-A3B at 25\%/50\%, 6/1/1 and 7/1/0 on GLM-4.7-Flash).
Refill removes the adverse Qwen3.6-35B-A3B creative-writing interval at 25\%,
but the other three adverse cells remain. All comparisons share the same evaluation
shards and ordered resamples within each axis. Their intervals are paired
contrasts, not estimates of variation across checkpoints or calibration draws.
This view therefore decomposes the same predictions rather than replicating the
experiment.

\subsection{Domain-wise output fidelity}
\label{app:domain_fidelity}

Figure~\ref{fig:radar-four-criteria} resolves fidelity by domain for four
backbones at both removal ratios, using predictions collected by
applying criterion-specific expert masks to each original model. These
mask-based diagnostics are separate from the downstream benchmark runs.
The figure compares REAP, \rcs, \rcsloo and \rcsrefill using reverse KL $D_{\mathrm{KL}}(q\|p)$ in nats over four
calibration-covered and four held-out domains. Each domain estimate averages
token means equally over six batches. ID/OOD summaries then give equal weight
to domain means, and intervals use 10,000 paired-batch bootstrap draws, shared
across criteria so every comparison stays paired on the same resampled batches.

The residual criteria improve on REAP in most domains, although the
held-out results on Qwen3.6-35B-A3B are mixed. For \rcsloo versus REAP, lower-KL domain counts
are 6/4/7/6 at 25\%
removal and 7/5/7/7 at 50\%, in GLM-4.7-Flash, Qwen3.6-35B-A3B, DeepSeek-V4-Flash-0731, and Hy3 order, each
out of eight domains. Qwen3.6-35B-A3B's held-out equal-domain mean is higher than REAP's
by 7.3\% at 25\% and 2.3\% at 50\%, with relative-change intervals crossing
zero at both budgets, whereas the other three models have lower held-out
means at 50\%. Refill improves the point estimates in this group. For \method the
corresponding counts are 5/6/6/6 at 25\% and 7/5/8/7
at 50\%, and its Qwen3.6-35B-A3B held-out mean changes by $-3.9\%$ at 25\% and
$+0.2\%$ at 50\% relative to REAP, again with intervals crossing zero and the
other three models again lower at 50\%. Held-out fidelity on that backbone is
therefore where our advantage is weakest, and it is the same backbone whose
excess perplexity moves against reverse KL in
Appendix~\ref{app:matched_components}.

Both removal ratios share reference predictions and token positions, with at
most 1,024 valid positions per batch and no further subsampling. For
\rcsloo the 50\% row uses the same predictions as
Figure~\ref{fig:routing-four-criteria}, but aggregates them by domain rather
than routing decile. Neither view evaluates full
generated trajectories.

Each spoke uses an independent KL scale spanning all four point estimates,
padded by the median bootstrap half-width across criteria. Curves are
comparable along a spoke, but distances are not comparable across spokes,
panels, or removal ratios. Eight of 256 whiskers extend beyond the frame, so
their upper ends are clipped rather than omitted. Intermediate labels give
mid-axis KL values.

\subsection{Routing-stratified output fidelity}
\label{app:mechanism}

The LOO factor increases with an expert's routing weight, motivating an
analysis of whether residual criteria improve fidelity most under
concentrated routing.
Figure~\ref{fig:routing-four-criteria} groups assistant tokens by routing
concentration within each learned-router MoE layer. We rank tokens by negative
normalized routing entropy $-H_r$ and form ten quantile groups with shared
ID/OOD boundaries. A common finite-value mask pairs methods and metrics,
including the accompanying entropy and PPL measurements. We exclude
layer--batch--decile cells with fewer than 20 tokens, then average within-cell
token means equally over valid layers and batches. This grouped estimator
differs from the main sweep's equal-axis estimator. The horizontal axis runs
from diffuse to concentrated routing.

Four calibration-covered and four held-out domains are evaluated separately.
Within each ID/OOD group, 2,000 paired-batch bootstrap resamples preserve all
ten deciles. The 25\% and 50\% curves share original-model routing profiles,
token masks, quantile boundaries, and valid layer--batch--decile cells, but
compute intervals separately.

Learned-router saliency scoring does not apply to DeepSeek-V4-Flash-0731's
three fixed-hash-router layers, which are excluded from the curves. Across
criteria, these layers use identical frozen keep sets based on occurrence
counts in the token-to-expert table, not calibration-token frequencies.
Remapping preserves $k$ distinct routes per token.

At 50\% removal, held-out reverse KL is $2.1$--$3.0\times$ its corresponding
in-distribution value. From the least to the most concentrated routing decile,
\rcsloo's relative advantage widens by $15.6/17.2$ percentage points on Qwen3.6-35B-A3B
ID/OOD and by $4.3$ points on GLM-4.7-Flash ID. Nominal paired-bootstrap intervals exclude
zero for these three endpoint contrasts but cross zero for the other five, giving mixed
evidence for a concentration-dependent advantage.

Aggregate point estimates favor \rcsloo by $14.4\%$ on GLM-4.7-Flash OOD,
$15.4/10.3\%$ on Hy3 ID/OOD, and $7.5/8.9\%$ on DeepSeek-V4-Flash-0731 ID/OOD. Qwen3.6-35B-A3B's OOD
KL reduction is near zero ($-1.5\%$, CI $[-5.3,+2.7]$). The same estimator applied to
\method, whose 50\% predictions share the reference predictions and token
mask, favors it by $15.6\%$ on GLM-4.7-Flash OOD, $17.2/12.9\%$ on Hy3 ID/OOD, and
$10.0/13.1\%$ on DeepSeek-V4-Flash-0731 ID/OOD, with Qwen3.6-35B-A3B OOD again nearly tied
($+0.5\%$, CI $[-3.7,+4.9]$). Relative to REAP, \rcsloo's mean absolute
entropy drift is lower by $15\%$, $22\%$, $11\%$, and $21\%$ on GLM-4.7-Flash,
Qwen3.6-35B-A3B, DeepSeek-V4-Flash-0731, and Hy3, respectively. Unlike the
KL reductions above, these entropy figures average over all eight axes rather
than an ID or OOD group, so they do not describe held-out behavior.

Adding the leave-one-out factor does not uniformly improve routing-stratified
fidelity under conditional RMS. Averaged over deciles, \rcs attains
the largest KL reduction against REAP in three of the eight model--group
panels. These are both Qwen3.6-35B-A3B panels ($+15.9\%$ ID and $+11.9\%$ OOD, against
$+9.4\%$ and $-0.4\%$ for \method) and DeepSeek-V4-Flash-0731 ID ($+11.4\%$
versus $+9.7\%$). \method leads the remaining five panels by $1.2$--$4.3$
percentage points over the better of \rcs and \rcsloo. The largest
\rcs--\rcsloo gap occurs on Qwen3.6-35B-A3B OOD, where \rcsloo and \method both fall below REAP
on average, while \rcs improves on it. The gap narrows at the tenth decile,
where the reductions are $+16.5\%$, $+11.8\%$, and $+12.4\%$ for \rcs,
\rcsloo, and \method, respectively. These equal-decile means at 50\% removal
do not establish a downstream ranking or replace the matched equal-axis
estimates in Table~\ref{tab:loo_ablation}. The \rcs--\rcsloo contrast shows
that the gate factor's effect on fidelity depends on the backbone even with
aggregation held fixed.

Entropy drift complements reverse KL by measuring changes
in predictive dispersion. Predictive entropy itself is neither accuracy,
calibration, nor response diversity, and neither higher nor lower entropy
alone is preferable.

\begin{figure*}[!t]
\centering
\includegraphics[width=\textwidth]{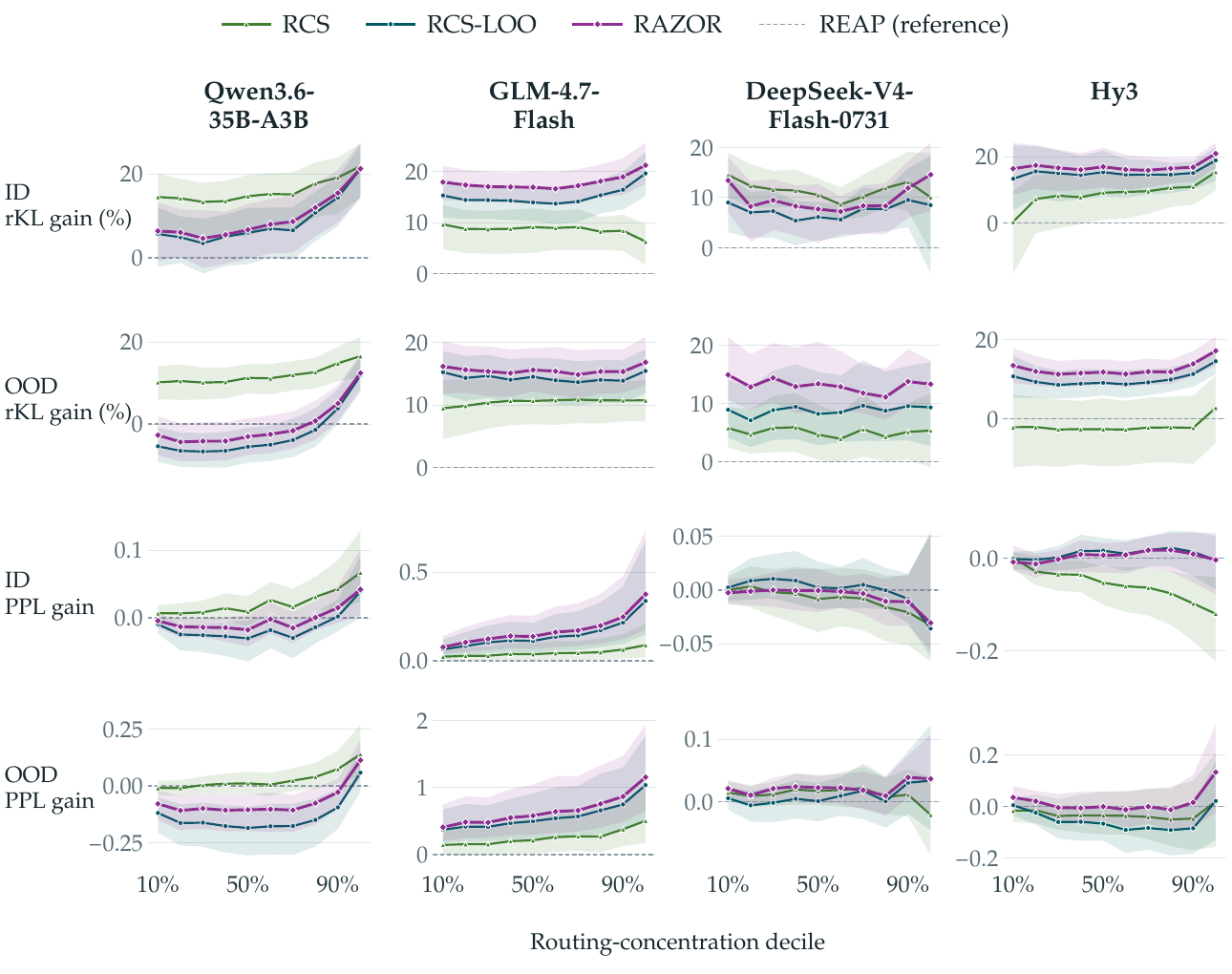}
\caption{\textbf{Routing-stratified fidelity at 50\% removal.}
The three residual criteria use conditional RMS and are evaluated on the same
batches, reference prefixes, and token positions.
All panels compare against REAP at the same decile, with a dashed
zero line and higher values indicating better fidelity. The top two rows show
reverse-KL reduction in percent. The bottom two show REAP's
$|\mathrm{PPL}_q-\mathrm{PPL}_p|$ minus the criterion's, in PPL points
(Equation~\ref{eq:story_ppl}), so positive values mean closer to the original.
Differences avoid unstable ratios where REAP's deviation approaches zero
($0.0016$ on Hy3 ID). Bands are pointwise 95\% paired batch-bootstrap intervals,
and panels use independent vertical scales.}
\label{fig:routing-four-criteria}
\end{figure*}

Reverse KL and perplexity deviation give different rankings. \method lowers
reverse KL relative to REAP in 73 of 80 deciles, with all seven exceptions in
Qwen3.6-35B-A3B's held-out group. On perplexity deviation, however,
it is farther from the original than REAP in $7/10$ ID deciles on
Qwen3.6-35B-A3B and $9/10$ on DeepSeek-V4-Flash-0731, while remaining closer
in all twenty GLM-4.7-Flash ID and OOD deciles. Across all groups it preserves
PPL more closely in 47 of 80 deciles. These are the same paired point
comparisons summarized in Appendix~\ref{app:ppl_view}, not separate
experiments.

The incremental effect of refill depends on the budget and domain group.
At 50\% removal, \method exceeds \rcsloo by $1.0$--$4.3$ percentage points
in REAP-relative KL reduction, averaging deciles equally within each of the
eight model--domain groups. Their marginal bands overlap throughout on
Qwen3.6-35B-A3B. At 25\% removal, the ordering varies by model. In the
calibration-covered groups, \method leads by $8.2$ points on Qwen3.6-35B-A3B
but trails \rcsloo by $3.7$ on GLM-4.7-Flash and by $1.5$ on
DeepSeek-V4-Flash-0731. The refill margin also does not increase consistently
with concentration, as its linear slope at 50\% ranges from $-0.30$ to
$+0.06$ percentage points per decile across the eight groups. Refill therefore
has model- and budget-dependent effects, without a consistent trend across
routing concentration.

\subsection{Reference-token perplexity on the same routing groups}
\label{app:ppl_view}

To distinguish distributional fidelity from reference-token likelihood, we
compute PPL at 50\% removal for all four backbones on the same predictions,
finite-token mask, and layer--batch--decile cells as
Figure~\ref{fig:routing-four-criteria}, whose lower two rows plot the result.
Using the reference-token losses in
Appendix~\ref{app:metrics}, let $\Delta\ell(t)=\ell_q(t)-\ell_p(t)$. Grouped
PPL is
\begin{equation}
 \mathrm{PPL}_{p}=\exp\!\big(\mathcal A[\ell_p]\big),\qquad
 \mathrm{PPL}_{q}=\exp\!\big(\mathcal A[\ell_p+\Delta\ell]\big),
\label{eq:story_ppl}
\end{equation}
where $\mathcal A$ is the routing-stratified average over valid token cells,
layers, and batches. Exponentiation follows NLL aggregation rather than
preceding it. The result is neither mean per-token perplexity nor exponentiated
predictive entropy, and need not equal whole-corpus token-weighted PPL.
We measure preservation by the absolute deviation
$|\mathrm{PPL}_{q}-\mathrm{PPL}_{p}|$, so ``closer'' means a smaller deviation
than REAP on the same group, not simply a lower PPL.

The bootstrap pairs the original, REAP, and \rcsloo predictions on the same
retained positions. Within each ID/OOD group, shared whole-batch indices
preserve all ten deciles. Recomputing and exponentiating each resampled NLL
average gives marginal intervals and paired PPL differences.

The two fidelity views disagree widely. Across 80 point
estimates from four models, two ID/OOD groups, and ten deciles, \rcsloo has
lower reverse KL than REAP in 72 but PPL closer to the original in only 47, and
\method gives 73 and 47 on the same cells. The disagreement is concentrated
rather than diffuse. Both measures favor \rcsloo throughout GLM-4.7-Flash's
groups, whereas on Qwen3.6-35B-A3B PPL is closer in only two ID and one OOD
decile despite lower reverse KL in every ID decile. DeepSeek-V4-Flash-0731
illustrates why the measures can part ways, since its ID NLL changes are
negative, so a lower PPL there means a larger deviation from the original
rather than a smaller one. Improving the full predictive distribution and
improving the likelihood of the particular reference continuation are therefore
distinct objectives, and our scores optimize neither directly.

\subsection{Response diversity}
\label{app:response_diversity}

Fidelity on fixed reference prefixes does not determine diversity under
free-running generation. We measure lexical and embedding diversity across
repeated responses using two sampling protocols. A
three-benchmark comparison covers all four criteria on Qwen3.6-35B-A3B, and a
larger collection covers fewer criteria on four tasks and two
backbones.

\paragraph{Criterion-family comparison.}
\label{app:distinct_n_protocol}
The left panel of Figure~\ref{fig:diversity-grid} compares REAP, \rcs,
\rcsloo, and \method on Qwen3.6-35B-A3B using Distinct-$n$ for $n=2,3,4$.
For each benchmark, all four criteria at both budgets and the unpruned
reference use 16 fixed response positions per question, drawn from the first
256 positions for AIME and 32 for each coding task, the largest pools common
to every criterion. This comparison draws from different sampling pools than
the four-task analysis below, so their percentages are not directly comparable.

Within each benchmark, values are
$100(\bar M/\bar M_{\mathrm{Original}}-1)$ for question-averaged metrics.
AIME 2026, HumanEval+, and LiveCodeBench receive equal weight, so AIME contributes
one third of the weight on 30 of 369 questions. For pointwise 95\% paired
question-bootstrap intervals, we average the resampled relative changes across
benchmarks before taking percentiles. These intervals condition on the
16 fixed responses rather than measure variation over new generations.
Of twelve intervals per budget (four criteria and three orders), all twelve
exclude zero at 50\% and nine do so at 25\%. They describe changes from
Original, not the \rcsloo--REAP difference. In per-benchmark point estimates,
\rcsloo exceeds REAP in all nine benchmark--order cells at 50\% and six at
25\%. Across the six benchmark-averaged budget--order cells, its median
advantage over \rcs is 4.02 percentage points in relative change from Original.
Refill lowers Distinct-$n$ relative to \rcsloo in five of these six cells.

\paragraph{Four-task collection and sampling.}
The repeated-generation collection is separate from the nine-task
benchmark. It contains 30 AIME 2026 \citep{maaAIME}, 219 Minerva
\citep{minerva2022}, 164 HumanEval+ \citep{evalplus2023}, and 175 LiveCodeBench
\citep{livecodebench2024} questions. The Minerva evaluation uses a
219-problem subset of the 272 OCWCourses undergraduate STEM problems
assembled from MIT OpenCourseWare. The 175 LiveCodeBench task IDs
are fixed across all variants compared here, which is what the paired
comparisons require, and are drawn from a different release than the
nine-task benchmark.

It covers Original plus EAN, Frequency, REAP, and \rcsloo at 25\% and 50\%
removal on Qwen3.6-35B-A3B, and Original, REAP, and \rcsloo at both budgets on
GLM-4.7-Flash. Unlike the three-benchmark comparison in
Figure~\ref{fig:diversity-grid}, this analysis includes
Minerva and GLM-4.7-Flash.

Mathematics has 2,048 generations per question, HumanEval+ has 512, and
LiveCodeBench has 256. Generation uses temperature $0.7$, top-$p=0.95$,
top-$k=20$, repetition penalty $1.0$, an 8,192-token output cap, and thinking
disabled, which differs from benchmark decoding in
Appendix~\ref{app:config}. We sample 16 response positions per question
without replacement, reusing the same positions across methods so that all
comparisons are paired. This yields 131,712 responses. Of these, 53,136
Qwen3.6-35B-A3B responses cover the three tasks shared with
Figure~\ref{fig:diversity-grid}.

\paragraph{Metric definitions.}
Metrics are computed per question and then averaged over questions.
We adapt the distinct-$n$ approach of \citet{li2016diversity} to model tokens,
using $n\in\{2,4\}$ in the four-task analysis and $n\in\{2,3,4\}$ in
Figure~\ref{fig:diversity-grid}. Distinct-$n$ divides the number of unique token $n$-grams
pooled across 16 responses by their total occurrences, without crossing
response boundaries. Jaccard distance averages
$1-|G_i\cap G_j|/|G_i\cup G_j|$ over all 120 response pairs, where $G_i$
and $G_j$ are response-level token 4-gram sets. Both lexical measures use
each model family's own tokenizer.

Embedding distance averages pairwise cosine distances using the fixed
all-MiniLM-L6-v2 encoder.\footnote{\url{https://huggingface.co/sentence-transformers/all-MiniLM-L6-v2}}
Our full-response extension
splits text into non-overlapping chunks of at most 254 WordPieces, mean-pools
token representations, and normalizes each chunk vector. We then normalize
their length-weighted mean. This retains content beyond the first chunk but
not chunk order. The resulting distances do not measure execution behavior or
mathematical-strategy equivalence.

\paragraph{Normalization and uncertainty.}
For each method, family, and task, this comparison reports
$100(\bar M/\bar M_{\mathrm{Original}}-1)$ using question means $\bar M$,
where Original denotes the unpruned MoE. Higher values indicate more diversity.
We recompute this ratio over 2,000 paired question-bootstrap draws and report
pointwise percentile 95\% intervals, conditional on the 16 fixed responses
and evaluated checkpoints. They exclude new-response and checkpoint-level
variation.

\paragraph{Length control and interpretation.}
The four diversity metrics are correlated, so they are not independent
evidence. We assess length
sensitivity by repeating the analysis on exactly the first 128 tokens of each
response. Within a family, we retain only questions eligible for every
available method at both pruning ratios. In AIME 2026, Minerva, HumanEval+,
and LiveCodeBench order, the resulting question counts are 30/219/134/143
for Qwen3.6-35B-A3B and 29/176/82/142 for GLM-4.7-Flash.

At 50\% removal, the four-task mean Distinct-4 gap between \rcsloo and REAP
changes from $+4.30\%$ on full responses to $-0.37\%$ under prefix control on
Qwen3.6-35B-A3B. The corresponding gaps on GLM-4.7-Flash are $+10.35\%$ and $+16.50\%$. These means weight task-relative
changes against REAP equally, unlike the Original-normalized diversity panel
of Figure~\ref{fig:diversity-grid}.
The control changes both length and question eligibility, so Qwen3.6-35B-A3B's reversal
under prefix control reflects both factors and is not interpretable as a pure
length effect. Lexical and embedding diversity measure surface variation in token
sequences and do not measure correctness, token efficiency, or the number of
distinct correct solutions.

\subsection{Response delimiters, formatting, and termination}
\label{app:behaviour}

We measure stray closing thinking delimiters, code-fence structure, and
length-limited finishes. These observable response properties complement task
scores without assessing semantic correctness. Their association with removal
budget differs across backbones, and some change substantially at 75\% removal.

\paragraph{Protocol and operational definitions.}
This analysis follows the repeated-generation protocol in
Appendix~\ref{app:response_diversity}, using temperature $0.7$, top-$p=0.95$,
top-$k=20$, an 8{,}192-token cap, and thinking disabled. The tables and
the three response-form panels of Figure~\ref{fig:diversity-grid} use 256 shared
sample positions per question, giving 44,800 completions per variant on
LiveCodeBench. These measurements characterize response form, not the
benchmark accuracy of
Table~\ref{tab:main}.

A completion contains a \emph{stray} \thinkclose when that literal string
appears at least once. We count each such completion once, regardless of the
number of occurrences. In the Qwen3.6-35B-A3B tokenizer it is a single vocabulary entry that is not
flagged as a special token, so ordinary detokenization leaves it in the text.
The chat template already closes the thinking block inside the prompt, so any
further occurrence is emitted by the model and is not evidence about latent
reasoning.
\emph{No fence} denotes the absence of a line-delimited code fence, while
\emph{prefix} denotes a nonempty prefix before such a fence. Both are shares
of all completions, and neither is a semantic code or narration classifier.
Unfenced responses can contain code, which the evaluator may extract through
its bare-code fallback. \emph{Length-limited} uses the decoder's finish reason.
The rightmost panel of Figure~\ref{fig:diversity-grid} reports a \emph{loop
share}, the fraction of length-limited completions in which fewer than one
quarter of the token 4-grams are unique. The threshold is a heuristic
definition of high repetition. The share measures how often length-limited
completions also satisfy this criterion. It does not identify why generation
reached the output cap or measure repetition among all completions.

Response lengths in this table count Unicode characters, whereas
Figure~\ref{fig:diversity-grid} counts tokens under the backbone's own
tokenizer, matching the units of its Distinct-$n$ panel. The ratio of median
character length to median token length falls from $3.44$ for the unpruned
model to $3.04$ for REAP at 50\% removal. Consequently, the proportional
increase in median character length is smaller than the increase in median
token length. Intervals resample questions
to retain dependence among completions sharing a prompt.

\begin{table}[!t]
\centering
\small
\setlength{\tabcolsep}{4pt}
\renewcommand{\arraystretch}{1.14}
\caption{\textbf{Response form on LiveCodeBench by criterion and budget,}
\textbf{Qwen3.6-35B-A3B.}
Rates are percentages except stray \thinkclose per 10,000. Lengths are in kchar.
Brackets give pointwise 95\% question-bootstrap intervals.
GLM-4.7-Flash is measured but not printed here because that collection covers
only REAP and \rcsloo, with no \method row for comparison.
Its \rcsloo ladder appears in Table~\ref{tab:behaviour-ladder}.}
\label{tab:behaviour-lcb}
\begin{tabularx}{\columnwidth}{@{}>{\raggedright\arraybackslash}p{1.7cm}c>{\raggedleft\arraybackslash}p{2.7cm}*{4}{>{\raggedleft\arraybackslash}X}@{}}
\toprule
\thead{Criterion} & \thead{Cut} & \thead{Stray \thinkclose/10k} & \thead{No fence} & \thead{Prefix} & \thead{Len-lim.} & \thead{Med. kchar} \\
\midrule
\multicolumn{7}{@{}l}{\textit{Qwen3.6-35B-A3B, 44,800 completions per checkpoint}} \\
Original & -- & 3.6\,[1.6, 6.0] & 7.1 & 61.4 & 16.0 & 8.3 \\
EAN & 25\% & 0.4\,[0.0, 1.1] & 8.7 & 52.3 & 19.5 & 9.7 \\
Frequency & 25\% & 2.0\,[0.7, 3.6] & 5.2 & 60.5 & 15.1 & 9.2 \\
REAP & 25\% & 1.6\,[0.4, 3.1] & 8.4 & 55.1 & 19.1 & 10.3 \\
\rcsloo & 25\% & 0.2\,[0.0, 0.7] & 8.4 & 52.3 & 17.0 & 8.6 \\
\method & 25\% & 0.4\,[0.0, 1.1] & 7.7 & 54.7 & 16.3 & 8.4 \\
EAN & 50\% & 4.2\,[2.2, 6.9] & 18.3 & 40.2 & 33.9 & 9.6 \\
Frequency & 50\% & 16.1\,[10.9, 22.3] & 2.9 & 30.4 & 13.0 & 6.7 \\
REAP & 50\% & 13.4\,[7.6, 20.1] & 11.5 & 66.3 & 24.2 & 12.0 \\
\rcsloo & 50\% & 68.5\,[22.8, 133.7] & 10.2 & 16.6 & 25.1 & 7.3 \\
\method & 50\% & 8.5\,[4.2, 13.6] & 8.2 & 19.6 & 25.8 & 7.8 \\
\bottomrule
\end{tabularx}
\end{table}

\begin{table}[!t]
\centering
\small
\setlength{\tabcolsep}{4pt}
\renewcommand{\arraystretch}{1.14}
\caption{\textbf{The \rcsloo budget ladder on LiveCodeBench.}
Spearman correlations use the nine pruned checkpoints, not only the four shown
and not the 0\% reference. Units follow Table~\ref{tab:behaviour-lcb}.}
\label{tab:behaviour-ladder}
\begin{tabularx}{\columnwidth}{@{}>{\raggedright\arraybackslash}p{2.9cm}*{4}{>{\raggedleft\arraybackslash}X}>{\raggedleft\arraybackslash}p{1.0cm}>{\raggedleft\arraybackslash}p{1.2cm}@{}}
\toprule
\thead{Quantity} & \thead{0\%} & \thead{25\%} & \thead{50\%} & \thead{75\%} & \thead{$\rho$} & \thead{$p$} \\
\midrule
\multicolumn{7}{@{}l}{\textit{Qwen3.6-35B-A3B}} \\
Stray \thinkclose/10k & 3.6 & 0.2 & 68.5 & 42.0 & $+0.75$ & 0.020 \\
No fence (\%) & 7.1 & 8.4 & 10.2 & 25.8 & $+0.82$ & 0.007 \\
Prefix (\%) & 61.4 & 52.3 & 16.6 & 43.6 & $-0.28$ & 0.460 \\
Length-limited (\%) & 16.0 & 17.0 & 25.1 & 42.0 & $+0.82$ & 0.007 \\
Median kchar & 8.3 & 8.6 & 7.3 & 8.4 & $+0.43$ & 0.244 \\
\midrule
\multicolumn{7}{@{}l}{\textit{GLM-4.7-Flash}} \\
Stray \thinkclose/10k & 0.4 & 0.0 & 2.9 & 70.5 & $+0.89$ & 0.001 \\
No fence (\%) & 0.2 & 0.5 & 1.7 & 17.0 & $+0.93$ & $<$0.001 \\
Prefix (\%) & 66.5 & 68.9 & 58.3 & 14.3 & $-0.88$ & 0.002 \\
Length-limited (\%) & 17.5 & 19.6 & 18.6 & 78.7 & $+0.58$ & 0.099 \\
Median kchar & 3.8 & 4.0 & 3.0 & 30.4 & $+0.17$ & 0.668 \\
\bottomrule
\end{tabularx}
\end{table}

\paragraph{Sample-pool sensitivity.}
Rare-event estimates depend on the number of sampled completions. REAP's
stray-delimiter rate is $8.93$ per $10{,}000$ with 32 completions per question
versus $13.39$ with 256, although the criteria retain the same point-estimate
ordering at 50\% removal under either pool. Both criteria in the main-text
panel use the full
44,800-completion pool, whose rates and intervals are reported below.

\paragraph{Budget-dependent changes.}
We evaluate the unpruned reference and nine \rcsloo checkpoints with removal
ratios up to 75\% on each backbone.
Table~\ref{tab:behaviour-ladder} shows the quarter budgets and reports each
rank correlation over the nine pruned checkpoints, excluding the unpruned
0\% reference, which is shown for comparison but is not a removal budget.
Only \rcsloo covers every measured
budget in this collection, while Table~\ref{tab:behaviour-lcb} compares other
criteria at shared budgets. On Qwen3.6-35B-A3B, the unfenced and
length-limited shares each have Spearman $\rho=+0.82$ with removal budget
($p=0.007$), and the rate of stray \thinkclose has $\rho=+0.75$ ($p=0.020$).
GLM-4.7-Flash gives $\rho=+0.93$ ($p<0.001$), $+0.58$ ($p=0.099$), and
$+0.89$ ($p=0.001$), respectively. These descriptive associations are not
stepwise monotonic. Among the undisplayed budgets, Qwen3.6-35B-A3B's
length-limited share falls at 40\% and GLM-4.7-Flash's stray-delimiter rate
peaks at 60\%. The displayed values also show Qwen3.6-35B-A3B's stray rate
higher at 50\% than at 75\%.

The remaining two measures show no such trend. Median response length has
$\rho=+0.43$ ($p=0.24$) on Qwen3.6-35B-A3B and $+0.17$ ($p=0.67$) on
GLM-4.7-Flash, and the prefix share has $\rho=-0.28$ ($p=0.46$) and $-0.88$
($p=0.002$), respectively. Nonsignificance here is weak evidence either way,
since all tests are exploratory, use nine pruned checkpoints on shared questions, report
two-sided $p$-values uncorrected for multiple comparisons, and do not estimate
independent-run variability.

At 75\% removal,
Qwen3.6-35B-A3B produces $25.8\%$ unfenced and $42.0\%$ length-limited
completions, while GLM-4.7-Flash produces $78.7\%$ length-limited completions
at a median length of 30{,}444 characters. These rates identify a substantial
change in termination behavior under the evaluated decoding protocol.

\paragraph{Refill and delimiter emissions on Qwen3.6-35B-A3B.}
At 25\% removal, every criterion's tag-rate point estimate is at or below the
unpruned Qwen3.6-35B-A3B rate of $3.6$ per 10{,}000, with \rcsloo lowest at $0.2$
$[0.0, 0.7]$ (Table~\ref{tab:behaviour-lcb}). At 50\%, \rcsloo reaches
$68.5$ $[22.8, 133.7]$, versus REAP's $13.4$ $[7.6, 20.1]$. The refill
checkpoint reduces the observed rate to $8.5$ $[4.2, 13.6]$, below REAP's point
estimate but with overlapping marginal intervals.

The \rcsloo--REAP gap also appears when each variant's generated responses
are binned by their own lengths using cutoffs from the unpruned length
distribution. In the fourth interval, rates are $100.7$ and $14.2$ per
10{,}000, with \method at $14.4$. Matching question and sample position against
Original gives 307 newly occurring tag events for \rcsloo, 60 for REAP, and
38 for \method. The event comparison pairs question and sample position,
whereas the length-stratified comparison uses common cutoffs with
variant-specific bin membership. Neither isolates the effect of changed
response length or termination. On the same collection, GLM-4.7-Flash's corresponding
\rcsloo rate is $2.9$ $[1.1, 4.9]$, as shown at 50\% removal in
Table~\ref{tab:behaviour-ladder}. REAP, which is not included in that table,
has a rate of $1.6$ $[0.4, 2.7]$. The magnitude of the Qwen3.6-35B-A3B gap
therefore does not carry over. Across variants, $98$--$100\%$ of tag-containing
completions also contain code output, but this does not establish answer
correctness or preserved utility.

\paragraph{Length shifts on Qwen3.6-35B-A3B.}
At every evaluated 25\% Qwen3.6-35B-A3B checkpoint, the median paired length change relative to
Original, matched by question and sample position, is at most 17 characters
in magnitude. At 50\%, REAP
lengthens responses by a median 542 characters ($67.3\%$ of completions longer),
\rcsloo by 126 and \method by 114, while Frequency shortens them by 129. The share of completions with a
prefix before the code fence falls to $16.6\%$ for \rcsloo and $19.6\%$ for
\method but rises to $66.3\%$ for REAP, against $61.4\%$ unpruned.

\paragraph{Task and backbone dependence.}
HumanEval+, the shortest task on both backbones, is the only task nearly free
of visible tags throughout, with one tag-containing completion out of 671{,}744.
On the other tasks, the two backbones differ. Qwen3.6-35B-A3B emits tags on
LiveCodeBench but almost nowhere else, with one AIME 2026 completion out of
84{,}480 and none of 616{,}704 on Minerva. GLM-4.7-Flash instead shows this
behavior on mathematics. At 50\% removal, $2.90\%$ $[1.25, 5.18]$ of its
AIME 2026 completions carry a tag under REAP and $7.25\%$ $[4.11, 11.22]$
under \rcsloo. The corresponding rates are $0.67\%$ and $0.54\%$ on Minerva,
against at most $0.03\%$ on LiveCodeBench.

Task-level median length alone does not order these rates. Qwen3.6-35B-A3B's
AIME 2026 responses are longer than its LiveCodeBench responses (10{,}862
versus 8{,}332 characters) yet contain almost no tags. This comparison does
not isolate length from other task differences.

\paragraph{Matched examples.}
For one LiveCodeBench question at a fixed sample position, the 50\%
checkpoints produce 1{,}894 and 1{,}897 characters under Frequency and \method
(both opening with code), 4{,}228 under REAP, 4{,}875 under EAN, and 6{,}923
under \rcsloo, versus 3{,}423 unpruned. Only \rcsloo emits a closing
thinking tag, preceded by 6{,}233 characters ending with
``\texttt{Let\textquotesingle s code accordingly.}'' and followed by program text.
This describes the visible response sequence, not a latent reasoning process.
The median pre-tag length among this checkpoint's tag-containing completions is
12{,}507 characters.

A separate \rcsloo 70\% completion reaches the cap after 21{,}966 characters
without a code fence, contributing to both the unfenced and length-limited
counts. The two examples were selected by median preamble length within their
respective categories. The tag-containing example terminates normally.
They illustrate the measured categories rather than establish comparative
answer quality.

\section{Extended Related Work}
\label{app:related_work}

MoE compression methods differ along three axes, namely the removal unit, the
importance signal, and the selection or adaptation procedure applied
afterwards. Sorting the literature this way clarifies what our comparison does
and does not cover, because \method varies only the importance signal. It
removes whole experts, ranks them independently under a fixed layerwise budget,
and does not update learned expert or router weights. The four subsections below
group prior work by how far it departs from that setting, moving from
alternative whole-expert saliency signals, through methods that model expert
relations or search over candidate sets, to methods that change the surviving
functions or the router itself, and finally to compression below the level of a
whole expert.

\subsection{Whole-expert saliency and conditional aggregation}

Routing, activation, and parameter statistics offer inexpensive rankings.
SEER-MoE uses activation counts or routing-probability soft counts with
layerwise or global removal \citep{muzio2024seer}. AIMER instead ranks experts
without calibration data using a normalized $\ell_1/\ell_2$ concentration
statistic over their weights \citep{aimer2026}. \citet{fantasticexperts2025}
compare several expert-dropping criteria, including activation norms. These
frequency, activation-norm, and weight-concentration criteria do not reconstruct
the routed mixture after deletion.

REAP, our closest magnitude-based comparator, averages $w_i\|f_i\|_2$ over
tokens selecting expert $i$ \citep{reap2026}. Its analysis includes promoted
substitution and survivor renormalization, but its saliency uses only the
removed expert's weighted output norm.
In \citet{unified2026}, the exact token-level perturbation also contains both
survivor reweighting and the promoted replacement. The authors describe direct
candidate-wise rerouting as costly and report high overlap between selections
induced by exact damage and its magnitude proxy. Our refill identity evaluates
the token-level perturbation algebraically under the stated assumptions, using
the original routed outputs and one promoted expert output. This requires
additional local computation but avoids separately reconstructing every
candidate pruned mixture. The fixed-support identity covers deletion without
refill. Our conditional RMS aggregation differs from the corpus-summed damage
used in that analysis.

Token-level scoring and cross-token aggregation remain separate choices.
\method uses conditional RMS rather than REAP's conditional mean, while the
component comparisons hold the token quantity fixed when varying aggregation
(Appendix~\ref{app:ablations}). Neither conditioning nor RMS follows from the
local deletion identity, and no aggregation rule is uniformly preferred by
the available results.

\subsection{Expert relations and candidate-set selection}

Relational information is not unique to consensus-residual scoring. STUN
clusters router-derived behavior before representative selection or selective
reconstruction and unstructured pruning \citep{stun2024}. HC-SMoE instead
clusters average expert outputs and merges each group \citep{hcsmoe2024}.
SHAPE uses routing co-occurrence in a coalition-utility proxy
\citep{shape2026}. ConMoE combines routing-conditioned contribution with
nearest-expert parameter distance before prototype reassignment
\citep{conmoe2026}. MAESTRO derives importance from cross-layer routing
transitions before layerwise selection \citep{maestro2026}, with its main
results using attention-only LoRA recovery while experts and routers remain
frozen.

Our scores instead evaluate output change under specified deletion
counterfactuals at a fixed layer input. The reference is the routed mixture,
not a cluster representative or a routing graph. This distinction concerns
the local score, not a claim that independent expert ranking captures all
interactions. Residuals can reinforce or cancel under joint removal
(Appendix~\ref{app:set}), so exact single-deletion quantities do not yield an
exact set-selection algorithm.

Other approaches evaluate candidate sets directly.
\citet{lu2024notallexperts} select retained combinations by layer-output
reconstruction, EEP searches pruning and merging configurations
\citep{eep2024}, and MoE-I$^2$ evaluates layerwise deletion sets before jointly
choosing candidates across short layer blocks \citep{moeii2024}. We instead
retain the highest-scoring experts under a fixed layerwise budget. Our local
single-deletion identities neither solve combinatorial selection nor establish
superiority to subset search. \citet{provablepruning2024} give theoretical
conditions under a stylized classification model in which pruning experts
from a fine-tuned sparse MoE preserves performance. Our local surrogate
does not provide such a guarantee.

\subsection{Merging, router adaptation, and budget allocation}

Expert merging changes the surviving functions, rather than only selecting
which original experts remain. MC-SMoE groups experts around high-usage
representatives using router-logit similarity, aligns neurons, and averages
expert parameters before low-rank and structured compression
\citep{mcsmoe2024}. REAM protects high-saliency experts as centroids and merges
similar non-centroid experts into them \citep{ream2026}. Neither operation is
the router refill modeled by \rcsrefill, which promotes an existing unselected
expert without averaging its parameters with the removed expert.

Router KD adapts only router parameters after compression, distilling the
uncompressed model's next-token distribution while freezing the experts and
other backbone parameters \citep{routercalib2026}. It is a recovery procedure,
not an expert-ranking criterion. Budget allocation is another distinct
choice. GRAPE compares residual redundancy across layers to allocate
non-uniform budgets, merging similar expert pairs within the selected layer
\citep{globalprune2026}. MC-SMoE also permits adaptive layerwise merging ratios
\citep{mcsmoe2024}. These examples show that merging, adaptation, and budget
allocation are overlapping design choices rather than mutually exclusive
method families.
Our experiments use uniform layerwise removal ratios, leave retained expert
weights unchanged, and perform no recovery training. Combining residual-based
scores with merging, router adaptation, or non-uniform allocation is outside
the present comparison, so our results do not establish their compatibility
or joint benefit.

\subsection{Finer-grained compression and comparison scope}

MoE-Pruner sparsifies weights within experts using weight magnitudes and
router-weighted input activation statistics \citep{moepruner2024}. HEAPr treats
feed-forward intermediate neurons as atomic expert units and uses output-space
second-order information to estimate their importance \citep{heapr2025}.
Both operate at a finer granularity than removing entire routed experts. An
atomic expert in this terminology is not one of the router's original
selectable experts. Fisher-MoE also studies a finer-grained variant that ranks
FFN intermediate dimensions by Fisher importance and reports capability
concentration in small sets of those dimensions \citep{fishermoe2026}.

These approaches share an interest in preserving model behavior, but their
removal units and compression procedures differ from ours, so their reported
sparsity levels are not matched whole-expert budgets.

How pruned experts are judged is itself contested, which bears on our decision
to report generation behavior alongside task scores.
\citet{halfexperts2026} prune two open-weight MoE backbones for coding under
five selection strategies and report that the winning strategy changes between
model families, and that perplexity can rate a broken model above an intact
one. These findings concern cross-model sensitivity and disagreement between
perplexity and functional performance. Our experiments examine related
distinctions among predictive fidelity, downstream task scores, and generation
behavior (Appendix~\ref{app:ppl_view}). That study evaluates recovery tuning
separately, while our experiments use no recovery training and cover nine
reasoning-intensive tasks.

Our comparisons hold the removal unit, layerwise budget, and learned
parameters fixed while applying each architecture's routing rule to the
retained pool. They vary the scoring criterion among \rcs, \rcsloo, and \rcsrefill against the
magnitude-based and frequency-based baselines, with \method denoting
\rcsrefill under conditional RMS. Our evidence therefore concerns which local
signal best identifies replaceable experts, and it leaves open how such a
signal would combine with merging, router adaptation, non-uniform budget
allocation, or sub-expert granularity. Their joint benefit with residual-based ranking remains untested.

\end{document}